\documentclass[preprint,12pt]{elsarticle}

\usepackage{amsmath, amssymb, amsopn, amsthm} 
\usepackage{bm}
\usepackage{booktabs} 
\usepackage{xcolor}
\usepackage[colorlinks]{hyperref}
\usepackage{url} 
\usepackage[normalem]{ulem}
\usepackage[capitalize,noabbrev,nameinlink]{cleveref}
\crefformat{equation}{(#2#1#3)}
\usepackage{adjustbox}

\newcommand{\new}[1]{{\color{black} #1}}

\newcommand{\ubar}{\bar{\bm{u}}}
\newcommand{\ubari}{\bar{\bm{u}}^{(i)}}
\newcommand{\ubarj}{\bar{\bm{u}}^{(j)}}
\newcommand{\utrue}{\bm{u}}
\newcommand{\utruei}{\bm{u}^{(i)}}
\newcommand{\uhat}{\hat{\bm{u}}}
\newcommand{\uhati}{\hat{\bm{u}}^{(i)}}
\newcommand{\pari}{\bm{\mu}^{(i)}}
\newcommand{\param}{\bm{\mu}}
\newcommand{\nbar}{N}
\newcommand{\nhat}{M}

\newcommand{\ndim}{D}

\newcommand{\phii}{\bm{\phi}^{(i)}}
\newcommand{\bmphi}{\bm{\phi}}
\newcommand{\phihati}{\hat{\bm{\phi}}^{(i)}}
\newcommand{\phihat}{\hat{\bm{\phi}}}

\newcommand{\phimat}{\bm{\Phi}}
\newcommand{\phihatm}{\hat{\bm{\Phi}}}
\newcommand{\adj}{\bm{W}}
\newcommand{\degmat}{\bm{D}}
\newcommand{\gl}{\bm{L}}
\newcommand{\glpq}[2]{\bm{L}^{(#1,#2)}}

\newcommand{\ip}[1]{\langle #1 \rangle_{p-q}}
\newcommand{\ipF}[1]{\langle #1 \rangle_{p-q,F}}

\newcommand{\idnbar}{\bm{I}_{\nbar}}
\newcommand{\idnhat}{\bm{I}_{\nhat}}
\newcommand{\pnhat}{\bm{P}_{\nhat}}
\newcommand{\evalm}{\bm{\Lambda}}
\newcommand{\efunm}{\bm{\Psi}}
\newcommand{\phimapm}{\bm{\Phi}^*}
\newcommand{\phicov}{\bm{C}}

\newcommand{\onenbar}{\mathbf{1}_{\nbar}}
\newcommand{\nparam}{P}
\newcommand{\coeffm}{\boldsymbol{A}}
\newcommand{\coeff}{\boldsymbol{a}}

\DeclareMathOperator*{\argmin}{argmin}

\newtheorem{thm}{Theorem}[section]
\newtheorem{lem}[thm]{Lemma}
\newtheorem*{rmk}{Remark}

\newcommand{\caseonedim}{$5$}
\newcommand{\caseonelfn}{$6{,}000$}
\newcommand{\caseonehfn}{$150$}
\newcommand{\caseonefractionn}{$2.5\%$}
\newcommand{\caseonelferr}{$6.82$}
\newcommand{\caseonelfstd}{$5.11$}
\newcommand{\caseonemferrbf}{$\textbf{1.67}$}
\newcommand{\caseonemfstd}{$2.09$}
\newcommand{\caseoneerrred}{$75.5\%$}

\newcommand{\casetwodim}{$10{,}201$}
\newcommand{\casetwolfn}{$3{,}000$}
\newcommand{\casetwohfn}{$100$}
\newcommand{\casetwofractionn}{$3.3\%$}
\newcommand{\casetwolferr}{$29.49$}
\newcommand{\casetwolfstd}{$12.19$}
\newcommand{\casetwomferr}{$6.41$}
\newcommand{\casetwomfstd}{$4.77$}
\newcommand{\casetwomferrbf}{\textbf{$\casetwomferr{}$}}
\newcommand{\casetwoerrred}{$78.3\%$}

\newcommand{\casethreedim}{$10{,}201$}
\newcommand{\casethreelfn}{$6{,}000$}
\newcommand{\casethreehfn}{$120$}
\newcommand{\casethreefractionn}{$2\%$}
\newcommand{\casethreelferr}{$25.2$}  
\newcommand{\casethreelfstd}{$5.4$}  
\newcommand{\casethreemferr}{$3.96$} 
\newcommand{\casethreemfstd}{$2.18$} 
\newcommand{\casethreemferrbf}{\textbf{$\casethreemferr{}$}}
\newcommand{\casethreeerrred}{$84.3\%$}

\newcommand{\casefourdim}{$512$}
\newcommand{\casefourlfn}{$10{,}000$}
\newcommand{\casefourhfn}{$50$}
\newcommand{\casefourfractionn}{$0.5\%$}
\newcommand{\casefourlferr}{$30.7$} 
\newcommand{\casefourlfstd}{$2.68$} 
\newcommand{\casefourmferr}{$4.43$}
\newcommand{\casefourmfstd}{$3.34$}
\newcommand{\casefourmferrbf}{\textbf{$\casefourmferr{}$}}
\newcommand{\casefourerrred}{$85.6\%$}

\newcommand{\casefivedim}{$221$}
\newcommand{\casefivelfn}{$10{,}000$}
\newcommand{\casefivehfn}{$100$}
\newcommand{\casefivefractionn}{$1\%$}
\newcommand{\casefivelferr}{$18.63$}
\newcommand{\casefivelfstd}{$5.96$}
\newcommand{\casefivemferr}{$3.87$}
\newcommand{\casefivemfstd}{$2.32$}
\newcommand{\casefivemferrbf}{\textbf{$\casefivemferr{}$}}
\newcommand{\casefiveerrred}{$79.2\%$}

\journal{Computer Methods in Applied Mechanics and Engineering}

\begin{document}

\begin{frontmatter}



\author[usc]{Orazio Pinti}
\ead{pinti@usc.edu}

\author[bham,caltech]{Jeremy M. Budd}
\ead{j.m.budd@bham.ac.uk}

\author[caltech]{Franca Hoffmann}
\ead{franca.hoffmann@caltech.edu}

\author[usc]{Assad A. Oberai}
\ead{aoberai@usc.edu}

\affiliation[usc]{organization={Department of Aerospace and Mechanical Engineering, University of Southern California},
            addressline={3650 McClintock Ave},
            city={Los Angeles},
            postcode={90089},
            state={CA},
            country={USA}
            }
\affiliation[bham]{organization={School of Mathematics, University of Birmingham},
            addressline={Watson Building, Edgbaston},
            city={Birmingham},
            postcode={B15 2TT},
            %state={West Midlands},
            country={UK}
            }

\affiliation[caltech]{organization={Computing and Mathematical Sciences, California Institute of Technology},
            addressline={1200 E California Blvd},
            city={Pasadena},
            postcode={91125},
            state={CA},
            country={USA}}
            
\title{Graph Laplacian-based Bayesian Multi-fidelity Modeling}


\begin{abstract}
We present a novel probabilistic approach for generating multi-fidelity data while accounting for errors inherent in both low- and high-fidelity data. In this approach a graph Laplacian constructed from the low-fidelity data is used to define a multivariate Gaussian prior density for the coordinates of the true data points. In addition,  few high-fidelity data points are used to construct a conjugate likelihood term. Thereafter, Bayes rule is applied to derive an explicit expression for the posterior density which is also multivariate Gaussian. The maximum \textit{a posteriori} (MAP) estimate of this density is selected to be the optimal multi-fidelity estimate. It is shown that the MAP estimate and the covariance of the posterior density can be determined through the solution of linear systems of equations. Thereafter, two methods, one based on spectral truncation and another based on a low-rank approximation, are developed to solve these equations efficiently. The multi-fidelity approach is tested on a variety of problems in solid and fluid mechanics with data that represents vectors of quantities of interest and  discretized spatial fields in one and two dimensions. The results demonstrate that by utilizing a small fraction of high-fidelity data, the multi-fidelity approach can significantly improve the accuracy of a large collection of low-fidelity data points. 
\end{abstract}

\begin{graphicalabstract}
\centering
\vspace{1cm}
\includegraphics[width=1.\textwidth]{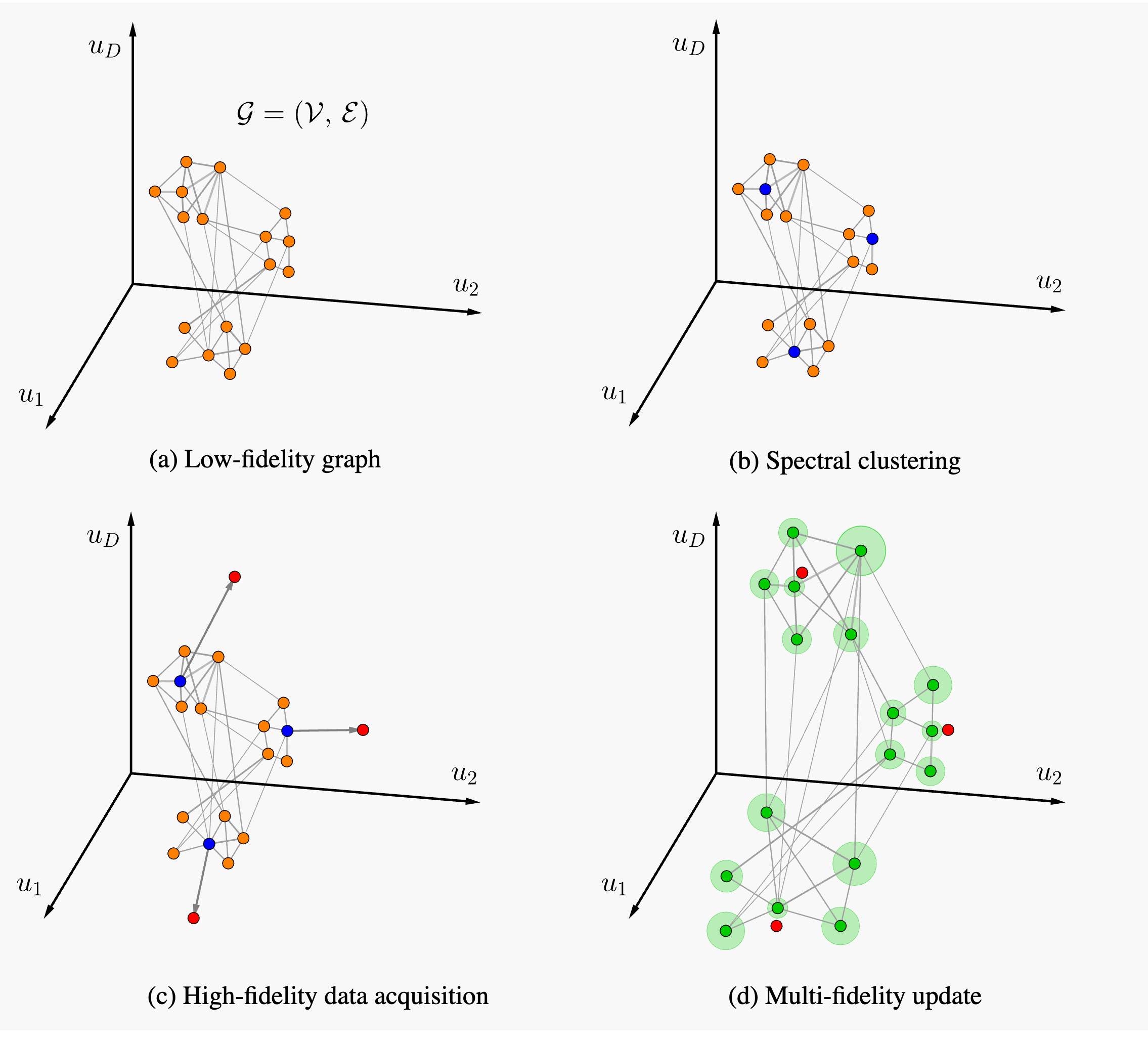}
\end{graphicalabstract}


\end{frontmatter}


\section{Introduction}
In numerous computational engineering tasks such as optimization, uncertainty quantification, sensitivity analysis, and optimal design, it is often necessary to conduct extensive simulations of a physical system. These simulations rely on models that approximate the true input-output relations of the system. Inputs typically include parameters or prescribed data like initial conditions, boundary conditions and forcing functions, while outputs are the resulting quantities or fields of interest. 

Evaluating a model involves its numerical implementation and execution, a process that varies in fidelity and cost. Fidelity refers to the accuracy of the model prediction relative to the true system behavior, whereas cost quantifies the required computational resources. Generally, higher fidelity is associated with greater cost. In practical scenarios, it is common to have access to multiple models of a system, or to be able to tune some hyper-parameter, like the mesh size, or the time-integration step, of a given model to increase its accuracy (and cost).  For instance, in computational fluid dynamics, choices range from high-cost, accurate direct numerical simulations (DNS) to less expensive Reynolds-averaged Navier-Stokes (RANS) models. Similarly, when using the finite element method, the mesh size can be used to control the accuracy and computational expense of the resulting model.

A model is said to be high-fidelity if it can capture the true behavior of the system within a level of accuracy that is equal to or greater than the one required for the given task. If a model is not high-fidelity, it is said to be low-fidelity. Lower-fidelity models are designed to trade some of the predictive accuracy in favor of a more competitive evaluation cost. This can be attained in several ways.
The most common methods include: using simplified physics or making stronger modeling assumptions, linearizing the dynamics of the system, employing projection-based or data-fitting surrogate models, or using coarser numerical discretizations. It is worth noting that a high- vs low-fidelity model pair can also be represented by experiments and numerical simulations, respectively. 

Completing a task that requires a large number of simulations solely relying on a high-fidelity model can often be prohibitive or impractical, while employing only lower-fidelity models might not lead to results that are accurate enough. 
This is where multi-fidelity methods are useful. The goal of a multi-fidelity method is to retain as much of the accuracy of high-fidelity model as possible, while incurring a fraction of the cost by leveraging lower-fidelity models.
In a typical multi-fidelity framework, low-fidelity data is generated to explore the input space, and obtain an approximation of the response of the system. Then, a limited number of high-fidelity data points are computed or measured, and techniques that learn the response from the low-fidelity data, and improve it by using the high-fidelity data are applied. 

Multi-fidelity methods have been widely used in optimization  \cite{Goel2007,gramacy2009adaptive}, uncertainty quantification \cite{Allaire_2014}, uncertainty propagation \cite{Ng2014,Peherstorfer2016}, and statistical inference \cite{Kaipio2007} (see Fernández-Godino et al. \cite{fernandez2016review} and Peherstofer et al. \cite{Peherstorfer2016} for two comprehensive reviews). 

\new{Broadly speaking, multi-fidelity methods can be classified into two main types: model-agnostic methods which operate at the data level, and model-specific methods which work with specific models. Model-agnostic methods combine data from both low- and high-fidelity models or experiments to produce results that are almost as accurate as the high-fidelity data and as abundant as the low-fidelity data. In contrast, model-specific methods use limited data from high-fidelity models to improve the form of a low-fidelity model. For these methods the outcome of the multi-fidelity update is not data but rather it is improvements to the low-fidelity model which can then be used to generate more accurate multi-fidelity data. The advantage of model-agnostic methods is that they are easy to apply across multiple fields which is not the case for model-specific methods. However, model-specific methods have the advantage that they are easier to interpret and work better outside the parameters on which they were trained. We note that both these types of methods can be deterministic or probabilistic, where the latter provides a multi-fidelity solution as well as an estimate of the uncertainty in this solution. In what follows we provide a brief review of model-agnostic and model-specific methods. We focus more heavily on the former, since the method developed in this manuscript also belongs to this category.}

\new{Model-specific multi-fidelity methods begin with an inexpensive low-fidelity model of the physics of the underlying process and incorporate additional terms into it. These terms are designed and added such that they are consistent with the underlying conservation laws. For example, when modeling turbulent flows, the low-fidelity model may be a RANS model like the $k-\epsilon$ model, and the additional terms that are inserted modify the definition of the Reynolds stress without modifying the conservation of mass and balance of linear momentum. The form of these terms is further guided by principles like Galilean invariance or material form indifference and their parameters are determined using data from high-fidelity simulations. In the context of turbulent flows, this data might be obtained from large eddy simulations. Often, the corrective terms added to the low-fidelity model are stochastic, resulting in a probabilistic multi-fidelity model \cite{soize2017nonparametric}. Such models have a rich history in computational physics, and have been developed for modeling turbulent flows \cite{xiao2017random,geneva2019quantifying}, chemical reactions \cite{najm2009uncertainty}, nonlinear solid mechanics \cite{teichert2019machine,soize2019probabilistic}, and various other fields of engineering and science. Within the context of uncertainty quantification, these methods are also often referred to as methods for determining model form uncertainty.}

\subsection{Model-agnostic multi-fidelity methods}

In \new{probabilistic model-agnostic multi-fidelity models} the parameter inputs are typically modeled as random variables drawn from a prescribed distribution, and the main objective is to compute the statistics of the output. When multiple models for the output are available, one can leverage the correlation between different approximations. The correlation coefficients act as auxiliary random variables, whose statistics are easier to compute. Inspired by multilevel Monte Carlo methods \cite{heinrich2001multilevel,giles_2015}, various techniques have been developed to construct an unbiased estimator for the mean and variance of the highest-fidelity model output, leveraging lower-fidelity approximations \cite{Ng2014,Peherstorfer2016}. 
For instance, Geraci et al. \cite{geraci2017multifidelity} proposed a multi-fidelity multi-level method, considering multiple models with different levels of accuracy at the same time. 
Similarly, Peherstorfer et al. \cite{Peherstorfer2016} suggested an optimal strategy for distributing computational resources across different models to minimize the variance of a multi-fidelity estimator.

In co-kriging methods \cite{kennedy2000predicting, forrester2007multi, perdikaris2015multi, park2017remarks} the multi-fidelity response is expressed as a weighted sum of two Gaussian processes, one modeling the low-fidelity data, and the other representing the discrepancy between the low- and high-fidelity data. The parameters of the mean and correlation functions of these processes are determined by maximizing the log-likelihood of the available data. 
Co-kriging has also been extensively investigated in the context of multi-fidelity optimization \cite{allaire2014mathematical, goel2007ensemble, gramacy2016modeling, gramacy2008bayesian}. Other methods make use of radial basis functions (RBFs) to model the low-fidelity response. Specifically, the low-fidelity surrogate is written as an expansion in terms of a set of RBFs, and the coefficients are determined by interpolating the available low-fidelity data. The multi-fidelity approximation is then obtained in different ways. These include determining a scaling factor and a discrepancy function, which can be modeled using a kriging surrogate \cite{park2017remarks}, or another expansion in terms of RBFs \cite{durantin2017multifidelity, song2019radial}. In some cases the multi-fidelity surrogate is constructed by mapping the low-fidelity response directly to the high-fidelity response \cite{zhou2017variable}.

More recently, deep neural networks have been used to fit low-fidelity data and learn the complex map between the input and output vectors in the low-fidelity model. Then, the relatively small amount of high-fidelity data is used in combination with techniques such as transfer learning \cite{chakraborty2021transfer, de2020transfer}, embedding the knowledge of a physical law through physics-informed loss functions \cite{PENWARDEN2022110844, MENG2020109020, meng2021multi, RAISSI2019686}, or, in the case of multiple levels of fidelity, concatenating multiple neural networks together to yield a multi-fidelity model \cite{li2020multi}. An approach that involves training a physics-constrained generative model, conditioned on low-fidelity snapshots to produce solutions that are higher-fidelity and higher-resolution, has also been proposed \cite{Geneva2020}. Similarly, a bi-fidelity formulation of variational auto-encoders, trained on low- and high-fidelity data has been developed to generate bi-fidelity approximations of quantities of interests and estimate their uncertainty \cite{CHENG2024116793}.

Another class of methods, suitable when the response of the system consists of a high-dimensional vector, first performs order-reduction using low-fidelity data, and then injects accuracy using high-fidelity data in a reduced-dimensional latent space. This is accomplished by computing the low- and high-fidelity Proper Orthogonal Decomposition (POD) manifolds, aligning them with each other, and replacing the low-fidelity POD modes with their high-fidelity counterparts \cite{perron2020development}. Similar multi-fidelity surrogate models have been developed by first solving a subset selection problem to construct a surrogate model of the low-fidelity response in terms of a few important snapshots, then generating the high-fidelity counterparts of the important snapshots, and finally using these in the multi-fidelity model \cite{narayan2014stochastic, keshavarzzadeh2019parametric, pinti2022multi}.

In this work we present a novel multi-fidelity approach based on the spectral properties of the graph Laplacian. 
The starting point of this method is the generation of a large number of low-fidelity data. Thereafter, each data point is treated as a node of a weighted graph, where the weights on the edges are determined by the distance between the nodes. A normalized graph Laplacian and its eigendecomposition are then evaluated, and the nodes are embedded in the eigenfunction space. The data are clustered and the points closest to the clusters centroids are identified. Thereafter, the high-fidelity counterparts of these points are evaluated.
The problem of computing a multi-fidelity approximation using the low- and high-fidelity data is posed as a Bayesian inference problem. Within this framework, the prior probability distribution of the multi-fidelity data is defined to be a normal distribution centered at the low-fidelity points, with a covariance from the inverse of the graph Laplacian.
The likelihood term is derived by assuming a simple additive Gaussian model for the error in the high-fidelity data.
Given these choices, it is shown that the multi-fidelity data is also drawn from a multi-variate Gaussian distribution whose mean can be determined by solving a linear system of equations. Efficient algorithms for solving this linear system as well as determining the covariance matrix for the multi-fidelity data are described. The new multi-fidelity method is then applied to several problems drawn from linear elasticity, Darcy flow, and Navier-Stokes equations. It is used to improve the accuracy of vector quantities of interest, and spatial fields defined in one and two spatial dimensions. In each case it is observed that a small number of high-fidelity data leads to a multi-fidelity dataset that is significantly more accurate than its low-fidelity data counterpart.

The approach described in this study has strong connections with semi-supervised classification algorithms on graphs \cite{bertozzi2018uncertainty, slepcev2019analysis, belkin2004regularization, bertozzi2021posterior, dunlop2020large}, \new{especially with active learning\footnote{\new{Active learning extends semi-supervised learning by selecting data points at which to query a classification oracle, which is analogous to querying a high-fidelity model in multi-fidelity modeling.}} on graphs \cite{pmlr-v22-ji12,pmlr-v40-Dasarathy15,MillerCalder23,bhusal2024maladymulticlassactivelearning},}  and relies on theoretical results on consistency of graph-based methods in the limit of infinite data \cite{hoffmann2020consistency, trillos2018variational}. 
The results presented in this study may be considered as a probabilistic extension of the ideas developed for the SpecMF method \cite{pinti2023graph}. They also improve on those ideas in several ways. First, in addition to providing a point estimate for the multi-fidelity data, they also provide a measure of uncertainty associated with it. This measure may be used for downstream tasks like adaptive or active learning, or as demonstrated in this manuscript, it may be used to determine an important hyperparamater of the method, thereby obviating the need of validation data.  Second, through the use of conjugate priors and likelihood terms we ensure that the point estimate is  obtained though the solution of a linear system of equations, as opposed to a nonlinear system that is solved in the SpecMF method. Finally, the proposed approach is applied to problems where each data point is a discrete representation of spatial fields defined on one- and two-dimensional grids, with dimensions as high as $\mathcal{O}(10^4)$.

The layout of the remainder of this manuscript is as follows. In Section \ref{sec:problem}, we define the multi-fidelity problem and describe the proposed methodology. In Section \ref{sec:evaluating}, we detail efficient numerical techniques to approximate the graph Laplacian and evaluate the posterior distribution of the multi-fidelity estimates. In Section \ref{sec:analysis}, we provide a theoretical result on convergent regularization. 
In Section \ref{sec:numerical}, we present a comprehensive set of numerical experiments that quantify the performance of the method, and, finally, in Section \ref{sec:conclcusions} we conclude with final remarks. \ref{app:GLnormalization} discusses generalizations of the employed graph Laplacian framework for different choices of normalizations. 
\new{
In \ref{app:omega_cond_grad} we derive an explicit expression for the gradient of a loss function used to find a hyper-parameter of the method.
}
\section{Problem formulation}
\label{sec:problem}
We are interested in solving the problem of determining a quantity of interest $\utrue \in \mathbb{R}^{\ndim}$ as a function of input parameters $\param \in \mathbb{R}^{\nparam}$ sampled from a distribution $P_{\param}(\param)$. We assume that we do not have direct access to the underlying true values $\utrue$, but for each input parameter value $\pari$ we have access to two models to estimate it. 
One of these is the high-fidelity model whose estimate, denoted by $\uhati$, is accurate but computationally expensive. The other is the low-fidelity model, whose estimate, denoted by $\ubari$, is  less accurate but computationally inexpensive.

Our goal is to combine a large set of low-fidelity data points $\bar{\mathcal{D}} = \{ \ubari \}_{i=1}^{\nbar}$, with a smaller set of select high-fidelity points, $\hat{\mathcal{D}} = \{ \uhati \}_{i=1}^{\nhat}$, with $\nhat \ll \nbar$, to generate a multi-fidelity estimate for all points in $\bar{\mathcal{D}}$.
Following the SpecMF method \cite{pinti2023graph}, our approach exploits the spectral properties of the graph Laplacian constructed from the low-fidelity data points to  identify the key input parameter values at which to acquire the high-fidelity data, and to combine the low- and high-fidelity data to construct a probabilistic multi-fidelity estimate. This approach leverages the graph Laplacian computed using the low-fidelity data to define a prior distribution for the multi-fidelity estimates. Thereafter, it utilizes the few select high-fidelity data points to construct a likelihood term. The prior and the likelihood terms are chosen to be Gaussian conjugate pairs, which yields a Gaussian distribution for the posterior density as well. In what follows, the aforementioned steps are discussed in detail.

\subsection{Low-fidelity data and graph Laplacian}
As a starting point, $\nbar \gg 1$ instances of parameters $\pari$ are sampled from the density $P_{\param}(\param)$ and for each parameter value $\pari$ the low-fidelity model is deployed to generate the low-fidelity data $\ubari \in \mathbb{R}^{\ndim}$. These points are used to assemble the low-fidelity dataset $\bar{\mathcal{D}}$.

Depending on the application, different techniques may be used to normalize the low-fidelity data if needed. A first approach is to normalize every component of $\ubari$ to have zero-mean and unit standard deviation. That is,

\begin{equation}
    \bar{u}^{(i)}_k \leftarrow \frac{\bar{u}^{(i)}_k - \mathbb{E}(\bar{u}_k)}{\sqrt{\mathbb{VAR}(\bar{u}_k)}},\qquad i \in \{ 1, \dots, \nbar\}, \quad k\in \{1,\,\dots,\,\ndim\},
    \label{eq:norm-qoi}
\end{equation}
where $\mathbb{E}(\cdot)$ and $\mathbb{VAR}(\cdot)$ denote the mean and variance, respectively, computed over the low-fidelity dataset. This is more suitable for cases where $\ubari$ is a set of physical quantities of interest with different scales.

If $\ubari$ represents a discretized field, a more appropriate alternative is to normalize each instance as
\begin{equation}
    \ubari \leftarrow \frac{\ubari}{\| \ubari \|},\qquad  i\in \{1,\,\dots,\,\nbar\},
    \label{eq:norm-fields}
\end{equation}
where $\|\cdot\|$ is a suitable norm. We note that once a normalization procedure for the $i^\text{th}$ data point is determined, it is applied to both the low- and high-fidelity data (Section \ref{sec:selection-hf}).

Thereafter, a complete weighted graph $\mathcal{G}$ with nodes $\mathcal{V}$ and edges $\mathcal{E}$ is constructed from the low-fidelity data. Specifically, the graph has $\nbar$ nodes $i \in \mathcal{V} := \{1,\,2,\,\dots,\,\nbar\}$, equipped with attributes given by the corresponding low-fidelity data points, $\ubari$. The weight $W_{ij}$ associated with each edge $E_{ij} \in \mathcal{E}$, which measures the similarity between the nodes $i$ and $j$, is computed via a monotonically decreasing function of the Euclidean distance between $\ubari$ and $\ubarj$. 
In this work, we use a Gaussian kernel, and the resulting weights are given by
\begin{equation}
    W_{ij} := \begin{cases}{\color{black}\exp\left(-\frac{\|\ubari - \ubarj\|_2^2}{\ell_{ij}^2}\right)}, & \text{if } i \neq j,\\
    0, & \text{if } i = j.
    \end{cases}
\end{equation}
This leads to a fully connected graph, but other choices leading to sparser graphs are also common. Choosing the weights $W_{ij}$ plays an important role in the behavior of any graph Laplacian-based algorithm, an aspect we are not focusing on in this work. The scale $\ell_{ij}$ can be seen as an additional hyper-parameter, with a common choice being to set it to be constant for the entire graph, or it may be tuned. \new{We use the self-tuning approach of Zelnik-Manor and Perona \cite{zelnik2004self}: we set $\ell_{ij} := \sqrt{\ell_i\ell_j}$ with the local length-scale $\ell_i$ given by $\ell_i:=\|\ubari - \bar{\bm{u}}^{(i_R)}\|_2$, where $\bar{\bm{u}}^{(i_R)}$ is the $R^\text{th}$ nearest low-fidelity data point to $\ubari$.\footnote{\new{We take $R = 7$, as in Zelnik-Manor and Perona \cite{zelnik2004self}.}  } This approach is more accommodating of data which may exhibit multiple scales. }
We remark that all our analysis also applies for any other choice of weights satisfying $W_{ij}=W_{ji}$, $W_{ij}\ge 0$, and $W_{ii}=0$.

The weights $W_{ij}$ form the components of the adjacency matrix $\adj \in \mathbb{R}^{N \times N}$, which is used to compute the (diagonal) degree matrix $\bm{D}$, 
\begin{equation}
    D_{ii} := \sum_{j = 1}^{N} W_{ij},
\end{equation}
and a family of graph Laplacians,
\begin{eqnarray}
\label{eq:GLdef}
\gl := \bm{D}^{-p} (\bm{D} - \bm{W}) \bm{D}^{-q}.
\end{eqnarray}
Different choices of $p$ and $q$ result in different normalizations of the graph Laplacian \cite{belkin2006convergence, COIFMAN20065, Hoffmann2019SpectralAO}. For simplicity, in the body of this work we will assume that $p=q$, i.e., $\gl$ is symmetrically normalized; however, our framework extends to the $p \neq q$ case (for details see \ref{app:GLnormalization}). The graph Laplacian embeds many important properties of the structure and topology of the graph. In particular, the eigenvectors of $\gl$ corresponding to small eigenvalues provide an embedding of the graph that promotes the clustering of vertices that are strongly connected \cite{pinti2023graph,belkin2001laplacian}.

\subsection{Construction of prior density}
We assume that the low-fidelity data $\bar{\mathcal{D}}$ represents an approximation of the true data, and denote the difference between the (unknown) true data $ \utrue^{(i)}$ and the observed low-fidelity data $\ubari$ by
\begin{equation}
    \phii := \utrue^{(i)} - \ubari, \qquad  i\in \{1,\,\dots,\,\nbar\}. 
\end{equation}
This relation is shown graphically in Figure \ref{fig:hf_vs_true_phi}. In the equation above $\phii$ may be interpreted as the displacement vector from a given low-fidelity data point to its true counterpart. Our goal is to write a prior probability distribution for the displacements $\phii \in \mathbb{R}^{\ndim}$. In order to accomplish this we define the displacement matrix
\begin{equation}
    \phimat  :=\begin{pmatrix}
\bmphi^{(1)}\\
\vdots\\
\bmphi^{(\nbar)}
\end{pmatrix}=\begin{pmatrix}\bmphi_{1} & \cdots & \bmphi_{\ndim}\end{pmatrix}\in\mathbb{R}^{\nbar \times \ndim},
\end{equation}
where a superscript on $\bm{\phi}$ indicates a row vector, while a subscript indicates a column vector. We note that while $\phii$ represents the displacement vector for the $i^\text{th}$ data point, the vector $\bmphi_{i}$ represents the displacement field for the $i^\text{th}$ component of the data.

For any vector $\bm{v} \in \mathbb{R}^\nbar$, the quantity $\frac{1}{2} \bm{v}^T \gl \bm{v}$ is referred to the as the Dirichlet energy associated with the graph Laplacian $\gl$. 
We construct a prior for $\phimat$ by stipulating that displacement fields that yield large values of the Dirichlet energy are less likely. In particular, we set
\begin{equation}
    \begin{split}
            p(\phimat) 
& \propto \exp\left(-\sum_{m=1}^{\ndim} \frac{\omega}{2} 
\bmphi_{m}^{T} \left(\gl + \tau \idnbar \right)^{\beta} \bmphi_{m} \right) \\
& = \exp \left( -\frac{\omega}{2} 
\langle \phimat, \left(\gl + \tau \idnbar \right)^{\beta} \phimat \rangle_F \right),
    \end{split}
\label{prior}
\end{equation}
where $\idnbar \in \mathbb{R}^{N \times N}$ is the identity matrix, $\langle \bm{A},\bm{B} \rangle_F := \operatorname{tr}(\bm{A}^T\bm{B}) = \sum_{ij} A_{ij}B_{ij}$ is the Frobenius inner product, and $\omega, \, \tau, \, \beta > 0$ are hyper-parameters. 
More precisely, each displacement component $\bmphi_{m}$ is independently and identically Gaussian distributed with zero mean and with covariance $\frac{1}{\omega} \left(\gl + \tau \idnbar \right)^{-\beta}$. The parameter $\tau>0$ is chosen suitably small in practice, and has been introduced to allow for this interpretation of the prior since the graph Laplacian $\gl$ itself is not invertible.\footnote{$\bmphi_{m}=\bm{D}^{q} \mathbf{1}_N$ is an eigenvector of $\gl$ with eigenvalue zero.} An alternative to this approach is to only consider components $\bmphi_{m}$ that are orthogonal to the kernel of $\gl$, thus avoiding the introduction of the additional parameter $\tau$ as done in \cite{Bertozzi2018}. The parameters $\omega$ (regularization strength) and $\beta$ (regularization exponent) are both used to control the regularity of the MAP estimator, and will be described in detail in Subsections~\ref{sec:prior} and \ref{sec:hyperpara}.

\subsubsection{Interpretation of the prior density}
\label{sec:prior}
Since we have assumed that $p=q$, the graph Laplacian $\gl$ is a symmetric positive semi-definite matrix, and so permits an eigenvalue decomposition of the form $\bm{L} = \bm{\Psi \Lambda \Psi^T}$, where $\bm{\Lambda}$ is the diagonal matrix of eigenvalues, $\bm{\Psi}$ is a matrix whose columns are the eigenvectors, and $\bm{\Psi}^T \bm{\Psi} = \bm{\Psi} \bm{\Psi}^T = \idnbar$. 

The spectral decomposition of the graph Laplacian has the useful property that when the underlying data is clearly separated into clusters, the eigenfunctions corresponding to the smallest eigenvalues tend to vary smoothly over the clusters on a coarser spatial scale, when compared with eigenfunctions corresponding to larger eigenvalues.
We utilize this observation to better understand the prior distribution in (\ref{prior}). In particular, we express $\phimat$ as a linear combination of the eigenvectors via the coefficients $\bm{A}=[A_{ij}] \in \mathbb{R}^{\nbar \times \ndim}$, that is $\phimat = \bm{\Psi} \bm{A}$. Thereafter, we substitute this in  (\ref{prior}) to obtain
\begin{equation}
\begin{split}
    p(\bm{A}) 
& \propto \exp \left( -\frac{\omega}{2} 
\mathrm{tr} \left( \bm{A}^T \bm{\Psi}^T \bm{\Psi} \left(\bm{\Lambda}+\tau \idnbar \right)^{\beta} {\bm{\Psi}}^T \bm{\Psi}\bm{A} \right) 
\right) \\
& =\exp \left( -\frac{\omega}{2} \mathrm{tr} \left( \bm{A}^T \left(\bm{\Lambda}+\tau \idnbar \right)^{\beta} \bm{A} \right) \right)  \\
& =\exp \left( -\frac{\omega \tau^{\beta}}{2} \sum_{j=1}^{\ndim} \sum_{i=1}^{{N}}  A_{ij}^2 \left( 1 + \frac{\lambda_i}{\tau} \right)^{\beta}  \right). 
\end{split}
\label{eq:pA}
\end{equation}
The final expression in the equation above makes it clear that for every component of the displacement field (denoted by the index $j$ in the sum above), the prior promotes coefficients of eigenfunctions corresponding to smaller eigenvalues. When the data is arranged in clusters, these eigenfunctions have little variation over the clusters, which in turn means that displacement fields that are constant over clusters are considered more likely by the prior density. In other words, while moving the low-fidelity data points to their multi-fidelity location, the prior deems that it is more likely that points within a cluster will move together. When the data is not arranged in clusters, or for data points within each cluster, the eigenfunctions corresponding to smaller eigenvalues are those that tend to vary more smoothly. This in turn implies that displacements fields that vary smoothly are considered more likely. 
That is, the prior ensures that low-fidelity data points that are close to each other will have multi-fidelity approximations that continue to be close. In summary, the prior ensures that multi-fidelity approximations that are consistent with the structure of the low-fidelity data are more likely.

The analysis above also provides an intuitive interpretation of the hyper-parameters. The parameter $\tau$ scales with the eigenvalues, so that if set to be equal to the smallest non-zero eigenvalue, the problem is equivalent to considering a scaled spectrum with smallest non-zero eigenvalue equal to 1. The regularization exponent $\beta$ controls the extent to which higher-order eigenvalues are penalized; larger values of $\beta$ make the contribution from the eigenfunctions corresponding to large eigenvalues to the displacement vector less likely. Finally, $\omega$ controls the strength of the contribution of the prior to the posterior distribution, relative to the contribution of the likelihood term. By re-parameterizing the regularization strength $\omega = \kappa \tau^{-\beta}$,
\begin{equation}
    p(\bm{A}) \propto \exp \left( -\frac{\kappa}{2} \sum_{i=1}^{{N}} \sum_{j=1}^{{N}} A_{ij}^2 \left( 1 + \frac{\lambda_i}{\tau} \right)^{\beta}  \right)
\end{equation}
we note that the strength of the prior depends only on $\kappa$ and is independent of $\tau$ and $\beta$. This expression is convenient when comparing priors for different values of hyper-parameters.

\subsection{Selection of high-fidelity data}
\label{sec:selection-hf}
In this section, we describe the policy used to select $\nhat \ll \nbar$ high-fidelity data points to acquire. This is accomplished by \new{performing spectral clustering} on the graph constructed from the low-fidelity data to determine $\nhat$ clusters and their corresponding centroids. Thereafter, the data points closest to each of these centroids are determined, and the corresponding high-fidelity data is acquired.
The logic \new{of using spectral clustering is that it is consistent with our choice of prior, based on a graph Laplacian regularizer. Whereas, the choice of} selecting \new{the centroids} is driven by the observation that points that belong to a given cluster will tend to move together from their low-fidelity coordinates to their high-fidelity coordinates. Thus the displacement vector for the centroid is likely to be representative of the displacement vector for most of the points in the cluster. The steps in this process are:

\begin{enumerate}
    \item Compute the low-lying eigenfunctions of the graph Laplacian, $\bm{\psi}^{(m)}$, for each $m\in \{1,\,\dots,\,\nhat\}$. 
    \item Compute the coordinates of every low-fidelity data point $\ubari$ in the eigenfunction space. That is, compute $ \bm{\xi}^{(i)} := [\psi_i^{(1)},\dots,\psi_i^{(\nhat)}] $ for each $ i \in \{ 1, \dots, \nbar\}$. 
    \item Perform clustering on the points $\{\bm{\xi}^{(i)}\}_{i = 1}^{\nbar}$ to find $\nhat$ clusters. \new{For this task, standard clustering techniques such as $k$-means or DBSCAN can be used}.\footnote{\new{In this work, we assume that the number $\nhat$ of computed eigenfunctions equals the number of desired high-fidelity points, and therefore of clusters. However, if one desired to decouple these quantities, then an advantage of DBSCAN is that it does not require a pre-specified number of clusters, unlike $k$-means.}}
    \item For each cluster, determine the centroid and the low-fidelity data point closest to it.
    \item Re-index the low-fidelity data set $\bar{\mathcal{D}}$ and the corresponding input parameters so that the points identified above correspond to the first $\nhat$ points. 
    \item Compute the high-fidelity data at the parameter values corresponding to these points, and assemble the data set $\hat{\mathcal{D}} := \left\{\uhati \right\}_{i=1}^{\nhat}$. Note that the elements of $\hat{\mathcal{D}}$ are the high-fidelity counterparts of the first $\nhat$ elements of $\bar{\mathcal{D}}$.
    \item Scale the high-fidelity data by the same procedure used to normalize the low-fidelity data.
\end{enumerate}

\subsection{Construction of the likelihood} 
We now wish to update the prior density for the displacement matrix based on the newly acquired high-fidelity measurements. This step is performed via a likelihood term within a Bayesian update. The construction of the likelihood term is described next.

From the high-fidelity data $\hat{\mathcal{D}} = \{\uhati \}_{i=1}^{\nhat}$, we construct the matrix of low-to-high-fidelity displacements $\phihati := \uhati - \ubari$, $i\in \{1,\,\dots,\,\nhat\}$ (see Figure \ref{fig:hf_vs_true_phi}),
\begin{equation}
    \phihatm  :=\begin{pmatrix}
\phihat^{(1)}\\
\vdots\\
\phihat^{(\nhat)}
\end{pmatrix} = \begin{pmatrix}
\phihat_1 & \cdots & \phihat_D \end{pmatrix} \in \mathbb{R}^{\nhat \times \ndim}.
\end{equation}

\begin{figure}
    \centering
    \includegraphics[width=0.6\linewidth]{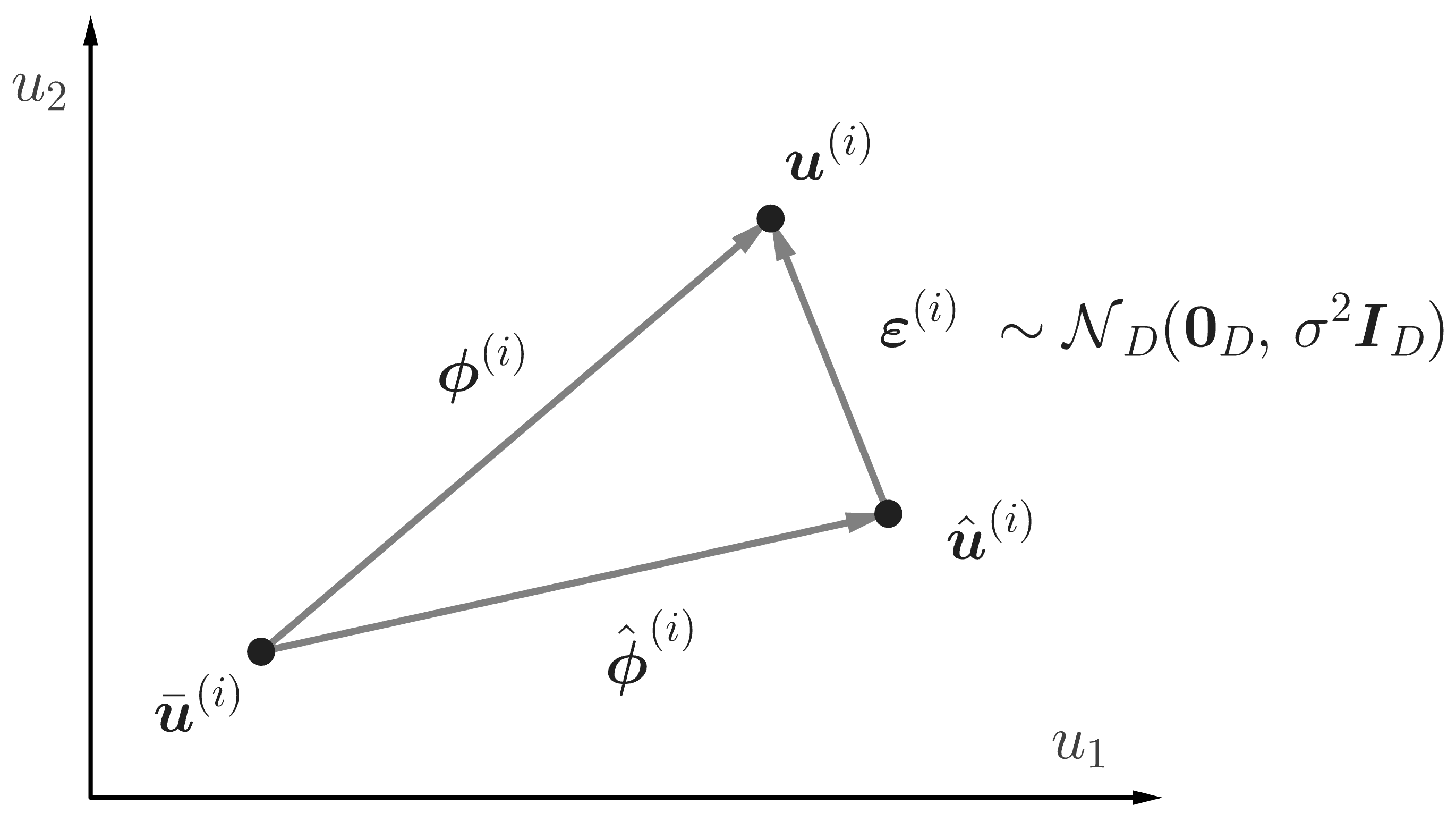}
    \caption{Schematic representation of the low-fidelity approximation $\ubari$, connected to the high-fidelity approximation $\uhati$ and the true data point $\utruei$ through their respective displacement vectors, $\phihati$ and $\bm{\phi}^{(i)}$, for $i \in \{1,\, \dots,\, M\}$.}
    \label{fig:hf_vs_true_phi}
\end{figure}

We assume that the error in the high-fidelity data is additive and can be modeled with a  multivariate normal distribution with zero mean and variance $\sigma^2 \boldsymbol{I}_D$, i.e.,
\begin{align}
\uhati = \utruei + \bm{\varepsilon}^{(i)},\qquad i\in \{1,\,\dots,\,\nhat\},
\end{align}
with $\bm \varepsilon^{(i)} \sim \mathcal{N}_{\ndim}(\boldsymbol{0}_D,\,\sigma^2 \boldsymbol{I}_D)$ \textit{i.i.d}. 
\begin{rmk}
    One could also consider correlations between the components of the error in the high-fidelity data by considering instead $\bm \varepsilon^{(i)} \sim \mathcal{N}_{\ndim}(\boldsymbol{0}_D,\,\boldsymbol{C}_D)$ for some non-diagonal covariance matrix $\boldsymbol{C}_D$. We will not consider this case in this paper.
\end{rmk}
Subtracting $\ubari$ from both sides of the equation above we arrive at,
\begin{align}
 \phihati = \phii + \bm{\varepsilon}^{(i)},\qquad i\in \{1,\,\dots,\,\nhat\}.
\end{align}
This implies that the likelihood of observing the measurements $\phihatm$ given the displacement matrix $\phimat$ is
\begin{align} 
p(\phihatm | \phimat)  \propto \exp \left( -\frac{\|\phihatm - \pnhat \phimat \|_{F}^{2}}{2\sigma^{2}}\right),
\label{likelihood}
\end{align}
where $\pnhat \in \mathbb{R}^{\nhat \times \nbar}$ is the matrix that extracts the first $\nhat$ rows from $\phimat$, and is given by, 
\begin{equation}
\pnhat := (\idnhat \; \bm{0}_{\nhat \times (\nbar - \nhat)}) 
\in \mathbb{R}^{\nhat \times N}.
\end{equation}

\subsection{Definition of the posterior density}
The posterior distribution of $\phimat$ is proportional to the product of the prior and the likelihood, that is
\begin{equation}
\begin{split}
    p(\phimat | \phihatm)
& \propto p(\phihatm | \phimat) p(\phimat)  \\
& \propto \exp \left( -\frac{\|\phihatm - \pnhat \phimat \|_{F}^{2}}{2\sigma^{2}} -\frac{\omega}{2} \langle \phimat, \left(\gl +\tau \idnbar \right)^{\beta} \phimat\rangle_F \right) \\
& =\prod_{m = 1}^{\ndim} \exp\left( -\frac{\|\hat{\bmphi}_{m} - \pnhat \bmphi_{m} \|_{2}^{2}}{2\sigma^{2}} - \frac{\omega}{2} \bmphi_{m}^T\left(\gl  + \tau \idnbar \right)^{\beta} \bmphi_{m} \right),
\end{split}
\label{eq:posterior}
\end{equation}
where we have made use of (\ref{prior}) and (\ref{likelihood}). Since both the prior and the likelihood are normal distributions, we can write the expression above as
\begin{equation}\label{eq:posteriorC}
\begin{split}
    p(\phimat | \phihatm)
& \propto \exp \left( - \frac{1}{2} \langle \phimat - \phimapm, \phicov^{-1} (\phimat - \phimapm) \rangle_F \right) \\
& = \prod_{m=1}^{\ndim} \exp \left( -\frac{1}{2} \left( \bmphi_m - \bmphi^*_m \right)^{T} \cdot \phicov^{-1} \left( \bmphi_m - \bmphi^*_m \right) \right),
\end{split}
\end{equation}
where
\begin{align}
\phicov &= \left( \frac{1}{\sigma^2} \pnhat^T \pnhat + \omega (\gl + \tau \idnbar)^{\beta} \right)^{-1}, \label{covariance} \\
\phimapm & =\frac{1}{\sigma^{2}} \phicov \pnhat^T \phihatm.
\label{MAP-estimate} 
\end{align}
Hence, $\bmphi^*_m =\frac{1}{\sigma^{2}} \phicov \pnhat^T \phihat_m$.
That is, the posterior distribution is also normal, and the mean for the $m^\text{th}$ component of the low-fidelity displacement vector is given by $\bmphi^*_m$ and the covariance is given by $\phicov$. Further, since for a normal distribution the mean and the mode are the same, this means that $\bmphi^*_m$ is also the maximum \textit{a posteriori} (MAP) estimate for the $m^\text{th}$ component of the low-fidelity displacement vector. 
This implies that the estimated posterior distribution for the $m^\text{th}$ component of the true data is also multivariate normal. The mode of this distribution is given by $\ubar_{m} + \bmphi^*_m$. We define this quantity to be our multi-fidelity estimate. We note that the estimated covariance for all components of the true data is the same and is given by the matrix $\phicov$.

It is instructive to write (\ref{MAP-estimate}) as a solution to a linear system of equations. This yields,
\begin{equation}
    \label{eq:MAP}
\left( \frac{1}{\sigma^2} \pnhat^T \pnhat + \omega (\gl + \tau \idnbar)^\beta\right)\phimapm = \frac{1}{\sigma^2} \pnhat^T \phihatm.
\end{equation}

\subsection{Selection of the hyper-parameters}\label{sec:hyperpara}
The method described requires the evaluation of the three hyper-parameters $\tau$, $\omega$ and $\beta$. The parameter $\tau>0$ was artificially introduced into the problem to guarantee invertibility of $\gl + \tau \idnbar$, and therefore allows the interpretation of our set-up as a Bayesian inverse problem with \eqref{covariance} providing the expression of the covariance matrix of the posterior distribution. As described in Section \ref{sec:prior}, $\tau$ is set equal to the smallest non-zero eigenvalue of the graph Laplacian.

Next, we comment on the role of the regularization exponent $\beta$. When maximizing the posterior distribution $p(\phimat | \phihatm)$ over $\phimat$, higher values of $\beta$ enforce that the resulting $\bm\phi_{m}$ has smooth variation within clusters also for higher derivatives; in short, it enforces control of derivatives of order $\beta$ and may be selected $\beta\ge 1$ \cite{hoffmann2020consistency}. Given this, we select $\beta=2$ and note that this choice is adequate in all our numerical experiments. 
 
Finally, $\omega$ determines the strength of the prior distribution with respect to the likelihood; we consider it a regularization strength since the prior distribution plays the role of a regularizer when maximizing the posterior distribution. The parameter $\omega$ can be determined by requiring the average standard deviation of the multi-fidelity estimates to be greater than the standard deviation of the high-fidelity data, $\sigma$. That is, select $\omega$ such that
\begin{equation}
    \new{\frac{1}{\nbar} \sum_{i=1}^{\nbar} \sqrt{C_{ii}} = r \sigma,}
\label{eq:omega_condition}
\end{equation}
for some $r > 1$. This guarantees that the confidence of the multi-fidelity model is not greater than the one of the high-fidelity model. In our numerical experiments we have observed that that a value of $r = 3$ leads to good results for all cases. Remarkably, the selection of the hyper-parameters using the approach described above does not require knowledge of any additional high-fidelity data. All the high-fidelity data is used to estimate the mode of the true data points. 
\new{Additionally, in \ref{app:omega_cond_grad} we show that the gradient of the function
\begin{equation}
    \mathcal{J}(\omega) := \left( \frac{1}{\nbar} \sum_{i=1}^{\nbar} \sqrt{C_{ii}(\omega)} - r \sigma \right)^2,
    \label{eq:omega_loss}
\end{equation}
which needs to be minimized to solve (\ref{eq:omega_condition}), can be found in closed form.
}

\section{Evaluating the posterior distribution}
\label{sec:evaluating}
In this section we describe two computational methods for determining $\phimapm $ and $\phicov$, and thereby completely characterizing the posterior distribution. The first method relies on computing the spectral decomposition of the graph Laplacian, whereas the second relies on constructing a low-rank approximation of the inverse of the covariance matrix. 

\subsection{Expansion in a truncated eigenfunction basis}
\label{sec:truncation}
This method utilizes the facts that: 
\begin{enumerate}
    \item The eigenfunctions of the graph Laplacian form a complete basis that can be used to represent the displacement components.
    \item Once this basis is used, the prior penalizes the contributions from eigenfunctions corresponding to large eigenvalues. 
\end{enumerate}
Therefore, a reasonable approximation is to represent the displacement components using a truncated basis set comprising of eigenfunctions with small eigenvalues.

We compute the low-lying spectrum  of the graph Laplacian $\gl$. That is, we compute the eigenpairs
$(\lambda_k,\,\boldsymbol{\psi}_k)$ for each $k\in \{1,\,\dots,\,K\}$, where $\boldsymbol{\psi}_k \in \mathbb{R}^{\nbar}$ and $K \ll \nbar$ indicates a cutoff. This can be done as follows: First, let 
\begin{equation}
    a:=2 \max_{i\in\{1,...,N\}} D_{ii}^{1-p-q}.
\end{equation}

Then it is a standard result that the eigenvalues of $\gl$ lie in $[0,a]$ \cite[Lemma 2.5(f)]{vanGennip14}.\footnote{To apply this Lemma, we use that $\gl$ is similar to $\degmat^{-(p+q)}(\degmat - \adj)$.} Next, compute the $K$ leading eigenvalues $ \lambda'_k$ and eigenvectors ${ \boldsymbol{\psi}}'_k$ of the symmetric positive semi-definite matrix $a \idnbar - \gl$.
Finally, the low-lying eigenvectors of $\gl$ are $\boldsymbol{\psi}_k = \boldsymbol{\psi}'_k$, with corresponding eigenvalues $\lambda_k = a -  \lambda'_k$.

\new{
\begin{rmk}
Choosing smaller values of the cutoff value $K$ reduces the computational budget of the method (see \Cref{tb:trunc}). In practice, $K$ is a parameter that is often fixed using heuristics. In preliminary experiments (a more thorough investigation is the subject of ongoing work) we observed non-monotonic error dependence on $K$, often seeing a double-descent phenomenon. It is thus hard to recommend principled heuristics for $K$, except to note that when we did observe a second descent in error as a function of $K$, it roughly was for $K > M$, the number of high-fidelity points. Alternatively, $K$ can be tuned as a parameter using a holdout validation set, wherein a $K$ is sought (within computational budget) which minimizes validation error.
\end{rmk}
}

We store the eigenfunctions and eigenvalues in the matrices:
\begin{align}
   \efunm_K &= \begin{pmatrix}
   \boldsymbol{\psi}_1 & \cdots & \boldsymbol{\psi}_K \end{pmatrix} \in \mathbb{R}^{\nbar \times K}, \\
    \evalm_K &= \operatorname{diag}\left(\lambda_1,\,\dots,\,\lambda_K\right) \in \mathbb{R}^{K \times K},
\end{align}
where $\efunm_K^T \efunm_K = \bm{I}_K$. 

Thereafter, we express $\phimat$ as a truncated expansion of the eigenfunctions via the truncated coefficients matrix $\coeffm_K = [A_{ij}] \in \mathbb{R} ^{K \times \ndim}$,

\begin{equation}
    \phimat = \efunm_K \coeffm_K.
\end{equation}

Hence, each component of the displacement $\bmphi_m \in \mathbb R ^ {\nbar}$, $m\in\{1,\dots,\,\ndim\}$, may be written as
\begin{equation}
    \bmphi_m = \sum_{k=1}^K A_{km} \boldsymbol{\psi}_k.
\end{equation}

Using this expansion in the posterior distribution for the displacement (\ref{eq:posterior}), we arrive at an expression for the  distribution for the coefficients $\coeffm_K$,

\resizebox{0.995\textwidth}{!}{
\begin{minipage}{\textwidth}
\begin{align}
p(\coeffm_K|\phihatm)
& \propto \exp \left (-\frac{\|\phihatm -\pnhat\boldsymbol{\Psi}_K \coeffm_K\|_{F}^{2}}{2\sigma^{2}} - \frac{\omega}{2} \operatorname{tr}\left(\coeffm_K^T \boldsymbol{\Psi}_K^T \left(\gl + \tau \idnbar \right)^{\beta} \boldsymbol{\Psi}_K \coeffm_K \right)\right) \nonumber \\
& = \exp \left (-\frac{\|\phihatm -\pnhat\boldsymbol{\Psi}_K \coeffm_K\|_{F}^{2}}{2\sigma^{2}} - \frac{\omega}{2} \operatorname{tr}\left(\coeffm_K^T \left(\evalm_K + \tau \bm{I}_K \right)^{\beta} \coeffm_K \right)\right),\label{eq:pAPhihat}
\end{align}
\end{minipage}
}
where we have used that, since $\efunm_K^T\left(\gl + \tau \idnbar\right)\boldsymbol{\Psi}_K = \evalm_K + \tau \bm{I}_K$, it follows that $\efunm_K^T \left(\gl + \tau \idnbar \right)^{\beta} \boldsymbol{\Psi}_K = \left(\evalm_K + \tau \bm{I}_K \right)^{\beta}$.\footnote{Recall that $\gl + \tau \idnbar = \bm{\Psi}\bm{\Lambda}\efunm^T + \tau \idnbar =  \bm{\Psi}(\bm{\Lambda}+\tau \idnbar) \efunm^T  $. It directly follows that $(\gl + \tau \idnbar)^\beta = \bm{\Psi}(\bm{\Lambda}+\tau \idnbar)^\beta \efunm^T$. Then since 
$\efunm_K^T\efunm = 
\begin{pmatrix}
    \bm{I}_K & \mathbf{0}_{K\times(N-K)}
\end{pmatrix}$ 
and 
$\efunm^T\efunm_K = 
\begin{pmatrix}
    \bm{I}_K & \mathbf{0}_{(N-K)\times K}
\end{pmatrix}^T$ 
by the orthonormality of the eigenvectors, the result follows.}

Recognizing that the distribution for $\coeffm_K$ is the product of two normal distributions and is therefore itself a normal distribution, we may write the distribution \cref{eq:pAPhihat} as, 
\begin{align}
p(\coeffm_K|\phihatm)
& \propto \exp \left( \frac{1}{2} \langle \coeffm_K - \coeffm_K^*, \phicov_{A_K}^{-1} (\coeffm_K - \coeffm_K^*) \rangle_F \right) \nonumber \\
& = \prod_{m=1}^{\ndim} \exp \left( -\frac{1}{2} \left(\coeff_{m} - \coeff^*_{m} \right)^{T} \cdot \phicov_{A_K}^{-1} \left(\coeff_{m} - \coeff^*_{m} \right)\right),
\end{align}
where $\coeff_{m}$ is the $m^\text{th}$ column of $\coeffm_K$. Further for the posterior distribution, the mean, $\coeffm_K^* \in \mathbb{R} ^{K \times \ndim} $, and the covariance, $\phicov_{A_K} \in \mathbb{R} ^{K \times K}$, are given by
\begin{align}
\phicov_{A_K}^{-1} &= \frac{1}{\sigma^2} \left(\pnhat\efunm_K\right)^T \pnhat \efunm_K + \omega (\evalm_K + \tau \bm{I}_K)^{\beta}, \label{eq:cov_a} \\
\coeffm_K^* & =\frac{1}{\sigma^{2}} \phicov_{A_K} \left(\pnhat\efunm_K\right)^T \phihatm \label{eq:map_a}.
\end{align}

\subsubsection{Computational complexities}
The quantities to be computed and stored for this method, and the computational complexity of doing so, are summarized in \Cref{tb:trunc}. For the complexity of computing the $K$ leading eigenvalues and eigenvectors of $a\idnbar -\gl$ (which is the bottleneck for computing the truncated eigendecomposition), we assume that randomized block Krylov  methods \cite{MuscoMusco15} are used.

\begin{table}[ht]
\centering
\renewcommand{\arraystretch}{1.4}
\begin{tabular}{c l l}
\hline
\textsc{Quantity} & \textsc{Complexity in time} & \textsc{Complexity in space} \\ 
\hline
$\bm{\Psi}_K, \bm{\Lambda}_K$ & $\mathcal{O}(\nbar^2K\log \nbar)$ \cite{MuscoMusco15}& $\mathcal{O}(\nbar^2(\log \nbar)^2)$ \cite{MuscoMusco15} \\
\hline
$\pnhat\bm{\Psi}_K$ & $\mathcal{O}(MK)$ & $\mathcal{O}(MK)$ \\ 
$\phicov_{A_K}^{-1}$ & $\mathcal{O}(MK^2)$ & $\mathcal{O}(K^2)$ \\ 
$\phicov_{A_K}$ & $\mathcal{O}(K^3)$ & $\mathcal{O}(K^2)$ \\
$\coeffm_K^*$ & $\mathcal{O}(DMK + DK^2)$ & $\mathcal{O}(DK)$ \\ 
\hline
\end{tabular}
\caption{Computational complexity of the truncation method. \label{tb:trunc}}
\end{table}

\subsection{Low rank Nystr\"om approximation} \label{sec:Nys}
Recall that to compute $\phimapm$ we need to solve (\ref{eq:MAP}). 
In this section we describe a way to compute this in a way that is very efficient in $\nbar$, in the setting where $\mathcal{O}(\nbar^2)$ time and space complexity are both intractable. We can do this in the special case of $p=q=1/2$ by computing a low rank approximation of $(\gl + \tau \idnbar)^\beta$ via the Nystr\"om extension. 

\subsubsection{Nystr\"om-QR approximation}
Notice that, since $p=q=1/2$, $\gl:= \idnbar - \degmat^{-1/2}\adj \degmat^{-1/2}$. Choosing $X \subseteq \{1,...,\nbar\}$ at random (with perhaps the condition that $X$ contain some or all of the high-fidelity points) with $|X|=:K$, we can approximate $\adj$ via the Nystr\"om extension \cite{Nystrom30,FowlkesBelongieChungMalik04,BelongieFowlkesChungMalik02}:
\begin{equation}\label{eq:WNys}
    \adj \approx \adj(:,X)\adj(X,X)^{\dagger}\adj(X,:), 
\end{equation}
where we have used the \textsc{Matlab} notation $\adj(:,X) := [W_{ij}]_{i=1, j \in X}^{\nbar}$ and $\adj(X,X) := [W_{ij}]_{i, j \in X}$, and $\adj(X,X)^{\dagger}$ denotes the pseudoinverse of $\adj(X,X)$. Observe that this approximation of $\adj$ has at most rank $K$. In this subsection we will use this approximation to derive a low-rank approximation for $(\gl + \tau \idnbar)^\beta$, and thereby efficiently approximate solutions to \cref{eq:MAP} and matrix-vector products with $\phicov$.

\begin{rmk}
    The Nystr\"om extension is most well-behaved applied to symmetric positive semidefinite matrices, as if $\bm{S}$ is such a matrix then the condition number of $\bm{S}(X,X)$ is bounded above by that of $\bm{S}$. However, this is not satisfied by $\adj$, which is constructed to have zeroes on the diagonal (i.e., the graph has no self-loops) and hence has trace zero, and is therefore an indefinite matrix. In the indefinite case, positive and negative eigenvalues can combine to leave $\adj(X,X)$ with much smaller (in magnitude) non-zero eigenvalues than $\adj$ and correspondingly a much higher condition number. This case has seen recent attention in Nakatsukasa and Park \cite{NakatsukasaPark23}, 
    which recommends taking $K$ higher than one's desired rank $r$, and then truncating the spectrum of $\adj(X,X)$ to the $r$ largest (in magnitude) eigenvalues before computing the pseudoinverse, to control the condition number. We will ignore this refinement in the below, but it is only a minor modification to include.
\end{rmk}

Given \cref{eq:WNys}, we can therefore compute 
\begin{equation}
\degmat \approx \hat \degmat := \operatorname{diag}(\adj(:,X)\adj(X,X)^{\dagger}\adj(X,:)\onenbar)
\end{equation}

with $\onenbar \in \mathbb{R}^{\nbar}$ being a column vector of ones, and thus 
\begin{equation}
\gl \approx  \idnbar - \hat \degmat^{-\frac12} \adj(:,X)\adj(X,X)^{\dagger}\adj(X,:) \hat \degmat^{-\frac12}. 
\end{equation}

Next, we can extract an approximate low rank eigendecomposition of $\gl$ from this approximation by following a method of Bebendorf and Kunis \cite{BebendorfKunis09}, first recommended in the graph-based learning context by Alfke \textit{et al.} \cite{AlfkePottsStollVolkmer18}. 
We compute the thin QR factorization 
\begin{equation}
\bm{Q}\bm{R} = \hat \degmat^{-\frac12} \adj(:,X),
\end{equation}

where $\bm{Q} \in \mathbb{R}^{\nbar  \times K}$ has orthonormal columns and $\bm{R}  \in \mathbb{R}^{K \times K}$ is upper triangular. Then we compute the eigendecomposition 
\begin{equation}
\bm{R}\adj(X,X)^{\dagger}\bm{R}^T = \bm{\Gamma} \bm{\Sigma} \bm{\Gamma}^T,
\end{equation}

where  $\bm{\Gamma} \in \mathbb{R}^{K \times K}$ is orthogonal and $\bm{\Sigma}  \in \mathbb{R}^{K \times K}$ is diagonal. Finally, by computing $\tilde {\bm{U}}:= \bm{Q}\bm{\Gamma}\in \mathbb{R}^{\nbar  \times K}$ (which therefore has orthonormal columns) we have that 
\begin{equation}
\gl \approx \idnbar - \tilde {\bm{U}}\bm{\Sigma}\tilde{\bm{U}}^T,
\end{equation}

where $\tilde{\bm{U}}^T \tilde{\bm{U}}= \bm{I}_K$, where $\bm{I}_K$ is the $K\times K$ identity matrix. 
We can therefore approximate 
\begin{align}
    (\gl + \tau \idnbar)^\beta & \approx ( (1+\tau)\idnbar - \tilde{\bm{U}} \bm{\Sigma} \tilde{\bm{U}}^T) ^\beta \\
    & =  \left( (1+\tau)(\idnbar- \tilde{\bm{U}} \tilde{\bm{U}}^T) + \tilde{\bm{U}} \left((1+ \tau)\bm{I}_K -\bm{\Sigma}\right) \tilde{\bm{U}}^T \right) ^\beta\\
     &=(1+\tau)^\beta(\idnbar- \tilde{\bm{U}} \tilde{\bm{U}}^T)^\beta + \left( \tilde{\bm{U}} ((1+ \tau)\bm{I}_K -\bm{\Sigma}) \tilde{\bm{U}}^T \right)^\beta \label{eq:line-ref-1}\\
     &=(1+\tau)^\beta(\idnbar - \tilde{\bm{U}} \tilde{\bm{U}}^T) + \tilde{\bm{U}} ((1+ \tau)\bm{I}_K - \bm{\Sigma})^\beta \tilde{\bm{U}}^T \label{eq:line-ref-2} \\
      &=(1+\tau)^\beta \idnbar + \tilde{\bm{U}} \left(((1+ \tau)\bm{I}_K -\bm{\Sigma})^\beta - (1+\tau)^\beta \bm{I}_K \right) \tilde{\bm{U}}^T.
\end{align}
Here line (\ref{eq:line-ref-1}) follows from the fact that if $\bm{AB} = \bm{BA} = \bm{0}$ and $\bm{A},\bm{B}$ are diagonalizable, then $(\bm{A}+\bm{B})^\beta = \bm{A}^\beta + \bm{B}^\beta$, which is here satisfied by $\bm{A} = (1+\tau)(\idnbar- \tilde{\bm{U}} \tilde{\bm{U}}^T)$ and $\bm{B} = \tilde{\bm{U}} ((1+ \tau)\bm{I}_K -\bm{\Sigma}) \tilde{\bm{U}}^T$, since $\tilde{\bm{U}}^T \tilde{\bm{U}} = \bm{I}_K$ and $\bm{A}$ and $\bm{B}$ are symmetric matrices. Line (\ref{eq:line-ref-2}) follows from the fact that $\tilde{\bm{U}}^T \tilde{\bm{U}} = \bm{I}_K$ and therefore that $\idnbar- \tilde{\bm{U}} \tilde{\bm{U}}^T$ is idempotent. 

\subsubsection{Solving (\ref{eq:MAP})}
Defining the diagonal matrices 
\begin{align}
     \bm{\Theta} &:= \pnhat^T \pnhat + \sigma^2 \omega(1 + \tau)^\beta \idnbar \in \mathbb{R}^{\nbar  \times \nbar}, \\
\bm{\Xi} &:= \sigma^2 \omega\left((1+\tau)^\beta \bm{I}_K - ((1+ \tau)\bm{I}_K -\bm{\Sigma})^\beta  \right) \in \mathbb{R}^{K \times K}, 
\end{align} 
it follows that \cref{eq:MAP} can be approximately solved by solving  
\begin{equation}
\label{eq:MAPnys}
    \left(\bm{\Theta} - \tilde{\bm{U}} \bm{\Xi} \tilde{\bm{U}}^T \right)\phimapm = \pnhat^T \phihatm,
\end{equation}
which is equivalent to both
\begin{align}
\label{eq:MAPsaddle}
\begin{pmatrix}
    \bm{\Theta} &  \tilde{\bm{U}} \\
     \tilde{\bm{U}}^T & \bm{\Xi}^{-1}
\end{pmatrix}\begin{pmatrix}
    \phimapm \\
    \bm{B}
\end{pmatrix} &= \begin{pmatrix}
     \pnhat^T \phihatm \\
    \bm{0}
\end{pmatrix}
\end{align}
and 
\begin{align}
\label{eq:MAPsaddle2}
\begin{pmatrix}
    \bm{\Theta} & \tilde{\bm{U}} \\
    \bm{\Xi} \tilde{\bm{U}}^T & \bm{I}_K
\end{pmatrix}\begin{pmatrix}
    \phimapm \\
    \bm{B}
\end{pmatrix} &= \begin{pmatrix}
     \pnhat^T \phihatm \\
    \bm{0}
\end{pmatrix},
\end{align}
because both of these (assuming that $\bm{\Xi}$ is invertible in the former case) are equivalent to
\begin{align}
    \bm{\Theta}\phimapm + \tilde{\bm{U}} \bm{B} &=  \pnhat^T \phihatm, && \bm{B} =  -\bm{\Xi} \tilde{\bm{U}}^T\phimapm.
\end{align}
Because both
\begin{align}
\begin{pmatrix}
    \bm{\Theta} & \tilde{\bm{U}} \\
    \tilde{\bm{U}}^T & \bm{\Xi}^{-1}
\end{pmatrix} &&\text{and} && \begin{pmatrix}
    \bm{\Theta} & \tilde{\bm{U}} \\
    \bm{\Xi}\tilde{\bm{U}}^T & \bm{I}_K
\end{pmatrix}
\end{align}
are extremely sparse, with at most $2NK + N + K$ non-zero entries, \cref{eq:MAPsaddle} and \cref{eq:MAPsaddle2} can then be efficiently solved by sparse linear solvers, e.g. GMRES, see Saad \cite{Saad2003} for details. Since \cref{eq:MAPsaddle} is symmetric as well as sparse, it can be solved more efficiently than \cref{eq:MAPsaddle2}, e.g. via MINRES \cite{MINRES}, so long as the $\bm{\Xi}^{-1}$ term causes no issues.  

\begin{rmk}
    The matrix $\bm{\Xi}$ is a diagonal matrix with diagonal values
    \[
    \sigma^2\omega \left( (1 + \tau)^\beta - (1 + \tau - \sigma_i)^\beta  \right),
    \]
    where the $\sigma_i$ are the diagonal values of $\bm{\Sigma}$. Hence, $\bm{\Xi}$ will be well-conditioned unless for some $i$
    \[
    (1 + \tau)^\beta - (1 + \tau - \sigma_i)^\beta  \approx 0,
    \]
    i.e., $\sigma_i \approx 0$ or $\beta$ is an even integer and $\sigma_i \approx 2(1 + \tau)$. Since the $\sigma_i$ are approximate eigenvalues for $\idnbar - \gl$, they will approximately lie in $[-1,1]$, so for $\tau \geq 0$ this latter case will not arise so long as the approximation is sufficiently good, and furthermore $\sigma_i \leq 1 + \tau$ is desirable anyway so that $\bm{\Xi}$ is defined for all $\beta \geq 0$. The case of $\sigma_i \approx 0$ can be rectified by checking for this in advance and then redefining $\tilde{\bm{U}}$ and $\bm{\Sigma}$ with those columns of $\tilde{\bm{U}}$ and rows/columns of $\bm{\Sigma}$ removed.
\end{rmk}

\subsubsection{Computing the covariance matrix}
The covariance matrix is given by 
\begin{equation}
\bm{C} = \left( \frac{1}{\sigma^2} \pnhat^T \pnhat + \omega (\gl + \tau \idnbar)^\beta\right)^{-1}.
\end{equation}

As before, we can approximate \begin{equation}\label{eq:Woodbury}
\begin{split}
\bm{C} &\approx  \sigma^2 \left( \bm{\Theta} - \tilde{\bm{U}} \bm{\Xi} \tilde{\bm{U}}^T \right)^{-1} \\&= \sigma^2 \left( \bm{\Theta}^{-1} 
 + \bm{\Theta}^{-1} \tilde{\bm{U}} (\bm{\Xi}^{-1} - \tilde{\bm{U}}^T \bm{\Theta}^{-1} \tilde{\bm{U}}  )^{-1} \tilde{\bm{U}}^T \bm{\Theta}^{-1} \right),   
\end{split}
\end{equation}
with the latter equality following from the Woodbury identity \cite{Woodbury50}. This latter term requires $\mathcal{O}(\nbar^2 K )$ time to compute and $\mathcal{O}(\nbar^2 )$ space to store, so may be intractable. However, if all that are desired are matrix-vector products with $\bm{C}$, then these can be computed by precomputing 
\begin{equation}
 \left(\bm{\Xi}^{-1} - \tilde{\bm{U}}^T \bm{\Theta}^{-1} \tilde{\bm{U}} \right)^{-1}
\end{equation}
which requires $\mathcal{O}(NK^2)$ time, and then for any vector $\bm{v}\in\mathbb{R}^\nbar$, $\bm{C}\bm{v}$ can be approximated via \eqref{eq:Woodbury} in $\mathcal{O}(NK)$  time.
\begin{rmk}
    The expression in \cref{eq:Woodbury} provides another method for solving \cref{eq:MAPnys}. 
\end{rmk}

\subsubsection{Computational complexities}
The quantities to be computed and stored for this method, and the computational complexity of doing so, are summarised in \Cref{tb:NysQR}. 

\begin{table}[ht]
\centering
\renewcommand{\arraystretch}{1.4}
\begin{adjustbox}{width=0.995\textwidth}
\begin{tabular}{c l l}
\hline
\textsc{Quantity} & \textsc{Complexity in time} & \textsc{Complexity in space} \\ 
\hline
$\adj(:,X)$ & $\mathcal{O}(\nbar K)$ & $\mathcal{O}(\nbar K)$ \\ 
$\adj(X,X)^{\dagger}$ & $\mathcal{O}(K^3)$ & $\mathcal{O}(K^2)$ \\ 
$\hat \degmat$ & $\mathcal{O}(\nbar K)$ & $\mathcal{O}(\nbar)$ \\ 
$\bm{Q}$ & $\mathcal{O}(\nbar K^2)$ & $\mathcal{O}(\nbar K)$ \\ 
$\bm{R}$ & $\mathcal{O}(\nbar K^2)$ & $\mathcal{O}(K^2)$ \\ 
$\bm{\Gamma,\Sigma}$  & $\mathcal{O}(K^3)$ & $\mathcal{O}(K^2)$ \\ 
$\tilde{\bm{U}}$ & $\mathcal{O}(\nbar K^2)$ & $\mathcal{O}(\nbar K)$ \\ 
\hline
$\bm{\Theta}$ & $\mathcal{O}(\nbar)$ & $\mathcal{O}(\nbar)$\\
$\bm{\Xi}$ & $\mathcal{O}(K)$ & $\mathcal{O}(K)$\\
\hline
$\phimapm$ \footnotesize{($m$-step GMRES on \cref{eq:MAPsaddle2})} & $\mathcal{O}((m^2\nbar+m\nbar K ) D) $ & $\mathcal{O}(mN + m^2+ND)$ \\
$\phimapm$ \footnotesize{($m$-step MINRES on \cref{eq:MAPsaddle})} & $\mathcal{O}( mNKD) $ & $\mathcal{O}(m^2 + ND)$ \\
$\phimapm$ \footnotesize{(via Woodbury \cref{eq:Woodbury})} &
$\mathcal{O}(\nbar K^2 + \nbar KD)$ & $\mathcal{O}(ND)$ \\
$\bm{C}$ & $\mathcal{O}(N^2K)$ & $\mathcal{O}(N^2)$ \\
$\bm{C}\bm{v}$ & $\mathcal{O}(NK^2)$ {\footnotesize{(once)}}${}+\mathcal{O}(NK)$ \footnotesize{(per $\bm{v}$)}  & $\mathcal{O}(N+ K^2)$ \\
\hline
\end{tabular}
\end{adjustbox}
\caption{Computational complexity of the Nystr\"om-QR approximation. \label{tb:NysQR}}
\end{table}

\begin{rmk}
    Both the truncated spectrum and Nystr\"om-QR methods have a similar idea of utilising low-rank structure. The truncation method computes with (a slight modification of) the full Laplacian matrix, and hence incurs costs of $\mathcal{O}(\nbar^2(\log \nbar)^2)$ in space and  $\mathcal{O}(\nbar^2K \log N)$ in time (to compute the low-lying spectrum). However, once these are computed the remaining computations are very efficient. 

    The Nystr\"om-QR method only uses $\mathcal{O}(\nbar K)$ entries of the adjacency matrix, and computes the MAP estimator in $\mathcal{O}(\nbar K^2)$ time, making it better suited for extremely large $N$. Computing the full $\bm{C}$ matrix is not available via this method, as that requires $\mathcal{O}(\nbar^2 K)$ time and $\mathcal{O}(\nbar^2)$ space, however matrix-vector products with $\bm{C}$ can be computed efficiently by this method. 
\end{rmk}

\section{Convergent regularization}
\label{sec:analysis}
We here prove a convergent regularization result. This result shows that as the error in the high-fidelity data tends to zero, one can correspondingly tune down the regularization strength $\omega$ such that the corresponding MAP displacements converge to a displacement which agrees with the true displacement on the high-fidelity data. This gives a guarantee against over-regularization. Our proof technique in this section will be a special case of the technique employed in P\"oschl \cite{poschl2008tikhonov}.

We begin by noting that for a fixed $\omega$, which denotes the strength of the prior, finding the MAP estimate $\phimapm$ \cref{MAP-estimate} is equivalent to solving the optimization problem 
\begin{equation}    
\phimapm \in \argmin_{\boldsymbol{\Theta} \in \mathbb{R}^{\nbar \times D}} \frac{1}{2\sigma^2}\| \pnhat\mathbf{\Theta} - \phihatm\|^2_F + \frac12 \omega \langle \mathbf{\Theta}, (\gl + \tau \idnbar)^\beta\mathbf{\Theta}\rangle_F.
\end{equation}

In this section, we will now suppose that we have a sequence $\{\bm{\hat u}^{(i)}_n\}_{i=1, n \in \mathbb{N}}^\nhat$ of realizations of the high-fidelity data. The idea will be that as $n \to \infty$, these realizations of the high-fidelity data will converge to the true data at the locations $i \in \{1,...,\nhat\}$. Let \begin{align}
[\phihatm_n]_{ij} := (\bm{\hat u}_n^{(i)})_j -  \bm{\bar u}_j^{(i)}, &&  && i \in \{1,...,\nhat\} && j \in \{1,...,D\},\end{align}
 denote the displacements from the low-fidelity data to the $n^\text{th}$ realization of the high-fidelity data. We consider the sequence of optimization problems 
\begin{equation}
    \bm{\Phi}_n^* \in \argmin_{\bm{\Theta} \in \mathbb{R}^{\nbar \times D}} \frac{1}{2}\|\pnhat \bm{\Theta} - \phihatm_n\|^2_F + \omega_n \underbrace{\langle \mathbf{\Theta}, (\gl + \tau \idnbar)^\beta\mathbf{\Theta}\rangle_F}_{=:\mathcal{R}(\bm{\Theta})}.
\end{equation}
Thus, $\bm{\Phi}_n^*$ is the MAP estimator for the true displacements $\bm{\Phi}$, given the $n^\text{th}$ high-fidelity data $\{\bm{\hat u}^{(i)}_n\}_{i=1}^M$ and setting $\omega= \omega_n$. We will show that if the error in the high-fidelity data vanishes and the $\omega_n$ are chosen with appropriate decay (i.e., tending to zero but more slowly than the squared error in the high-fidelity data) then the $\bm{\Phi}_n^*$ will converge to a displacement matrix which agrees with the true displacements at the locations $i \in \{1,...,\nhat\}$, and minimises $\mathcal{R}$ among the displacement matrices satisfying that constraint. 

Observe that 
\begin{itemize}
    \item $\mathcal{R}$ is strictly convex, continuous, coercive, and nonnegative.
    \item $\phihatm_n =  \pnhat\bm{ \Phi} + \bm{E}_n$ where $[\bm{E}_n]_{ij}  =  (\bm{\hat u}_n)_j^{(i)} - \bm{u}^{(i)}_j$ for $i \in \{1,...,\nhat\}$ and $j \in \{1,...,D\}$, i.e. $\bm{E}_n$ describes the error in the $n^\text{th}$ realization of the high-fidelity data.   
\end{itemize}

Let $\delta_n:= \|\bm{E}_n\|_F = \|\phihatm_n - \pnhat\bm{ \Phi}\|_F$ and suppose that as $n \to \infty$:
\begin{align}
    \text{(i) } \delta_n \downarrow 0, && 
    \text{(ii) } \omega_n \downarrow 0, &&
    \text{and} &&
    \text{(iii) } \frac{\delta_n^2}{\omega_n} \to 0.
\end{align}
We will need the following technical lemma. 
\begin{lem}\label{lem:seqlemma}
    Let $x_n$ be a sequence in a topological space $\mathcal{X}$, and let $x \in \mathcal{X}$. Suppose that every subsequence of $x_n$ has a further subsubsequence converging to $x$. Then $x_n$ converges to $x$. 
\end{lem}
\begin{proof}
    Suppose that $x_n \nrightarrow x$. Then there exists an open neighbourhood $U$ of $x$ and subsequence $x_n'$  of $x_n$ such that for all $n$, $x_n'\notin U$. It follows that $x_n'$ has no subsubsequence converging to $x$. 
\end{proof}

\begin{thm}\label{thm:cv-regularizer}
    Given (i-iii), it holds that $\bm{\Phi}_n^* \to \bm{\Phi}_\infty^*$, where $\bm{\Phi}^*_\infty$ is the unique solution to 
    \begin{equation}\label{eq:Rminimiser}
    \argmin_{\bm{\Theta} \in \mathbb{R}^{\nbar \times D}} \mathcal{R}(\bm{\Theta}) \qquad \text{ s.t. } \qquad \pnhat\bm{\Theta} = \pnhat\bm{ \Phi}. 
    \end{equation}
\end{thm}
\begin{proof}
    Since $\mathcal{R}$ is strictly convex, continuous, coercive, and bounded below it follows that \cref{eq:Rminimiser} has a unique solution $\bm{\Phi}_\infty^*$. By \Cref{lem:seqlemma}, it will suffice to show that every subsequence of $\bm{\Phi}_n^*$ has a convergent subsubsequence, and that the limit of this subsubsequence solves \cref{eq:Rminimiser} and is thus equal to $\bm{\Phi}_\infty^*$. 

     Observe that by (i-ii), the nonnegativity of $\mathcal{R}$, and the definition of $\bm{\Phi}_n^*$:
    \begin{equation}\label{eq:Jlimit}
        0 \leq \omega_n \mathcal{R}(\bm{\Phi}_n^*) + \frac12\|\pnhat\bm{\Phi}_n^* -  \bm{\hat\Phi}_n \|_F^2 \leq \omega_n \mathcal{R}(\bm{ \Phi}) + \frac12\|\pnhat\bm{ \Phi} -  \bm{\hat\Phi}_n \|_F^2 \to 0.
    \end{equation}
    Hence by (iii) and the definition of $\delta_n$:  
    \begin{equation}
        0 \leq \mathcal{R}(\bm{\Phi}_n^*) \leq  \mathcal{R}(\bm{\Phi}_n^*) + \frac1{2\omega_n}\|\pnhat\bm{\Phi}_n^* - \bm{\hat \Phi}_n \|_F^2 \leq  \mathcal{R}(\bm{\Phi}) + \frac{\delta_n^2}{2\omega_n} \to \mathcal{R}(\bm{\Phi}),
    \end{equation}
    
    and therefore $\{\mathcal{R}(\bm{\Phi}_n^*)\}_{n \in \mathbb{N}}$ is bounded. Since $\mathcal{R}$ is coercive, it follows that the $\bm{\Phi}_n^*$ are contained within a compact set. 

    Therefore, every subsequence of $\bm{\Phi}_n^*$ has a convergent subsubsequence. Abusing notation, we will denote an arbitrary such subsubsequence by $\bm{\Phi}_n^*$ (and likewise for the corresponding subsubsequences of $\omega_n$, $\bm{\hat\Phi}_n$, and $\delta_n$), and denote its limit by $ \bm{\tilde\Phi}$. It remains to show that $\bm{\tilde \Phi}$ solves \cref{eq:Rminimiser}.

    We first show that $\pnhat\bm{\tilde \Phi} = \pnhat\bm{ \Phi}$. Note that $\pnhat\bm{\tilde \Phi} = \lim_{n \to \infty}\pnhat \bm{\Phi}_n^*$, that by \cref{eq:Jlimit} and the nonnegativity of $\mathcal{R}$ we have that $\|\pnhat \bm{\Phi}_n^* - \bm{\hat \Phi}_n \|_F \to 0 $, and that by (i) $\bm{\hat \Phi}_n \to\pnhat\bm{ \Phi}$. The claim follows. 

    Finally, we show that $\mathcal{R}(\bm{\tilde\Phi}) \leq \mathcal{R}(\bm{\Theta})$ for all $\bm{\Theta}$ such that $\pnhat\bm{\Theta}= \pnhat\bm{ \Phi}$. Observe that by the definitions of $\bm{\Phi}_n^*$ and $\delta_n$:
    \begin{align}
        \omega_n \mathcal{R}(\bm{\Phi}_n^*) &\leq \omega_n \mathcal{R}(\bm{\Phi}_n^*) + \frac12\|\pnhat\bm{\Phi}_n^* - \bm{\hat \Phi}_n \|_F^2 \\ &\leq \omega_n \mathcal{R}(\bm{\Theta}) + \frac12\|\pnhat\bm{\Theta} - \bm{\hat \Phi_n} \|_F^2 \\& = \omega_n \mathcal{R}(\bm{\Theta}) + \frac12\delta_n^2 
    \end{align}
    and therefore by (iii) and since $\mathcal{R}$ is continuous
    \begin{equation}
        \mathcal{R}(\bm{\tilde \Phi}) = \lim_{n\to \infty} \mathcal{R}(\bm{\Phi}_n^*) \leq \lim_{n\to \infty} \mathcal{R}(\bm{\Theta}) + \frac{\delta_n^2}{2\omega_n} = \mathcal{R}(\bm{\Theta}). 
    \end{equation}
\end{proof}

\section{Numerical results}
\label{sec:numerical}
In this section we quantify the performance of the multi-fidelity approach for a suite of diverse problems in computational physics. These problems are governed by different physical models that include the equations of linear elasticity, Darcy's flow, Euler--Bernoulli beam theory, and the incompressible Navier Stokes equations. The data types considered include vectors of quantities of interest (QoIs), and one- and two-dimensional fields, where the dimension of the data space, $D$, ranges from $5$ to $10,201$. The size of the low-fidelity datasets ranges from $3,000$ to $10,000$, and in each case the fraction of the number of high-fidelity data points to the number of low-fidelity data points is small; it lies between $0.5\%$ and $3.3\%$. These attributes are summarized in Table \ref{tab:summary}. Each of the five case studies reported in the table is presented in detail in the five subsections below, followed by a sixth subsection on uncertainty of the multi-fidelity estimates. For all problems, we use a normalized graph Laplacian with $p=q=\frac{1}{2}$.

\new{For all of the experiments,  $\nbar$ is sufficiently small as to not require the numerical approximations described in \Cref{sec:evaluating}, and therefore in all of the below we will take $K = N$, i.e., we will use all of the eigenvectors of the graph Laplacian. In forthcoming work we will explore the effect of varying $K$ on the multi-fidelity accuracy.}

To measure the accuracy of the low- and multi-fidelity data for problems where the data points are vectors of quantities of interest, we compute the relative absolute difference with respect to the high-fidelity data at every point $i$ and for every component $k$. For the low-fidelity data this is given by
\begin{equation}
    e_k^{(i)} = \frac{| \bar{u}^{(i)}_k - \hat{u}^{(i)}_k|}{\frac{1}{\nbar} \sum_{j = 1}^{\nbar}| \hat{u}^{(j)}_k | } \times 100\%, 
    \qquad i\in\{1,\,\dots\,,\nbar\}, \quad k\in \{1,\,\dots,\,D\}.
    \label{eq:error-qoi}
\end{equation}
The expression for the error in the multi-fidelity data is identical, except in the equation above $\bar{u}^{(i)}_k$ is replaced by the multi-fidelity estimate for $u^{(i)}_k$.

When the quantities of interest are  discrete representations of fields, it may be more meaningful (and easier to visualize) to quantify the performance of the low- and multi-fidelity models by computing the $l_2$ norm of the difference from the high-fidelity data, and normalizing this by the average $l_2$ norm of all high-fidelity data. For the low-fidelity data, this is given by
\begin{equation}
    e^{(i)} = \frac{\| \ubari - \uhati \|_2}{\frac{1}{\nbar} \sum_{j=1}^{\nbar} \| \uhat^{(j)} \|_2}  \times 100\%, \qquad i\in\{1,\,\dots\,,\nbar\}.
    \label{eq:error-field}
\end{equation}
The expression for the error in the multi-fidelity data is identical, except in the equation above $\ubari$ is replaced by the multi-fidelity estimate for $\bm{u}^{(i)}$.

The average errors for low- and multi-fidelity data for all problems are also reported in Table \ref{tab:summary}.
It is observed that in each case the multi-fidelity approach significantly improves the accuracy of the low-fidelity data, with a percentage reduction in error that varies from $75\%$ to $86\%$. 

\begin{table}[hbt!]
\centering
\begin{adjustbox}{width=0.995\textwidth}
\begin{tabular}{c c c c c c }
\hline
\rule{0pt}{3ex} \textsc{Problem} & \textsc{Case 1} & \textsc{Case 2} & \textsc{Case 3} & \textsc{Case 4} & \textsc{Case 5} \tabularnewline
\hline 
\rule{0pt}{3ex} Physical model & Elasticity & Elasticity & Darcy & Beam & Navier Stokes \tabularnewline
\rule{0pt}{3ex} Data dimension $D$ & \caseonedim{} & \casetwodim{} & \casethreedim{} & \casefourdim{} & \casefivedim{}  \tabularnewline
\rule{0pt}{3ex} Data type & QoIs & 2D Field & 2D Field & 1D Field & 1D Field  \tabularnewline
\rule{0pt}{3ex} LF data $\nbar$ & \caseonelfn{} & \casetwolfn{} & \casethreelfn{} & \casefourlfn{} & \casefivelfn{}  \tabularnewline
\rule{0pt}{3ex} HF data $\nhat$ & \caseonehfn{} (\caseonefractionn{}) & \casetwohfn{} (\casetwofractionn{}) & \casethreehfn{} (\casethreefractionn{}) & \casefourhfn{} (\casefourfractionn{}) & \casefivehfn{} (\casefivefractionn{}) \tabularnewline
\rule{0pt}{3ex} \new{Cost ratio} & $\approx 1,000$ & $\approx 15,000$ & $\approx 5,000$ & N/A & $\approx 9,400$ \cite{CHENG2024116793}
\vspace{0.2cm}
\tabularnewline
\hline
\rule{0pt}{3ex} LF error & \caseonelferr{} (\caseonelfstd) & \casetwolferr{} (\casetwolfstd) & \casethreelferr{} (\casethreelfstd) & \casefourlferr{} (\casefourlfstd) & \casefivelferr{} (\casefivelfstd) \tabularnewline
\rule{0pt}{3ex} MF error & \caseonemferrbf{} (\caseonemfstd) & \casetwomferrbf{} (\casetwomfstd) & \casethreemferrbf{} (\casethreemfstd) & \casefourmferrbf{} (\casefourmfstd) & \casefivemferrbf{} (\casefivemfstd) \tabularnewline
\rule{0pt}{3ex} Error reduction & \caseoneerrred{} & \casetwoerrred{} & \casethreeerrred{} & \casefourerrred{} & \casefiveerrred{} \tabularnewline
\hline 
\end{tabular}
\end{adjustbox}
\vspace{2mm}
\caption{\label{tab:summary} Summary of the attributes of the numerical experiments (top) and the \new{accuracy of the low- and multi-fidelity models (bottom). In particular, we list the data dimensions, the datasets size, the ratio of HF data used, and an estimate of the ratio between the computational cost of the high- and low-fidelity models. For the beam problem, the low-fidelity model is an analytical model with negligible cost.
Finally, we report the mean and the standard deviation (in parenthesis) of the error distribution of the low- and multi-fidelity models.}}
\end{table}

\subsection{Force and traction attributes of an elastic body with a stiff inclusion}
\label{sec:tracion_prob} 

We consider a soft square sheet in plane stress with an internal stiffer elliptic inclusion. The length of the edge of the square is $L =10 \, \mathrm{cm}$, and its Young's modulus is $E = 1 \mathrm{MPa}$, whereas the Young's modulus of the inclusion is a parameter. Both background and the inclusion are incompressible. The bottom edge of the square is fixed, 
while a uniform downward displacement of $v_0 = -5 \mathrm{mm}$ is applied to the top edge. The vertical edges are traction-free in both horizontal and vertical directions, and the top edge is traction free only in the horizontal direction (Figure \ref{fig:square-inclusion}). The objective is to predict attributes of the vertical traction field on the upper edge as a function of the shape, stiffness, orientation and location of the inclusion. This problem is described in detail in \cite{pinti2023graph} and is motivated by the need to identify stiff tumors within a soft background tissue, which is particularly relevant to detecting and diagnosing breast cancer tumors \cite{sarvazyan2012mechanical,barbone2010review}.

\begin{figure}
    \centering
    \includegraphics[width=0.45\textwidth]{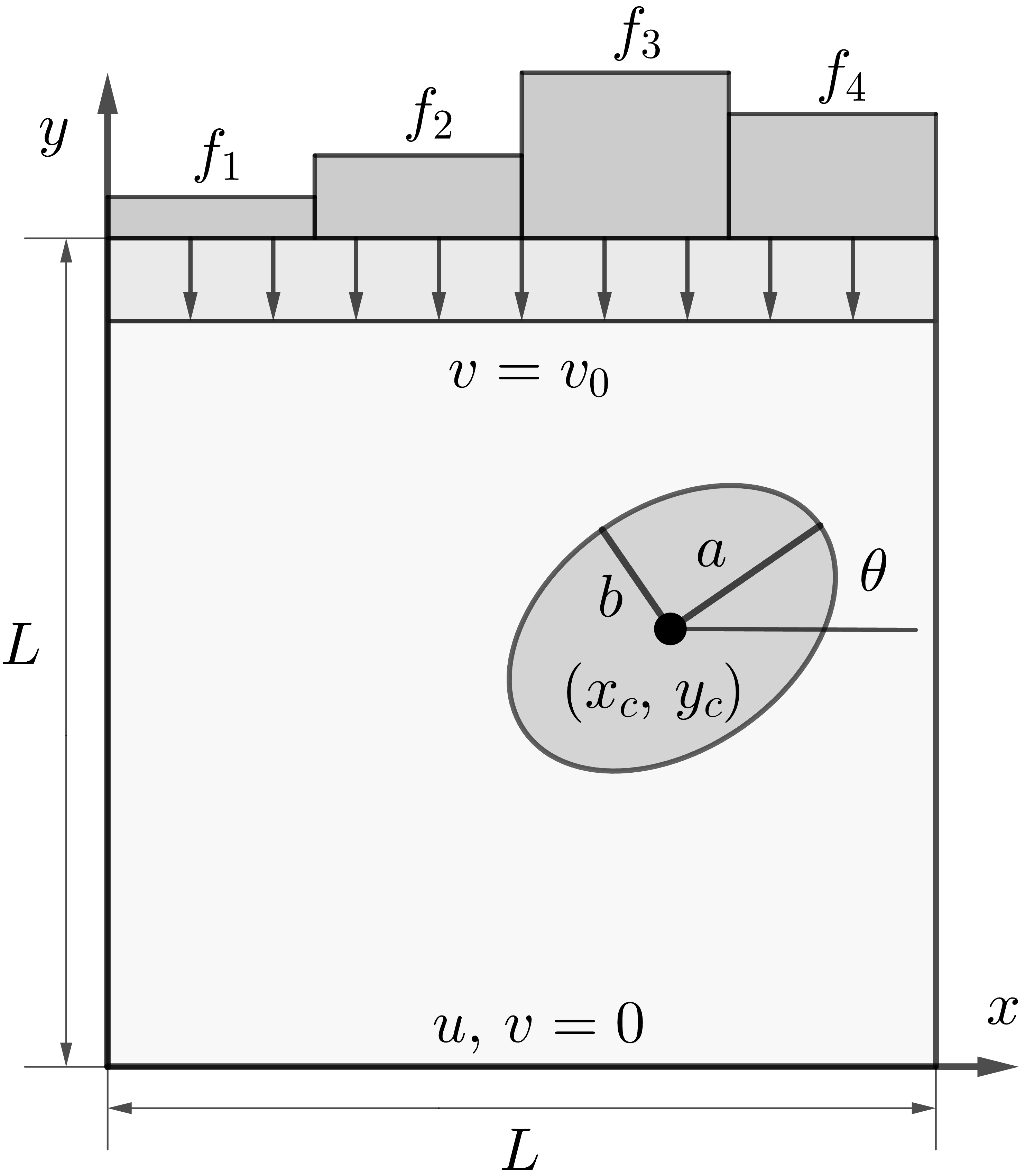}
    \caption{Schematic of the elastic body (light grey) with an elliptic stiffer inclusion (dark grey) for the elasticity problems (Case 1 and 2). The square is compressed on top with a uniform displacement $v=v_0$, while the bottom is fixed (Case 1). 
    }
    \label{fig:square-inclusion}
\end{figure}

\paragraph*{Parameters and quantities of interest} 
The input parameters include the coordinates of the center $(x_c,\,y_c)$ of the inclusion, its orientation $\theta$, Young's Modulus $E_I$, and major and minor semi-axes. The range of these parameters is reported in Table \ref{tab:Params_range} and when generating the low-fidelity data they are sampled from a uniform distribution within this range. 
The output quantities of interest are the localized vertical forces on the top edge which are determined by dividing the top edge into $4$ sections of equal length and integrating the vertical traction $\sigma_{yy}$ over each section. This results in 4 values of localized forces $f_i$, $i \in \{ 1, \dots, 4\}$ (see Figure \ref{fig:square-inclusion}),
\begin{equation}
    f_i = \int_{(i-1)\frac{L}{4}}^{i\frac{L}{4}} \sigma_{yy} (x,\, L) \mathrm{d}x,\;\;\;\;i\in\{1,\,\dots,\,4\}. 
\end{equation}
We also include the maximum value of vertical traction on the top edge as an additional feature, leading to $5$ quantities of interest: $u_i = f_i$, for each $i\in \{1,\,\dots,\,4\}$, and $u_{5} = \max_x |\sigma_{yy}(x,\,L)|=:\sigma_{yy}^{\mathrm{max}}$. As the location, orientation and size of the inclusion is varied, the traction field on the top surface changes, which in turn changes the components of the localized forces, and the maximum value of traction.

\begin{table}[hbt!]
\centering
\begin{tabular}{c c c c}
\hline
Parameter & Min & Max & Units\\
\hline
\rule{0pt}{2ex} $x_c$ & 2.5 & 7.5 & cm\\
\rule{0pt}{2ex} $y_c$ & 2.5 & 7.5 & cm\\
\rule{0pt}{2ex} $\theta$ & 0 & 180 & degree\\
\rule{0pt}{2ex} $E_I$ & 3 & 6 & MPa\\
\rule{0pt}{2ex} $a$ & 1 & 2 & cm\\
\rule{0pt}{2ex} $b$ & 1 & 2 & cm\\
\hline
\end{tabular}%
\hspace{1cm}
\begin{tabular}{c c c c}
\hline 
\end{tabular}
\vspace{2mm}
\caption{\label{tab:Params_range} Range spanned by the input parameters for the elasticity problem (Case 1).}
\end{table}

\paragraph*{Low- and high-fidelity models} 
We employ two finite element method-based solvers differing in mesh density to solve the problem. The low-fidelity model uses a coarse mesh with 200 triangular elements, and the high-fidelity model uses a fine mesh with 20,000 elements. The solution for high-fidelity model is verified to be mesh converged. 
\new{The ratio of the computational costs of the two models is a power of the mesh size ratio. The exact exponent lies between 1 and 2, depending on the solver and the problem. A typical value is approximately 1.5, leading to a computational cost ratio of around 1,000.}

\paragraph*{Numerical results} We sample $\nbar=$ \caseonelfn{} instances of the input parameters from a uniform distribution and run the low- and high-fidelity models. Thereafter, we normalize the low- and high-fidelity quantities of interest using (\ref{eq:norm-qoi}), so that each low-fidelity component has zero-mean and unit standard deviation.
We then use $\nhat=$ \caseonehfn{} high-fidelity data points (\caseonefractionn{} of the total) for generating the multi-fidelity data, and the remainder for testing the performance of the method. 

In Figure \ref{fig:datasets-1} we plot the projections of the low-fidelity (column 1), multi-fidelity (column 2) and high-fidelity (column 3) data points on four mutually orthogonal planes (rows 1-4). In the low-fidelity plots we also indicate (in blue) the points whose high-fidelity counterparts are used to compute the multi-fidelity data. For each plane, we observe that the multi-fidelity point cloud is closer in shape and form to the high-fidelity point cloud when compared with its low-fidelity counterpart.

We quantify the accuracy of the low- and multi-fidelity data via the error defined in (\ref{eq:error-qoi}), and compare the distribution of these two errors for each component in Figure \ref{fig:error-dist-prob-1}. For every component, the distribution of the error for the multi-fidelity data is closer to zero, and is narrower when compared with the distribution of the error for the low-fidelity data. 
In Table \ref{tab:errors} we report the mean \new{and the standard deviation of the error distributions} across all points for each component separately, while in Table \ref{tab:summary} we report the average of these values. 
We observe that the multi-fidelity update has reduced the error in the low-fidelity data by 70--84\%. 

\begin{table}[hbt!]
\centering
\begin{adjustbox}{width=0.995\textwidth}
\begin{tabular}{c c c c c c }
\hline 
\rule{0pt}{3ex} \textsc{Quantity of interest} & $f_{1}$ & $f_{2}$ & $f_{3}$ & $f_{4}$ & $\sigma_{yy}^{\mathrm{max}}$ \tabularnewline
\hline 
\rule{0pt}{3ex} Low-fidelity error & 5.74 (2.61) & 4.8 (4.81) & 7.57 (5.64) & 4.93 (4.72) & 11.04 (7.78) \tabularnewline
\rule{0pt}{3ex} Multi-fidelity error & \textbf{0.92} (0.97) & \textbf{1.45} (1.86) & \textbf{1.66} (1.85) & \textbf{0.96} (1.14) & \textbf{3.36} (4.62) \tabularnewline
\rule{0pt}{3ex} Error reduction & 84 \% & 69.8\% & 78.1\% & 80.5\% & 69.6\% \tabularnewline
\hline 
\end{tabular}
\end{adjustbox}
\vspace{2mm}
\caption{\label{tab:errors} \new{Error in the low- and multi-fidelity data for each output component for the elasticity problem (Case 1), together with its mean and standard deviation (in parenthesis)}.}
\end{table}

\begin{figure}
    \centering
    \includegraphics[width=0.85\textwidth]{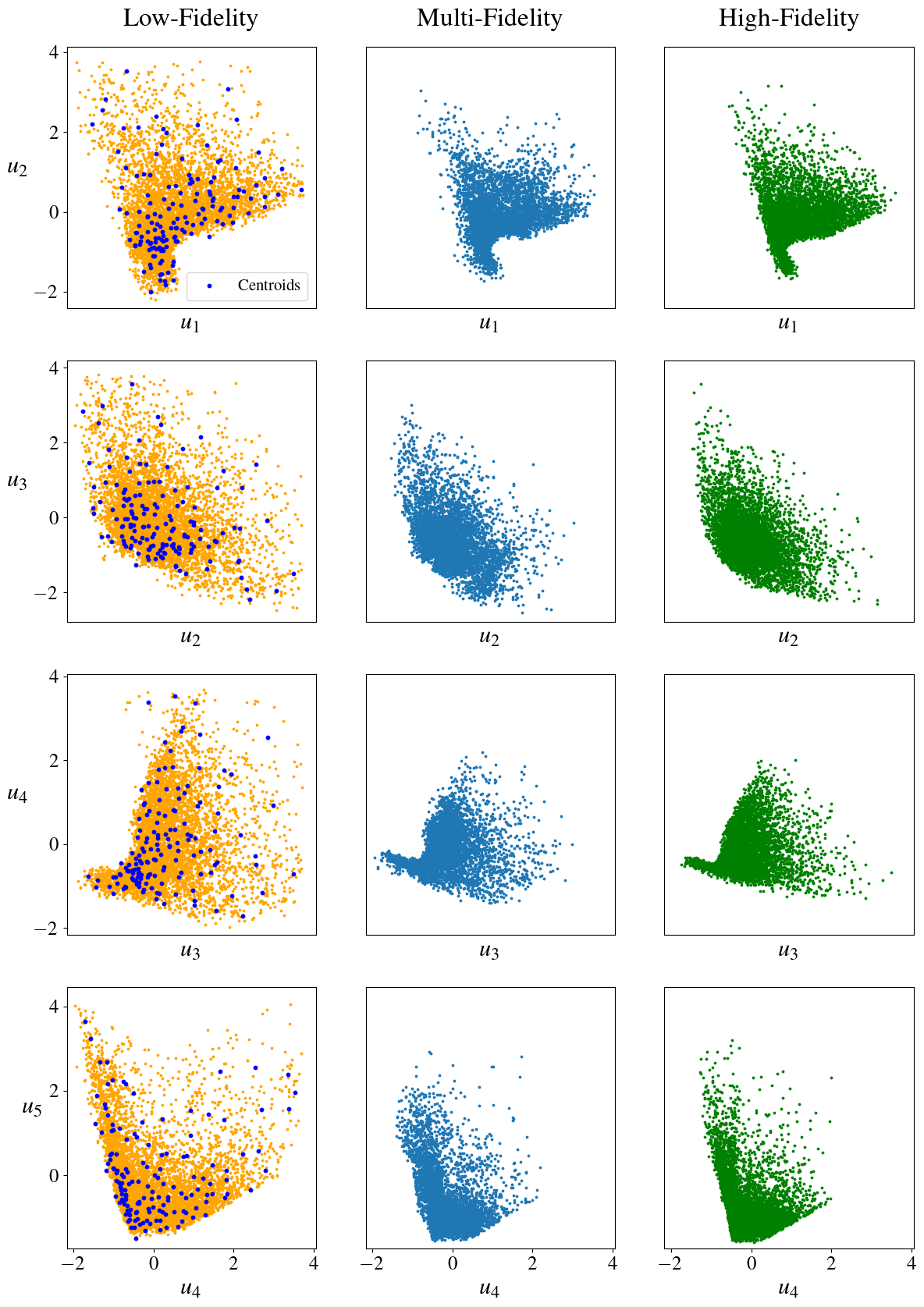}
    \caption{
    \new{The figure shows orthogonal projections of the datasets for the low-fidelity (left column), multi-fidelity (center column), and high-fidelity (right column) models for the elasticity problem (Case 1). 
    Each row shows a different orthogonal plane. The low-fidelity data points, in column 1, exhibit greater range compared to the high-fidelity point cloud. We note how the multi-fidelity points show better agreement with the target high-fidelity distribution.}
    }
    \label{fig:datasets-1}
\end{figure}

\begin{figure}
    \centering
    \includegraphics[width=0.5\textwidth]{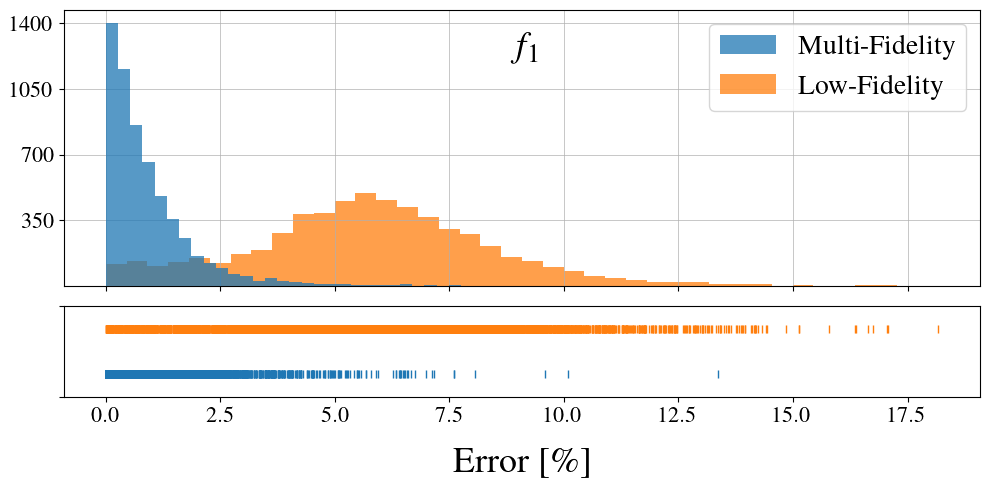}
    \includegraphics[width=0.5\textwidth]{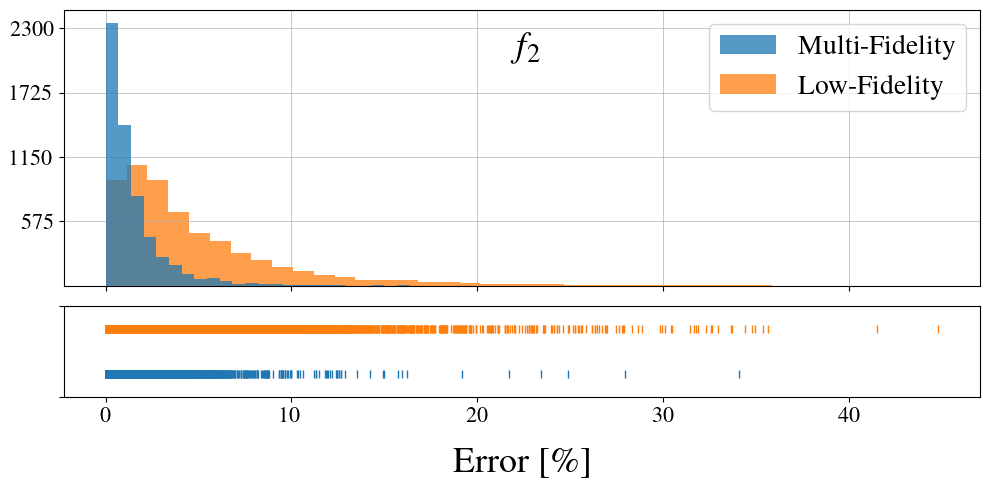}
    \includegraphics[width=0.5\textwidth]{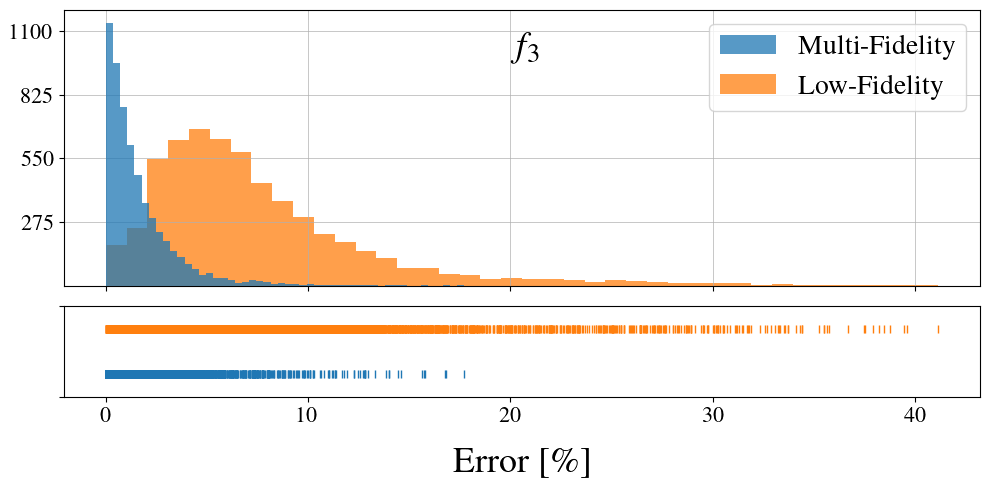}
    \includegraphics[width=0.5\textwidth]{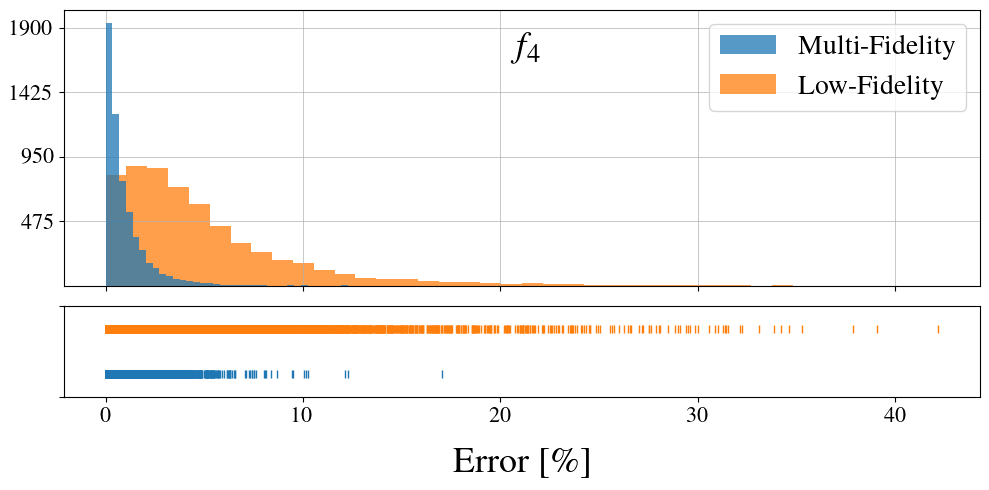}
    \includegraphics[width=0.5\textwidth]{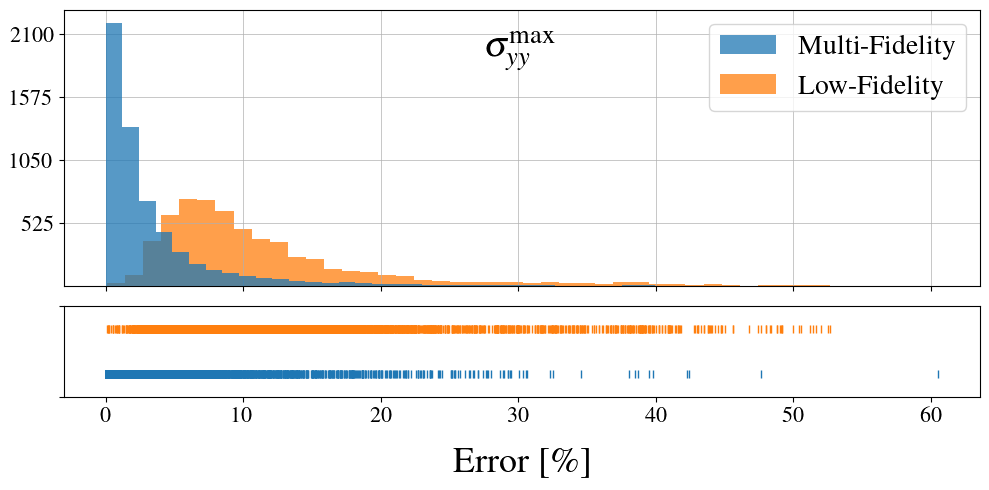}
    \caption{
    \new{Histogram plots of the error distribution for the low- and multi-fidelity data for the elasticity problem (Case 1). Each plot shows the distribution of the error of the two models for one component of the vector of quantities of interest across the whole datasets. The values of mean and standard deviation of these distributions are reported in Table \ref{tab:errors}.}}
    \label{fig:error-dist-prob-1}
\end{figure}

\subsection{Displacement field of an elastic body with a stiff inclusion}
\label{sec:displacement_problem}
The physical model for this problem is the same as in the previous problem and most of the parameters are also the same. The differences are (a) the modulus of the background is $E = 3 \, \mathrm{MPa}$, (b) on the top boundary we prescribe a uniform traction instead of uniform displacement, and (c) the center of the inclusion is fixed at the centre of the square domain.

\paragraph*{Parameters and quantities of interest} 
The input parameters for this problem include the orientation of the elliptical inclusion $\theta$, its Young's Modulus $E_I$, its major semi-axis $a$, and the uniform vertical traction applied on the top edge $t_y$.  The minor semi-axis is set to $b = \frac{1}{\pi a}$, to maintain a constant area for the elliptical inclusion. In Table \ref{tab:Params_range_2} we provide the minimum and maximum values for these parameters. The output quantity of interest is the vertical displacement field $u_y(x,y)$ in the entire domain, sampled at $101^2=10,201$ points on a uniform grid. 
We note that for every value of applied traction, $t_y$, the vertical displacement within the elastic body can be decomposed into a part that varies linearly from zero at $y=0$ to a maximum value at $y = 1$ and another part that is nonlinear in $y$. For the linear part of the displacement the value at $y=1$ edge is set to the average vertical displacement at this edge. We also note that both the low- and high-fidelity models are able to capture the linear part accurately. Therefore, the main goal of the multi-fidelity approach is to use the high-fidelity data to improve the accuracy of the non-linear part of the vertical displacement. For this reason, for every realization of the input parameters, we compute the linear part of the displacement from the low-fidelity model and subtract it from the low-fidelity data and high-fidelity data.

\begin{table}[hbt!]
\centering
\begin{tabular}{c c c c}
\hline
Parameter & Min & Max & Units\\
\hline
\rule{0pt}{2ex} $\theta$ & 0 & 180 & degree\\
\rule{0pt}{2ex} $E_I$ & 9 & 18 & MPa\\
\rule{0pt}{2ex} $a$ & 0.85 & 2.5 & cm\\
\rule{0pt}{2ex} $\sigma_{\mathrm{yy}}^{\mathrm{top}}$ & $10^2$ & $10^5$ & Pa\\
\hline
\end{tabular}%
\hspace{1cm}
\begin{tabular}{c c c c}
\hline 
\end{tabular}
\vspace{2mm}
\caption{\label{tab:Params_range_2} Range spanned by input parameters for elasticity problem (Case 2).}
\end{table}

\paragraph*{Low- and high-fidelity models} 
We use two finite element-based models to compute the low- and high-fidelity data. The high-fidelity model employs a structured triangular mesh with $20,000$ elements, while the low-fidelity model uses a much coarser structured triangular mesh with only $32$ elements. 
\new{Given the mesh size ratio of $\frac{20,000}{32}=625$, and assuming an exponent for the cost power law of 1.5, the ratio of the computational cost of the two models is approximately 15,000.}
We interpolate the low- and high-fidelity solutions on to the same $101 \times 101$ uniform grid to ensure the dimensionality $D = $ \casetwodim{}.

\paragraph*{Numerical results} 
We consider $\nbar=$ \casetwolfn{} pairs of low- and high-fidelity data, obtained by sampling the input parameters from a uniform distribution. Each low-fidelity data is scaled using (\ref{eq:norm-fields}), and corresponding high-fidelity data are scaled by the same procedure.
We compute the multi-fidelity model with $\nhat=$ \casetwohfn{} high-fidelity training data points (\casetwofractionn{} of the total).
Figure \ref{fig:datasets-2} shows examples of the resulting low- and multi-fidelity fields, along with their differences relative to the high-fidelity solutions. We observe that error in the multi-fidelity field is much smaller. 

We quantify the performance of the low- and multi-fidelity models via the error defined in (\ref{eq:error-field}).
Figure \ref{fig:error-dist-prob-2} is a histogram of the distribution of these errors for the low- and multi-fidelity data. From this plot we observe that the multi-fidelity data has significantly lower errors, and the spread in the error is also smaller. The mean \new{and standard deviation of the error for the low-fidelity data are \casetwolferr{}$\%$ and \casetwolfstd{}, whereas for the multi-fidelity data the values are \casetwomferr{}$\%$ and \casetwomfstd{}, respectively. This amounts to a \casetwoerrred{} reduction in the mean error (see Table \ref{tab:summary}).}

\begin{figure}
    \centering
    \includegraphics[width=0.95\textwidth]{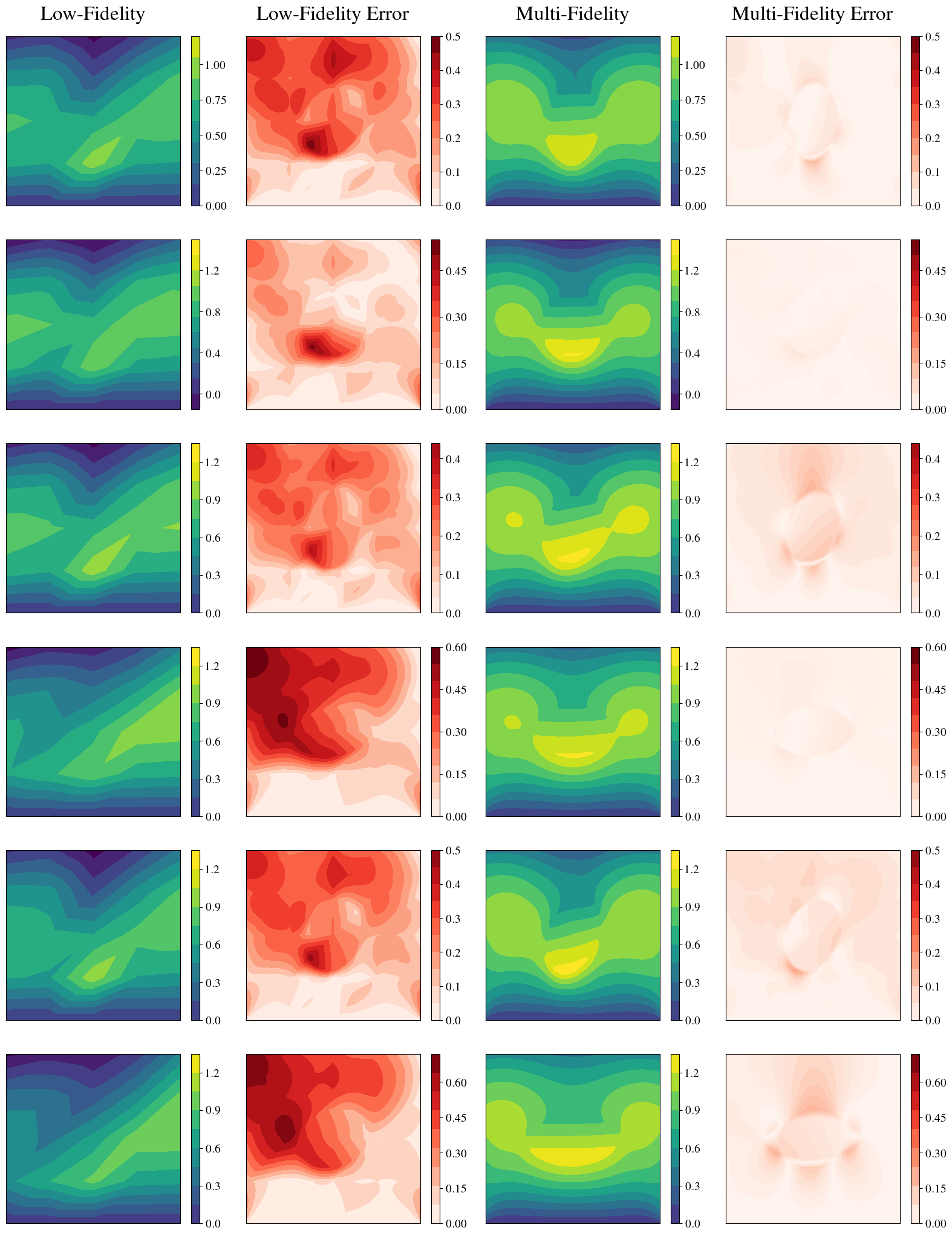}
    \caption{
    \new{Typical instances of the non-linear part of the vertical displacement fields from low- and multi-fidelity models for the elasticity problem (Case 2). The figure includes plots of point-wise error between each model and the high-fidelity field.}
    }
    \label{fig:datasets-2}
\end{figure}

\begin{figure}
    \centering
    \includegraphics[width=0.8\textwidth]{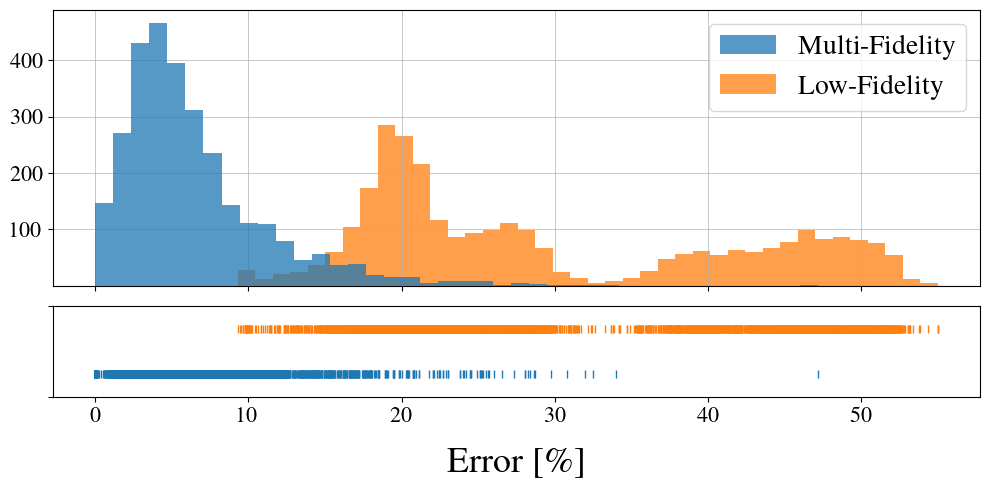}
    \caption{
    \new{Histogram of error distribution for low- and multi-fidelity datasets for the elasticity problem (Case 2). The error distribution for the low-fidelity data has a mean of \casetwolferr{}\% and a standard deviation of \casetwolfstd{}. In contrast, the multi-fidelity data shows a reduced mean error of \casetwomferr{}\% with a standard deviation of \casetwomfstd{}.}
    }
    \label{fig:error-dist-prob-2}
\end{figure}

\subsection{Darcy flow}
\label{sec:darcy_problem}
 
We consider the two-dimensional problem of a fluid percolating through a porous medium characterized by a non-homogeneous permeability field. This is commonly referred to as Darcy flow, and it is described by the following equations
\begin{align}
   - \nabla \cdot (\eta(\bm{x}) \nabla p(\bm{x})) &= g(\bm{x}), \qquad \bm{x} \in \Omega\\
   \int_{\Omega} p(\bm{x}) \mathrm{d}\bm{x} = 0.
\end{align}
where $\bm x = (x,y)$, $p(\cdot)$ is the pressure field, $\eta(\cdot)$ is the permeability field, $g(\cdot)$ is a source term, and $\Omega = [0, 1]^2$ is the domain of interest. The zero-mean condition for the pressure guarantees the uniqueness of the solution.
We consider a forcing term $g$ defined by a source at the top left corner, and a sink at the bottom right corner (see Figure \ref{fig:darcy-schematic}), i.e.,
\begin{equation}    
    g(\bm{x}) = 
    \begin{cases}
      +1, & \text{if } x^2 + (y - 1)^2 < 0.15^2, \\
      -1, & \text{if } (x - 1)^2 + y^2 < 0.15^2, \\
      0, & \text{otherwise}.
    \end{cases}
\end{equation}

\begin{figure}
    \centering
    \includegraphics[width=0.65\linewidth]{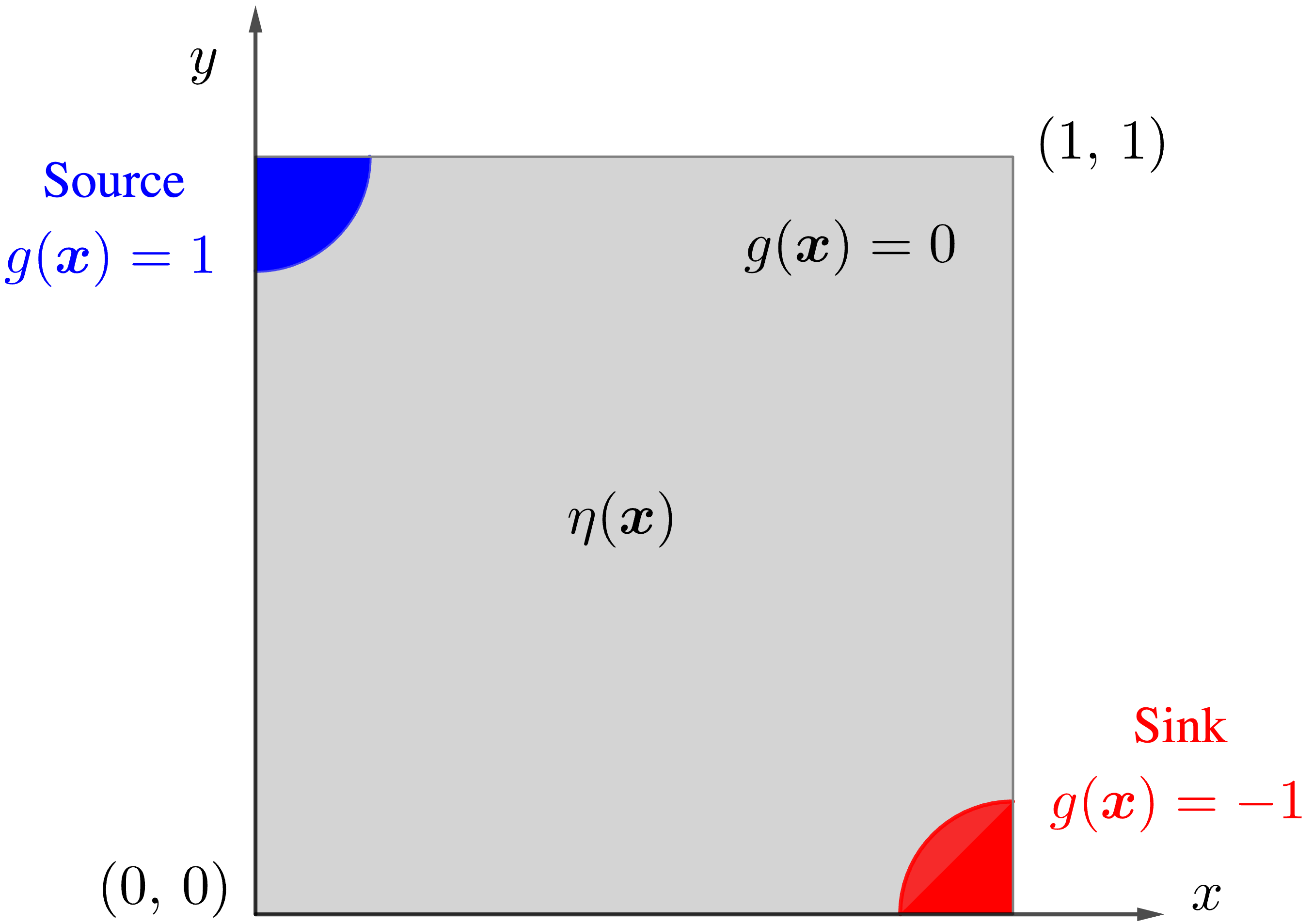}
    \caption{Schematic of the domain, source term and permeability field for the Darcy flow problem (Case 3).
    \new{A source term located in the top left corner injects fluid into the domain, while a sink in the bottom right corner allows the fluid to exit. The input parameter is an instance of permeability field $\eta(\bm{x})$, and the quantity of interest is the resulting pressure field $p(\bm{x})$ of the fluid percolating through the medium.}
    }
    \label{fig:darcy-schematic}
\end{figure}
We generate instances of $\eta(\cdot)$ by sampling a Gaussian process with a prescribed covariance and length scale, and take its exponential to ensure non-negativity of the permeability. That is, 
\begin{align}
   \eta(\bm x) &= \exp \left( {G( \bm x)} \right),\; G(\cdot) \sim \mathcal{N}(0, k(\cdot,\,\cdot)) \\
   k(\bm x,\, \bm x ') &= \exp \left( - \frac{\|\bm x - \bm x' \|^2_2}{ (0.25)^2} \right).
\end{align}

\paragraph*{Low- and high-fidelity models} 
We compute both low- and high-fidelity solutions using finite element-based models. The high-fidelity model employs a mesh of $20,000$ triangular elements, while the low-fidelity model uses a mesh with $72$ triangular elements. 
\new{With a mesh size ratio of $\frac{20,000}{72}=278$, and assuming an exponent for the cost power law of 1.5, the ratio of the computational cost of the two models is approximately 5,000.}
We interpolate both low- and high-fidelity solutions onto a uniform grid of $101^2=10,201$ points to ensure they have the same dimension.

\paragraph*{Parameters and quantities of interest} 
The input to the problem is an instance of the permeability field $\eta^{(i)}(\bm x)$. We generate this field efficiently by utilizing a truncated Karhunen–Loève expansion \cite{Huang-KL, Stefanou2007AssessmentOS}, retaining the first 200 terms. 
The quantity of interest is the resulting pressure field $p^{(i)}(\bm x)$ defined over the whole domain on a uniform grid of $101 \times 101$ points. Therefore, for this problem $D =$ \casethreedim{}. 

\paragraph*{Numerical results}
The low-fidelity dataset consists of $\nbar=$ \casethreelfn{} data points, and we use $\nhat=$ \casethreehfn{} high-fidelity training points (\casethreefractionn{} of the total) to compute the multi-fidelity estimates.

In Figure \ref{fig:datasets-3}, we have plotted four instances of the input permeability, and the corresponding low-and multi-fidelity pressure fields normalized with Eq. (\ref{eq:norm-fields}), as well as the difference between these fields and their high-fidelity counterpart. We observe that while for all cases the pressure field varies from a large value on the top-left corner to a small value on the bottom-right corner, the permeability field has a significant effect on this variation. Generally speaking, large permeability values lead to a uniform distribution of pressure, while small values lead to sharper changes in pressure. The low- and multi-fidelity pressure fields are qualitatively similar, however on closer look there are discernible differences. This is made clear by comparing the error fields for the two, which are plotted on the same scale. The error in the low-fidelity field is much larger when compared with the multi-fidelity field.

In Figure \ref{fig:error-dist-prob-3} we have plotted the histogram for the error distributions (as defined in Eq. (\ref{eq:error-field})) for the low-and multi-fidelity data. Once again we observe that the multi-fidelity data has much smaller error and that this error is more tightly centered about its mean. \new{The mean and standard deviation of the errors for the low-fidelity model are \casethreelferr{}$\%$ and \casethreelfstd{}, respectively, whereas for the multi-fidelity model we observe a mean and standard deviation of \casethreemferr{}$\%$ and \casethreemfstd , respectively (see Table \ref{tab:summary}). This implies a \casethreeerrred{} percentage reduction in the mean error due to the multi-fidelity approach.}

\begin{figure}
    \centering
    \includegraphics[width=0.995\textwidth]{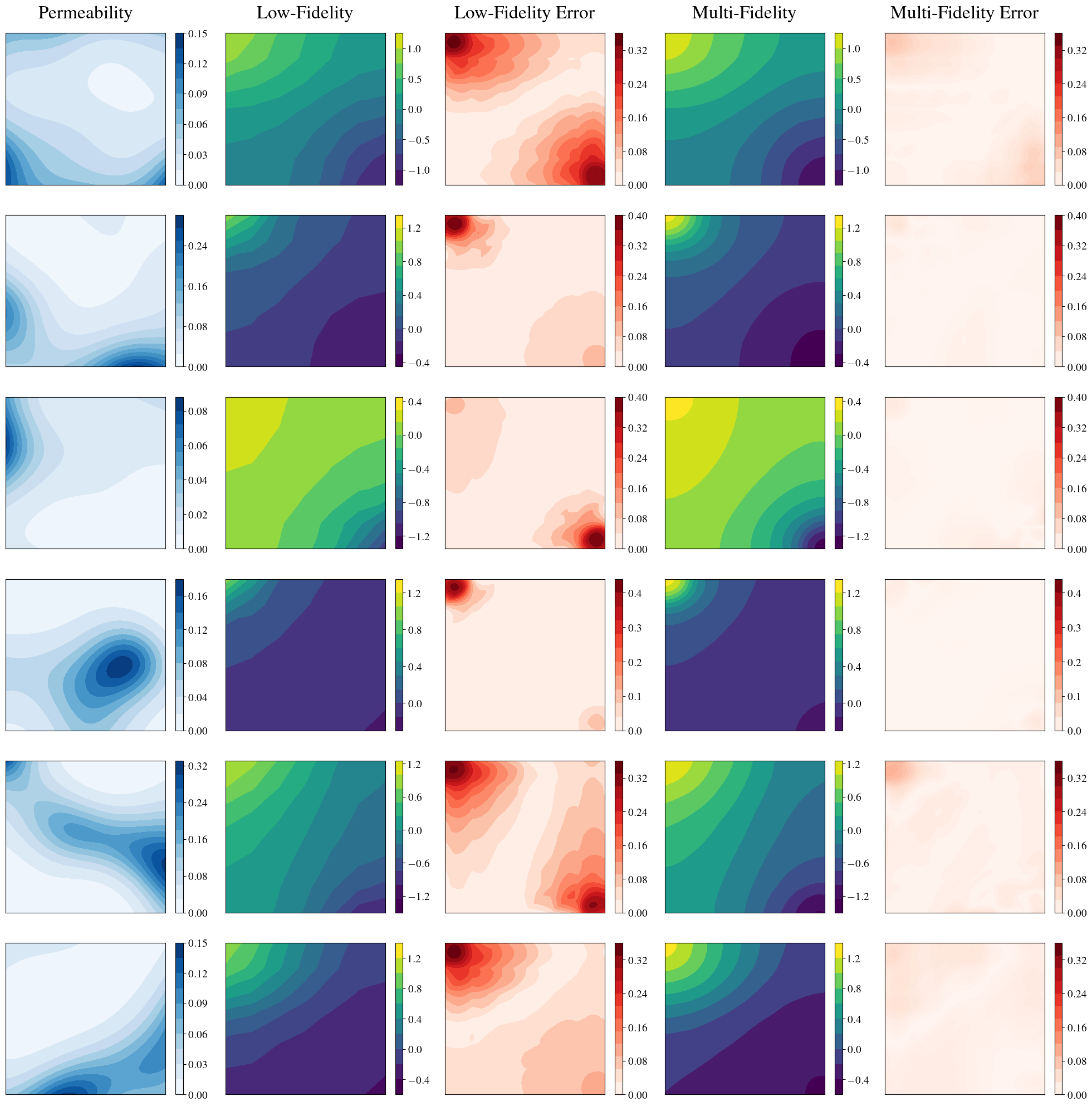}
    \caption{
    \new{Instances of permeability fields and corresponding low- and multi-fidelity pressure distributions for the Darcy flow problem (Case 3). Plots of point-wise error fields with respect to the high-fidelity field are also shown.}
    }
    \label{fig:datasets-3}
\end{figure}

\begin{figure}
    \centering
    \includegraphics[width=0.8\textwidth]{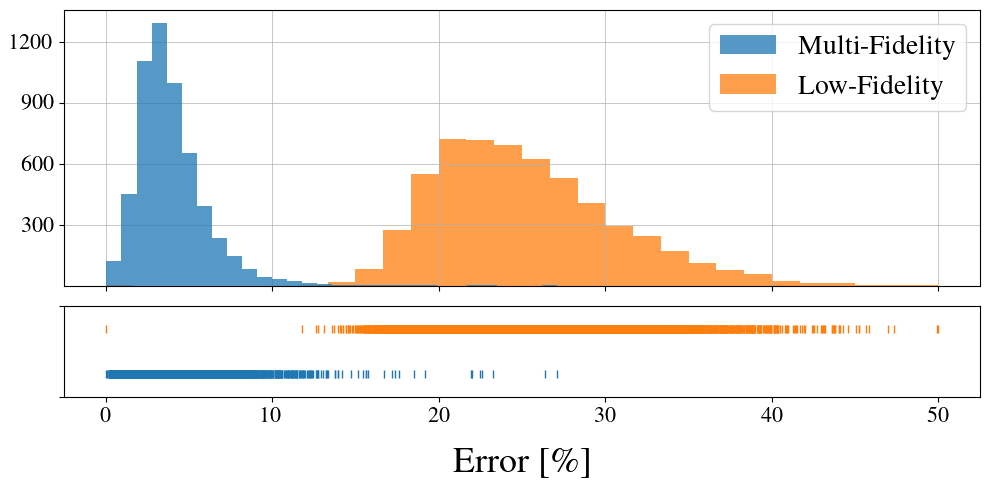}
    \caption{
    \new{Histogram of error distribution for low- and multi-fidelity datasets for the Darcy flow problem (Case 3). The mean and standard deviation of the low-fidelity distribution are \casethreelferr{}\% and \casethreelfstd{}, respectively. On the other hand, the mean and standard deviation of the distribution for the multi-fidelity data are \casethreemferr{}\% and \casethreemfstd{}, respectively.}
    }
    \label{fig:error-dist-prob-3}
\end{figure}

\subsection{Composite cantilever beam}
This problem is described in Cheng \textit{et al.} \cite{CHENG2024116793} and the authors have also provided the data. It involves a composite cantilever beam subject to uniform distributed vertical load. As shown in Figure \ref{fig:beam-schematic}, the cross section of the beam is composed of three materials with different properties, and there are five holes running through the lateral extent of the beam.

\begin{figure}
    \centering
    \includegraphics[width=0.9\textwidth]{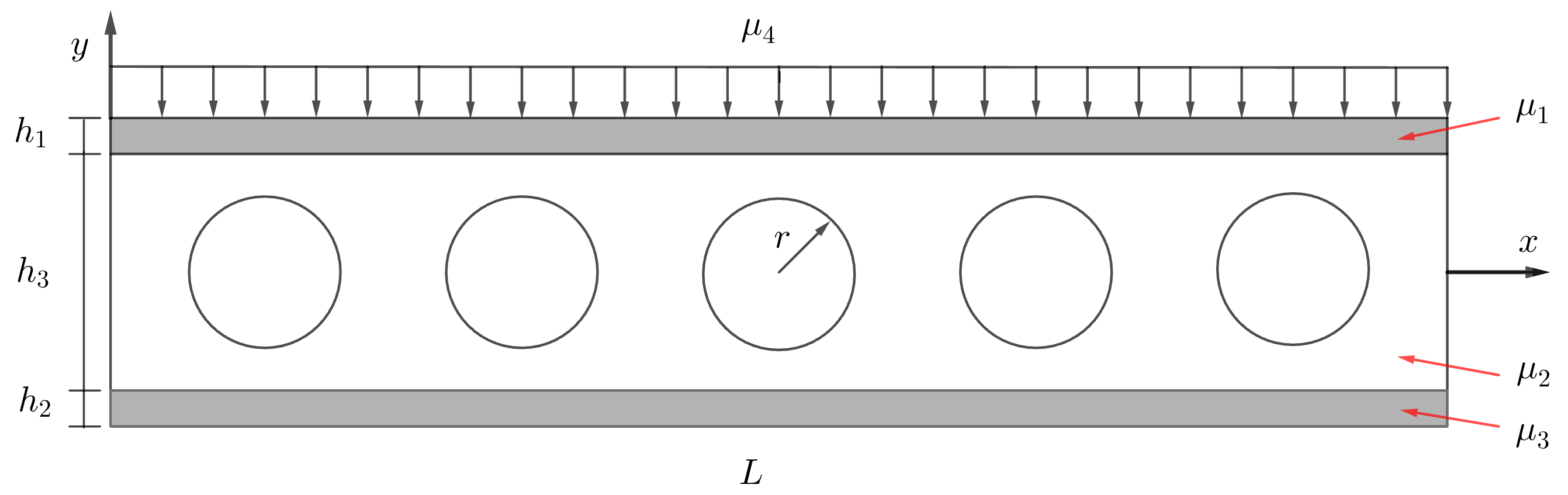}
    \caption{Schematic diagram of the composite cantilever beam subject to a uniform vertical load (Case 4). \new{The presence of the holes is neglected in the low-fidelity model, where an analytical expression for the displacement is derived based on the Euler-Bernoulli beam theory. The parameters include the Young's modulus of the three materials composing the beam, and the vertical load applied.}
    }
    \label{fig:beam-schematic}
\end{figure}

\paragraph*{Low- and high-fidelity models} 
The low-fidelity model is an analytical expression derived from the Euler–Bernoulli beam theory, which ignores the out of plane deformation of the beam and the effect of the holes in the beam. The high-fidelity model is based on the solution of the equations of three-dimensional elasticity using the finite element method.
\new{Because the low-fidelity model is an analytical model, its computational cost is negligible.}
The reader is referred to \cite{CHENG2024116793} for details of the problem and solution techniques.

\paragraph*{Parameters and quantities of interest} 
The inputs of the problem are the Young's moduli of the three components of the beam and the magnitude of the vertical load applied on the top cord ($\mu_1,\,\mu_2,\,\mu_3$ and $\mu_4$, respectively, in Figure \ref{fig:beam-schematic}).
The quantity of interest is the vertical displacement field $w(x)$ of the top cord, which is represented using a uniform grid with 512 points. Therefore, for this problem $D = $ \casefourdim{}.

\paragraph*{Numerical results}
The low-fidelity dataset consists of $\nbar= $ \casefourlfn{} data points, and the high-fidelity points used to compute the multi-fidelity estimates contains $\nhat=$ \casefourhfn{} (\casefourfractionn{} of the total) data  points. 
We do not normalize the data in this case.
In Figure \ref{fig:beam-comparison} we plot the low-, high-, and multi-fidelity versions of vertical displacements for four instances of input parameters. In all cases, the low-fidelity model under-predicts the displacement, while the multi-fidelity approach is able to correct this. Further, we observe that the multi-fidelity approach is able to capture the effect of the circular holes in the beam, represented by undulations in the displacement, which is missing from the low-fidelity displacement.

In Figure \ref{fig:beam-err-dist} we have plotted the histogram for the error distributions (defined in Eq. (\ref{eq:error-field})) for the low-and multi-fidelity data. The improvement in the performance of the multi-fidelity approach is significant, with no overlap between the low- and multi-fidelity errors. 
\new{The mean and standard deviation of the errors for the low-fidelity model are \casefourlferr{}$\%$ and \casefourlfstd{}, respectively, while for the multi-fidelity model we have \casefourmferr{}$\%$ and \casefourmfstd{}, respectively (see Table \ref{tab:summary}). The relative reduction in the mean error is  \casefourerrred{}.}

\begin{figure}
    \centering
    \includegraphics[width=0.99\textwidth]{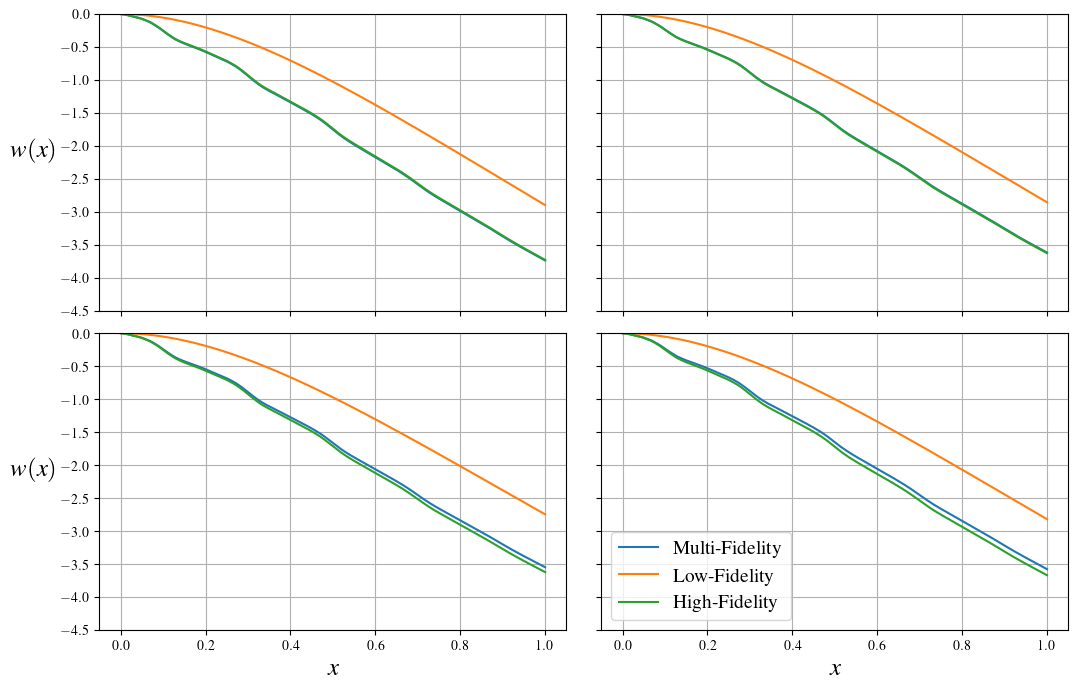}
    \caption{
    \new{Instances of low-, multi-, and high-fidelity solutions for the vertical displacement for the composite cantilever beam (Case 4), resulting from different input parameters. The low-fidelity model underestimates the vertical displacement and neglects structural features like holes. On the other hand, the multi-fidelity model provides solutions that closely align with high-fidelity results, sometime almost completely overlapping with the high-fidelity curves.}
    }
    \label{fig:beam-comparison}
\end{figure}

\begin{figure}
    \centering
    \includegraphics[width=0.8\textwidth]{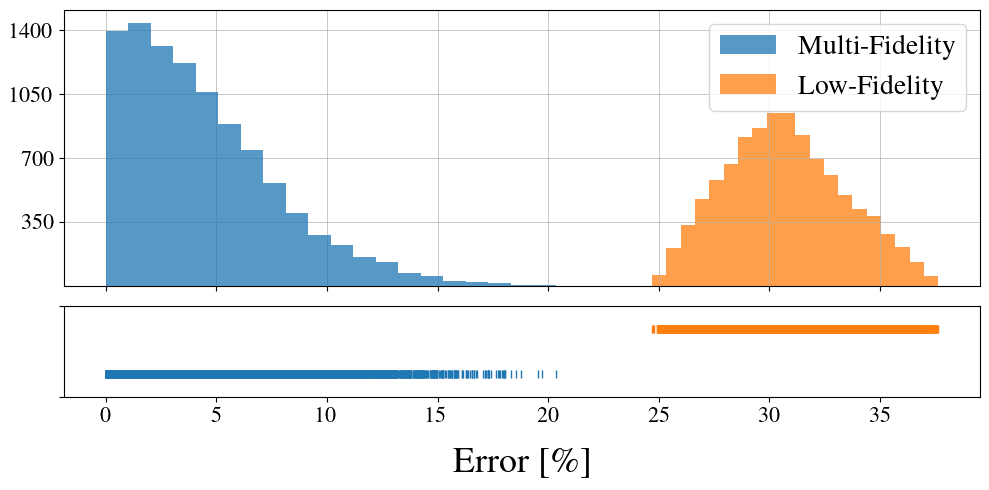}
    \caption{
    \new{Histogram of error distribution for low- and multi-fidelity datasets for the composite beam problem (Case 4). The multi-fidelity approach results in a distribution that is narrower and closer to zero. The mean and standard deviation for the low-fidelity distribution are \casefourlferr{}\% and \casefourlfstd{}, respectively. Whereas for the multi-fidelity model the mean and standard deviation are \casefourmferr{}\% and \casefourmfstd{}, respectively.}
    }
    \label{fig:beam-err-dist}
\end{figure}

\subsection{Heat flux in cavity flow}
Our final example is also solved using data from Cheng \textit{et al.} \cite{CHENG2024116793} and involves predicting the distribution of heat flux in a fluid.
In this problem the domain is a closed cavity in the form of a unit square which contains a fluid. The left wall of the cavity is maintained at a fixed temperature $T_h$, while on the the right wall (the cold wall) a stochastic profile, $T_c(y)$, with mean value  $\bar T_c < T_h$ is prescribed. The two horizontal walls are assumed to be adiabatic. The temperature difference between the vertical walls onsets a clockwise motion of the fluid inside the cavity, which is modeled using the unsteady, incompressible Navier-Stokes equations coupled with an equation for the conservation of energy. No-slip boundary conditions are prescribed for the velocity on the walls of the cavity.

\begin{figure}
    \centering
    \includegraphics[width=0.6\textwidth]{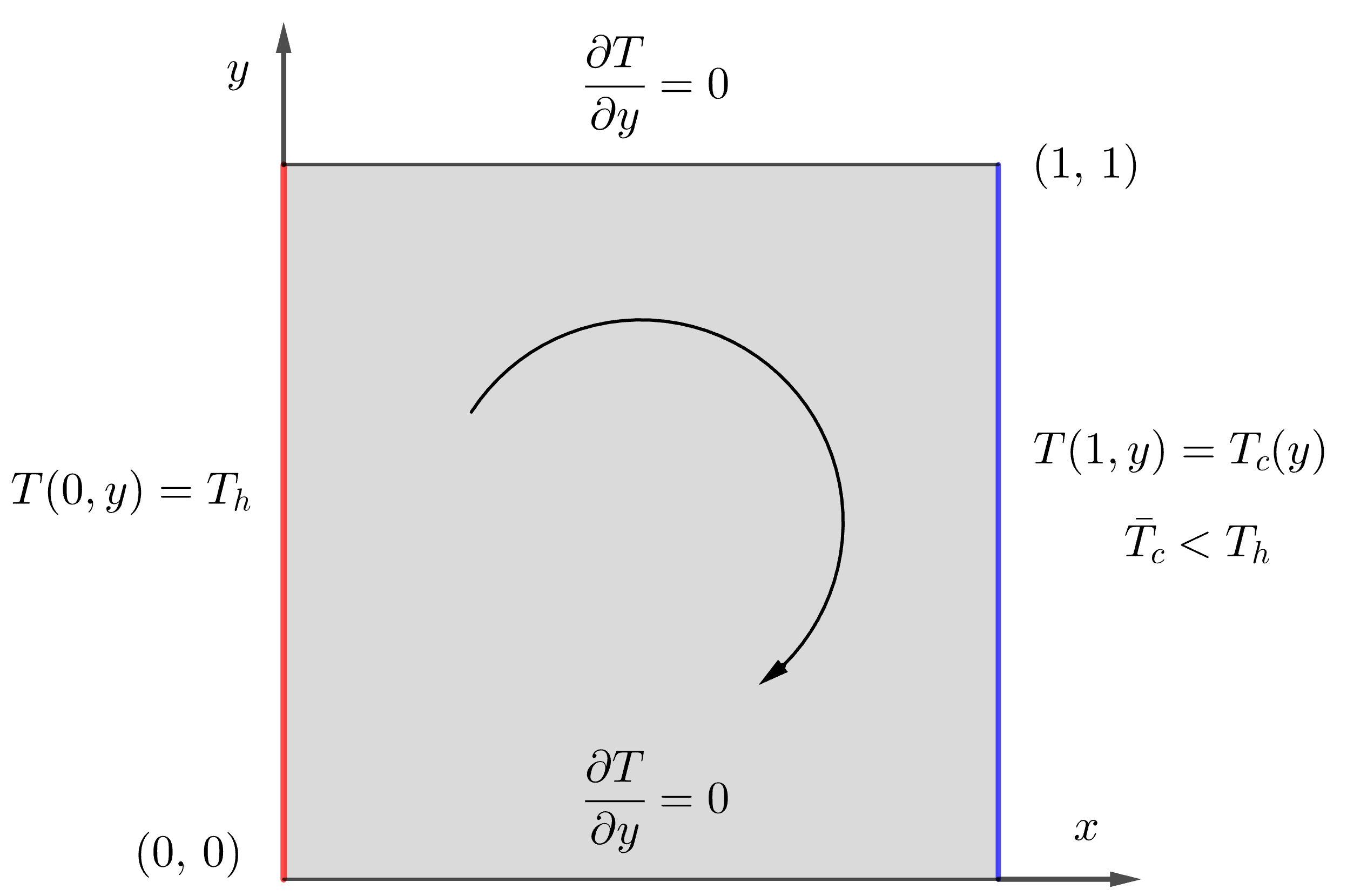}
    \caption{Schematic of temperature-driven cavity flow (Case 5).
    \new{The difference of temperature of the vertical walls onsets a clock-wise motion of the fluid inside the cavity. The parameters are the temperature at the hot left wall, and the distribution of temperature on the cold right wall. The quantity of interest is the heat flux through the hot wall at steady state as a function of the vertical coordinate.}
    }
    \label{fig:cavity-schematic}
\end{figure}

\paragraph*{Low- and high-fidelity models} 
The low- and high-fidelity data are obtained by solving the Navier-Stokes equations using a finite volume method with $16\times16$ and $256\times256$ grids, respectively. 
\new{The ratio between the computational costs of the two models is approximately 9,400 \cite{CHENG2024116793}.}
The reader is referred to \cite{CHENG2024116793} for more details regarding the problem and the data generation process.

\paragraph*{Parameters and quantities of interest}
The input to this problem includes the temperature prescribed at the hot wall and an instance of the temperature profile prescribed at the cold wall. The latter is expressed using a truncated Karhunen–Loève approximation to an underlying stochastic process. 

After an initial transitory phase, the system reaches a steady state, where all variables become independent of time.
We select the steady-state heat flux at the hot boundary, that is $\varphi_{q}(y) = k_{T} \frac{\partial T}{\partial x}(0,y,\infty)$, as the quantity of interest. In the expression above, $k_{T}$ is the thermal conductivity of the fluid. The heat flux is represented on a uniform grid of 221 points for the both the low- and high-fidelity models. Therefore, for this problem the dimension of the quantity of interest is $D = $ \casefivedim{}.

\paragraph*{Numerical results}
The low-fidelity dataset contains $\nbar=$ \casefivelfn{} data points, and $\nhat=$ \casefivehfn{} high-fidelity points are utilized to determine the multi-fidelity model (\casefivefractionn{} of the total). We note that each data point is obtained by solving the Navier-Stokes equations. 
Also in this case, we do not employ any data normalization procedure.
In Figure \ref{fig:cavity-comparison}, we have plotted four instances of the low-, high and multi-fidelity heat flux distributions. In each case we note that the low-fidelity solution over-predicts the heat flux and is unable to capture the subtle variations as a function of the vertical coordinate. The multi-fidelity approach improves on both these aspects and is much closer to the reference high-fidelity solution.

In Figure \ref{fig:cavity-err-dist} we plot the histogram for the error distributions (defined in Eq. (\ref{eq:error-field})) for the low-and multi-fidelity data for this problem. Once again, the multi-fidelity distribution is concentrated in a region where the error is much smaller. 
\new{The mean and standard deviation of the errors for the low-fidelity model are \casefivelferr{}$\%$ and \casefivelfstd{}, respectively, while for the multi-fidelity model we have \casefivemferr{}$\%$ and \casefivemfstd{}, respectively (see Table \ref{tab:summary}). The percentage reduction in the mean error is  \casefiveerrred{}.}

\begin{figure}
    \centering
    \includegraphics[width=0.99\textwidth]{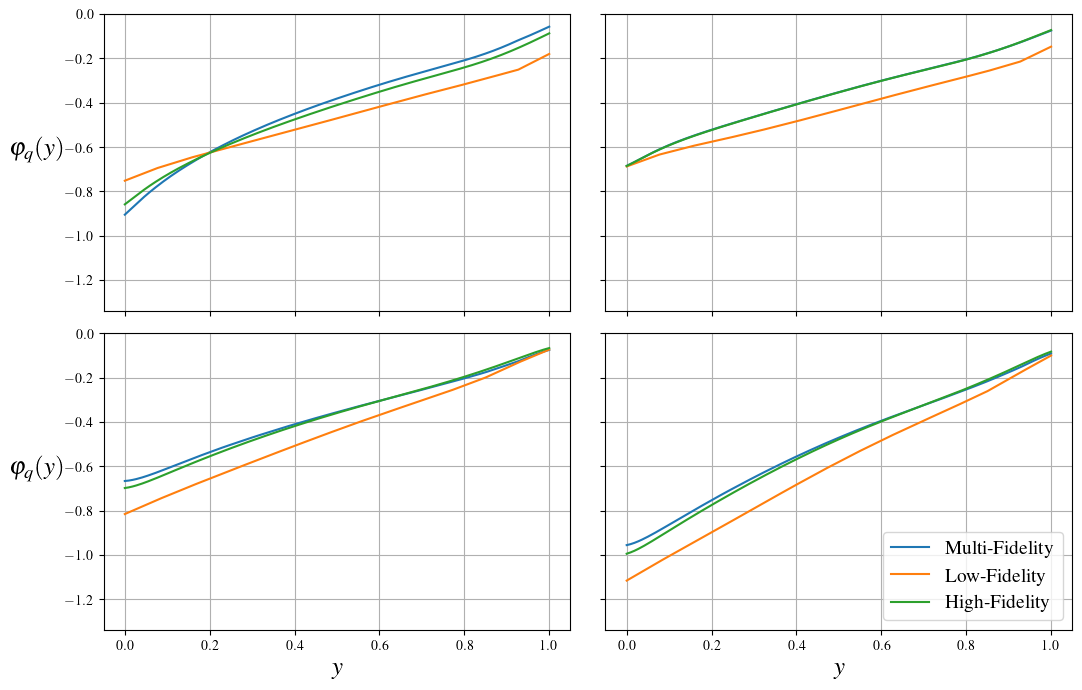}
    \caption{
    \new{Instances of heat flux curves corresponding to different input parameters for low-, multi-, and high-fidelity models for the cavity flow problem (Case 5). The multi-fidelity approach results in curves that are closer in  magnitude and trend to the high-fidelity solutions.}
    }
    \label{fig:cavity-comparison}
\end{figure}

\begin{figure}
    \centering
    \includegraphics[width=0.8\textwidth]{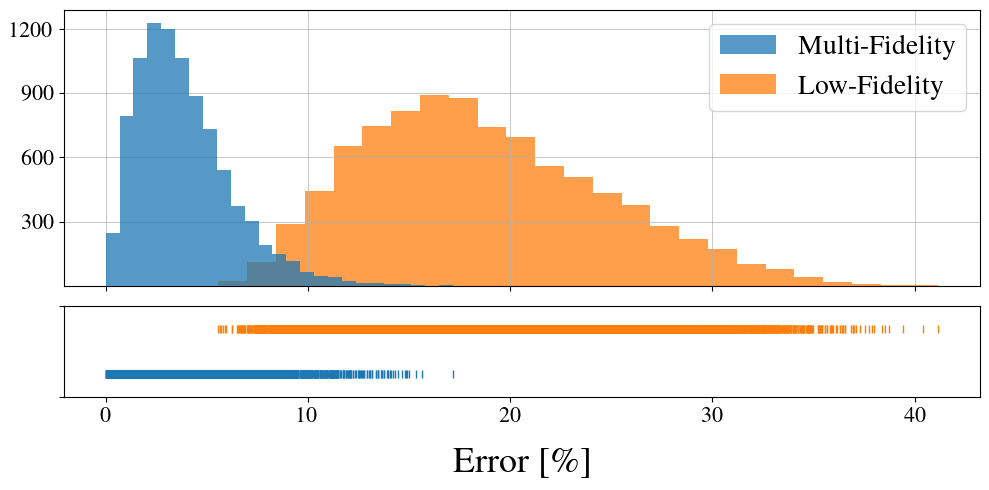}
    \caption{
    \new{Histogram of error distribution for low- and multi-fidelity datasets for the cavity flow problem (Case 5). The mean and standard deviation for the low-fidelity distribution are \casefivelferr{}\%  and \casefivelfstd{}, respectively. The mean and standard deviation of the distribution for the multi-fidelity data are \casefivemferr{}\%  and \casefivemfstd{}, respectively.}
    }
    \label{fig:cavity-err-dist}
\end{figure}

\subsection{Uncertainty of the multi-fidelity estimates}
In this section we visualize the uncertainty distribution of the multi-fidelity estimates for all the numerical problems. Specifically, we utilize the Uniform Manifold Approximation and Projection (UMAP) method \cite{McInnes2018} to project the multi-fidelity data onto a two-dimensional plane. UMAP is a dimensionality reduction technique that transforms high-dimensional data into a lower-dimensional space, preserving the local structure by keeping similar data points close together and dissimilar ones farther apart.

For each problem, we start with embedding the multi-fidelity data in the space of the graph Laplacian eigenfunctions (as described in Section \ref{sec:selection-hf}), and then use UMAP to project this data into two dimensions. In the resulting two-dimensional UMAP embedding, each data point $i \in \{1, \dots, \nbar\}$ is colored according to its variance, $\sqrt{C_{ii}}$. 
Figure \ref{fig:UMAP-1} shows these plots for the five numerical problems. 
From these visualizations, two main patterns emerge. First, the variance is lowest around the points where high-fidelity is available (marked in red), and increases as we move away, indicating that the model is more confident in its predictions for points that are closer to the training points. Second, the uncertainty tends to be higher for points that are isolated and not in close proximity to other data points.

\begin{figure}
    \centering
    \includegraphics[width=0.45\textwidth]{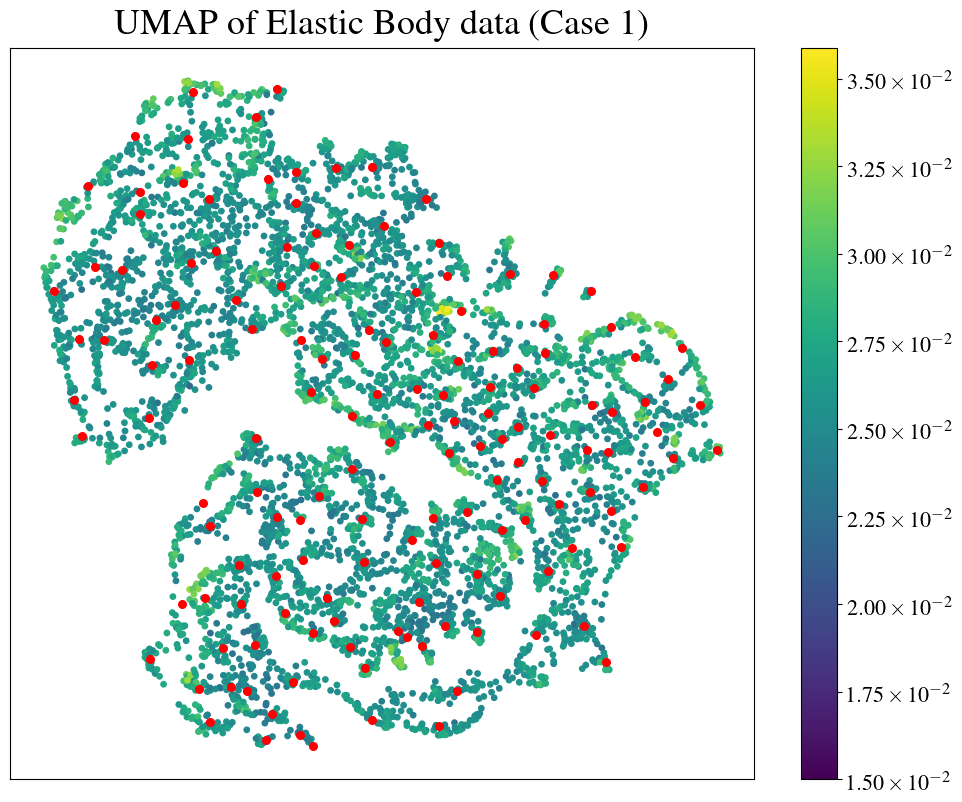}
    \includegraphics[width=0.45\textwidth]{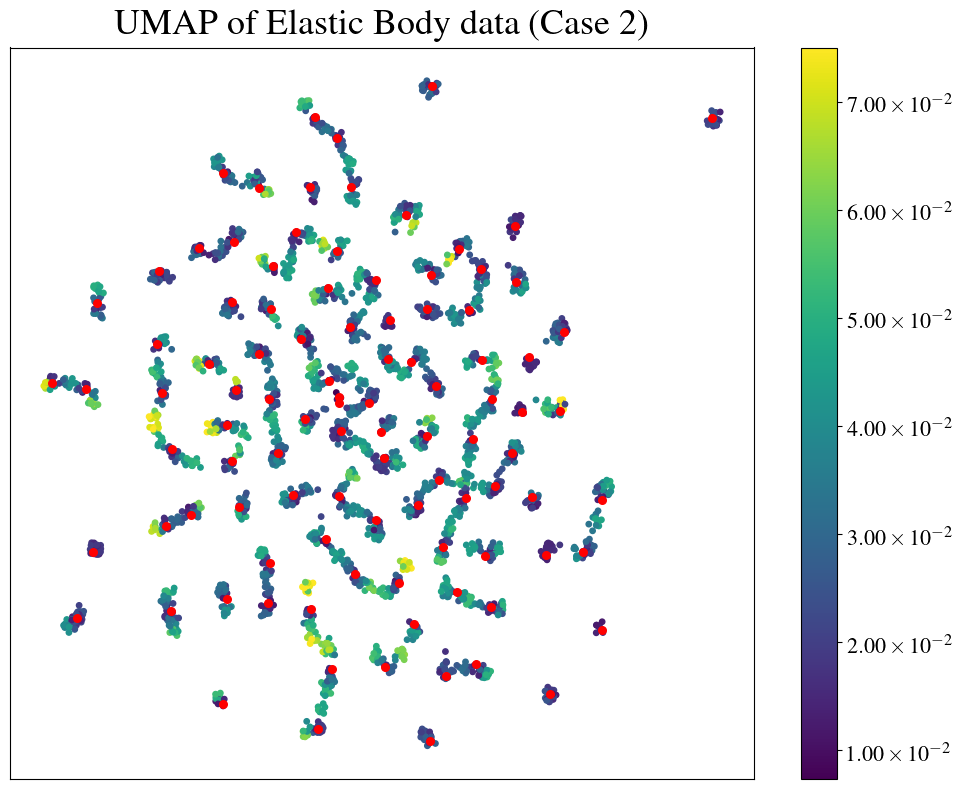}\\
    \includegraphics[width=0.45\textwidth]{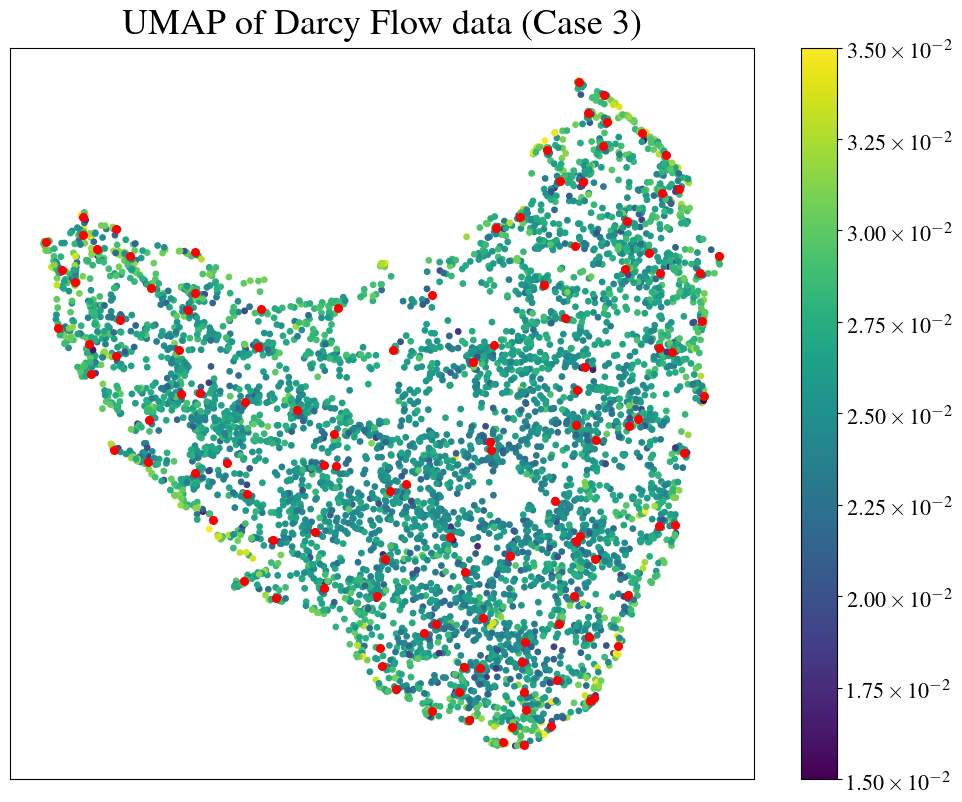}
    \includegraphics[width=0.45\textwidth]{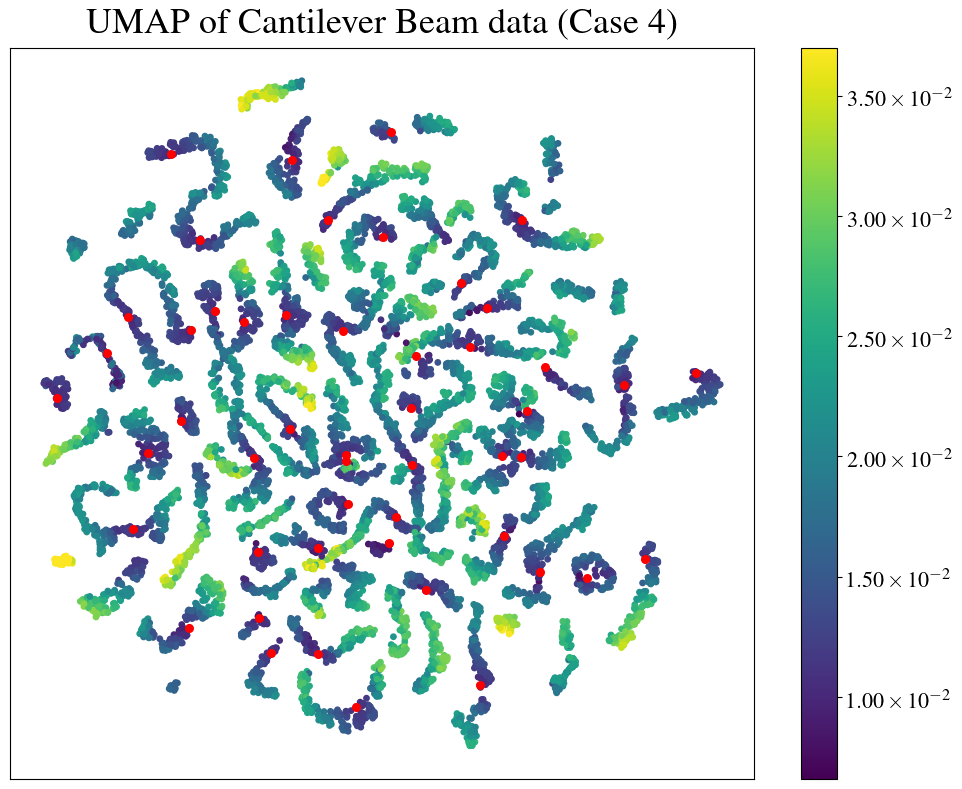}\\
    \includegraphics[width=0.45\textwidth]{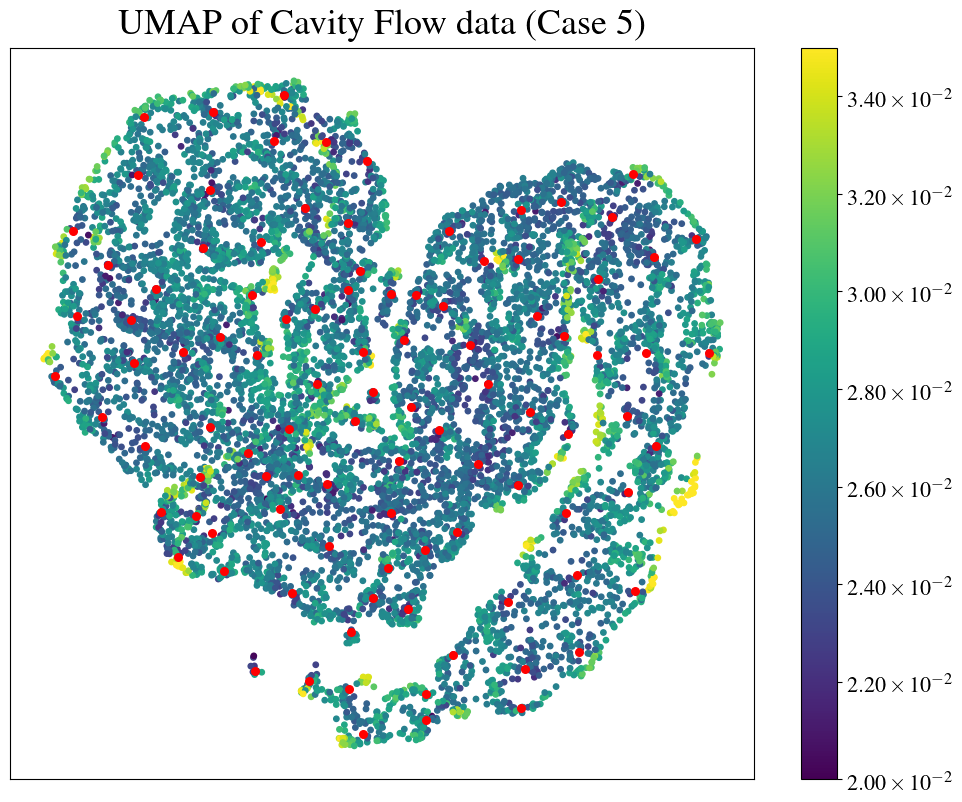}
    \caption{
    \new{The figures show the 2-dimensional UMAP embedding of the multi-fidelity data for the five numerical problems we considered. Here, the nodes of the multi-fidelity graph are mapped into a plane, and the adjacency matrix of the graph is used in the UMAP method to impose a distance metric. This ensures that nodes that are similar are mapped close to each other, and nodes that are dissimilar are mapped far apart. Finally, each data point is colored based on the uncertainty (standard deviation) determined  by the multi-fidelity method.}
    }
    \label{fig:UMAP-1}
\end{figure}

\section{Conclusions}
\label{sec:conclcusions}
In this study, we propose a Bayesian extension of the graph Laplacian-based spectral multi-fidelity model (SpecMF) \cite{pinti2023graph}. By leveraging the spectral properties of the graph Laplacian, this method constructs a prior distribution that captures the underlying structure of the data, which is then refined using a likelihood function determined using a few select high-fidelity data points. The resulting posterior distribution for the multi-fidelity coordinates of the data points is a multivariate Gaussian distribution, and the mean and the covariance of this distribution can be determined by solving linear systems of equations. We present two efficient numerical methods for solving these systems. 
We apply the method to a wide range of problems in solid and fluid mechanics where the quantities of interest range from low-dimensional vectors ($D=\mathcal{O}(10)$) to discrete representations of two-dimensional fields ($D=\mathcal{O}(10^4)$). In all cases the multi-fidelity approach improves the accuracy of the underlying low-fidelity data by 75 to 85\%, while only using a small proportion of high-fidelity  data (from 0.5 to 3.3\%).

\new{
We now outline the limitations of the methodology, which can be seen as possible directions for future research. 
The method depends on the low-fidelity model being able to accurately capture the relationships across different data points leading to an accurate adjacency matrix, which directly affects the multi-fidelity correction.
Furthermore, the method may also face challenges when used for hyperbolic problems that exhibit a large Kolmogorov $n$-width, as it may require a substantial number of high-fidelity data points. This is due to the fact that the additive correction used to compute the multi-fidelity estimates is a linear combination of the known low-to-high fidelity displacements, and the number of terms may increasing with increasing Kolmogorov $n$-width. 
Our approach scales effectively with the dimensionality of the data, which only affects the complexity of computing the edge weights of the graph.\footnote{\new{Parallelization strategies for efficiently computing edge weights for very high-dimensional data are discussed in Meng \emph{et al.} \cite{MengIPOL2017}.}} 
Finally, we propose efficient numerical techniques for applying our approach in a large data regime, but future efforts to prove the efficacy of these methods with large datasets should we explored.
}

\appendix

\section{General graph Laplacian normalization}
\label{app:GLnormalization}
In the above, we have assumed that $p=q$ in the definition of $\gl$ in \eqref{eq:GLdef}. Other normalizations of the graph Laplacian are possible and indeed common for different application settings, for example $(p,q) = (1,0)$ gives the random walk Laplacian, so-called because $\degmat^{-1}\adj$ is the transition matrix of a random walk on the graph. It is well-known that the choice of normalization has a direct impact on the output of any graph Laplacian based algorithm, and a spectral analysis of the effects of the choice of normalization in the large-data limit is given in Hoffmann \emph{et al.} \cite{Hoffmann2019SpectralAO}.

In this appendix, we will sketch how to generalise our above framework to the general $p,q\in \mathbb{R}$ case. First, we will redenote $\gl$ by $\glpq{p}{q}$, to keep track of our normalization parameters. 
Next, we define the following inner products (for any $K \in \mathbb{N}$):
\begin{align}
    \ip{\bm{u},\bm{v}} &:= \bm{u}^T \degmat^{p-q}\bm{v} &&\bm{u},\bm{v} \in \mathbb{R}^\nbar,\\
    \ipF{\bm{A},\bm{B}} &:= \operatorname{tr}\left(\bm{A}^T \degmat^{p-q}\bm{B}\right) &&\bm{A},\bm{B} \in \mathbb{R}^{\nbar \times K}, 
\end{align}
which are a reweighted dot product and Frobenius inner product, respectively. These are defined so that 
\begin{align*}
    \ip{\bm{u},\glpq{p}{q}\bm{v}} &= \bm{u}^T\bm{L}^{(q,q)}\bm{v} &\text{and} &&  \ipF{\bm{A},\glpq{p}{q}\bm{B}} &= \operatorname{tr}\left(\bm{A}^T\bm{L}^{(q,q)}\bm{B} \right) \hspace{1em}&\\
    &= \bm{v}^T\bm{L}^{(q,q)}\bm{u} &\text{} && &=\operatorname{tr}\left(\bm{B}^T\bm{L}^{(q,q)}\bm{A} \right) \hspace{1em}&\\
    & = \ip{\bm{v},\glpq{p}{q}\bm{u}} &&& &=\ipF{\bm{B},\glpq{p}{q}\bm{A}}, &
\end{align*}
and so $\glpq{p}{q}$ is self-adjoint with respect to these inner products. The matrix $\glpq{p}{q}$ is not symmetric, however it is similar to the symmetric graph Laplacian  $\glpq{(p+q)/2}{(p+q)/2}$. By eigendecomposing the latter as 
\begin{equation}
    \bm{L}^{((p+q)/2,(p+q)/2)} = \bm{\Pi \Lambda \Pi^T}\,,
\end{equation}
(where $\bm{\Lambda}$ is the diagonal matrix of eigenvalues, $\bm{\Pi}$ is a matrix whose columns are the eigenvectors of $\glpq{(p+q)/2}{(p+q)/2}$, and $\bm{\Pi}^T \bm{\Pi} = \idnbar$) we can write
\begin{equation}
    \glpq{p}{q} = \degmat^{-(p-q)/2}\bm{L}^{((p+q)/2,(p+q)/2)} \degmat^{(p-q)/2} = \bm{\Psi}
\bm{\Lambda} \tilde{\bm{\Psi}}^T\,,
\end{equation}
where $\bm{\Psi} := \degmat^{-(p-q)/2}\bm{\Pi}$ is now a matrix of eigenvectors of $\glpq{p}{q}$, $\tilde{\bm{\Psi}} := \degmat^{p-q}\bm{\Psi}$, and $\tilde{\bm{\Psi}}^T \bm{\Psi} = \bm{\Psi}\tilde{\bm{\Psi}}^T  = \bm{\Psi}^T \degmat^{p-q} \bm{\Psi}  = \idnbar$. A key observation is that the columns of $\bm{\Psi}$, i.e. the eigenvectors of $\glpq{p}{q}$, are thus orthonormal \emph{with respect to the $\ip{\cdot,\cdot}$ inner product}, rather than with respect to the standard dot product.

We then redefine our prior \eqref{prior} for $\bm{\Phi}$ using the reweighted Frobenius norm, i.e., 
\begin{align}
p(\phimat) \propto \exp \left( -\frac{\omega}{2} 
\ipF{\phimat, \left(\glpq{p}{q} + \tau \idnbar \right)^{\beta} \phimat }
\right).
\end{align}
Since the inner product (by design) respects the orthonormality of the eigenvectors, the expression for the prior in terms of the coefficient matrix \eqref{eq:pA} is unchanged. Choosing the likelihood as in \eqref{likelihood}, we can express the posterior as in \eqref{eq:posteriorC}, except now 
\begin{align} \label{eq:covpq}
\phicov &= \left( \frac{1}{\sigma^2} \pnhat^T \pnhat + \omega \degmat^{p-q} (\glpq{p}{q} + \tau \idnbar)^{\beta} \right)^{-1}, 
\end{align}
and so the MAP estimate solves 
\begin{equation} \label{eq:MAPpq}
\left( \frac{1}{\sigma^2} \pnhat^T \pnhat + \omega{\degmat^{p-q}} (\glpq{p}{q} + \tau \idnbar)^\beta\right)\phimapm = \frac{1}{\sigma^2} \pnhat^T \phihatm.
\end{equation}

In the truncation method in \Cref{sec:truncation}, the leading (orthonormal) eigenvectors and eigenvalues should instead be computed of the symmetric matrix
\begin{equation}
a \idnbar - \glpq{(p+q)/2}{(p+q)/2},
\end{equation}
before being converted to low-lying $\ip{\cdot,\cdot}$-orthonormal eigenvectors of $\glpq{p}{q}$ by multiplying each leading eigenvector of $a \idnbar - \glpq{(p+q)/2}{(p+q)/2}$ by $\degmat^{-(p-q)/2}$. This $\ip{\cdot,\cdot}$-orthonormality, combined with the redefinition of the prior density above, ensures that the eventual expressions for $\phicov_{A_K}$ and $\bm{A}^*_K$ in \eqref{eq:cov_a} and \eqref{eq:map_a} are unchanged, except that $\bm{\Psi}_K$ now refers to the low-lying eigenvectors of $\glpq{p}{q}$. 

In the Nystr\"om-based method of \Cref{sec:Nys} (which is in fact applicable for all $p$ and $q$ satisfying $p+q = 1$) the set-up begins as before, deriving $\tilde{\bm{U}}$ and $\bm{\Sigma}$ such that 
\begin{equation}
\tilde{\bm{U}}\bm{\Sigma}\tilde{\bm{U}}^T = \hat \degmat^{-\frac12} \adj(:,X)\adj(X,X)^{\dagger}\adj(X,:) \hat \degmat^{-\frac12},
\end{equation}
however we must now approximate
\begin{equation}
\glpq{p}{1-p} \approx \idnbar - \hat \degmat^{\frac12-p}\tilde {\bm{U}}\bm{\Sigma}\tilde{\bm{U}}^T\hat \degmat^{p-\frac12} = \idnbar - \bm{U}\bm{\Sigma} \bm{V}^T,
\end{equation}
where $\bm{U}:=\hat\degmat^{\frac12-p}\tilde {\bm{U}}$, $\bm{V}:=\hat\degmat^{p-\frac12}\tilde {\bm{U}}$, and $\bm{V}^T\bm{U}= \bm{I}_K$. The approximation of $(\glpq{p}{1-p} + \tau \idnbar)^\beta$ then goes through with $\tilde {\bm{U}}$ replaced by $\bm{U}$ and $\tilde {\bm{U}}^T$ replaced by $\bm{V}^T$. Redefining (in line with the $\degmat^{p-q}$ term in \eqref{eq:covpq})
\begin{equation}
\bm{\Theta} := \pnhat^T \pnhat + \sigma^2 \omega(1 + \tau)^\beta \hat{\degmat}^{2p-1} \in \mathbb{R}^{\nbar  \times \nbar},
\end{equation}
it follows that \cref{eq:MAPpq} can be approximately solved by solving  
\begin{equation}
\label{eq:MAPnyspq}
    \left(\bm{\Theta} -  \hat{\degmat}^{2p-1}\bm{U\Xi V}^T \right)\phimapm =\left(\bm{\Theta} - \bm{V\Xi V}^T \right)\phimapm= \pnhat^T \phihatm,
\end{equation}
since $\hat{\degmat}^{2p-1}\bm{U} = \hat{\degmat}^{p-1/2}\tilde {\bm{U}}= \bm{V}$, and the covariance matrix can be approximated 
\begin{equation}
\bm{C} = \left( \frac{1}{\sigma^2} \pnhat^T \pnhat + \omega \degmat^{2p-1} (\gl + \tau \idnbar)^\beta\right)^{-1} \approx  \sigma^2 \left( \bm{\Theta} - \bm {V} \bm{\Xi V}^T \right)^{-1}. 
\end{equation}
Hence, the methods for solving \eqref{eq:MAPnyspq} and computing $\phicov$ are as before, except for $\tilde{\bm{U}}$ everywhere replaced by $\bm{V}$. The computational complexities in \Cref{sec:evaluating} are unaffected by any of these changes. 

Finally, in \Cref{sec:analysis}, to ensure that the minimizer of the objective function solves \eqref{eq:MAPpq}, the regularizer must be redefined as
\begin{equation}
\mathcal{R}(\bm{\Theta}):= \ipF{ \mathbf{\Theta}, (\glpq{p}{q} + \tau \idnbar)^\beta\mathbf{\Theta}},
\end{equation}
which satisfies all of the required properties for the proof of \cref{thm:cv-regularizer} to go through unchanged.

\new{
\section{Gradient of $\mathcal{J}(\omega)$}
\label{app:omega_cond_grad}
We are interested in finding the gradient of function (\ref{eq:omega_loss_appx}) with respect to $\omega$:
\begin{equation}
\mathcal{J}(\omega) = \left( \frac{1}{\nbar} \sum_{k=1}^{\nbar} \sqrt{C_{kk}(\omega)} - r \sigma \right)^2.
\label{eq:omega_loss_appx}
\end{equation}

Let's first recall the form of the covariance matrix $\phicov$ (\ref{covariance}):
\begin{equation}
    \phicov^{-1} = \frac{1}{\sigma^2} \pnhat^T \pnhat + \omega (\gl + \tau \idnbar)^{\beta}.
\end{equation}

Then, we can write,
\begin{equation}    
\frac{\mathrm{d} \mathcal{J}}{\mathrm{d} \omega} = \sum_{i,\,j=1}^{\nbar} \frac{\partial \mathcal{J}}{\partial C_{ij}} \frac{\partial C_{ij}}{\partial \omega}.
\end{equation}

The first term writes
\begin{equation}
    \frac{\partial \mathcal{J}}{\partial C_{ij}} = \frac{1}{\nbar} \left( \frac{1}{\nbar} \sum_{k=1}^{\nbar} \sqrt{C_{kk}} - r \sigma \right) \frac{1}{\sqrt{C_{ii}}} \delta_{ij},
\label{eq:dJ_dC}
\end{equation}

with $\delta_{ij}$ being the Kronecker delta. For the second term, we note that
\begin{align}
    \frac{\partial \phicov}{\partial \omega} &= - \phicov \frac{\partial \phicov^{-1}} {\partial \omega} \phicov, \label{eq:dC_domega}\\
    \frac{\partial \phicov^{-1}}{\partial \omega} &= (\gl + \tau \idnbar)^{\beta}.
\end{align}

(\ref{eq:dC_domega}) is a standard result that can be proven by differentiating both sides of $\phicov  \phicov^{-1} = \idnbar$  with respect to $\omega$. Then, we have that
\begin{equation}
    \frac{\partial C_{ij}}{\partial \omega} = \left[ \frac{\partial \phicov}{\partial \omega} \right]_{ij}
    = - \left[\phicov (\gl + \tau \idnbar)^{\beta} \phicov \right]_{ij}
\end{equation}

Hence, 
\begin{equation}    
\frac{\mathrm{d} \mathcal{J}}{\mathrm{d} \omega} = - 
\left( \frac{1}{\nbar} \sum_{k=1}^{\nbar} \sqrt{C_{kk}} - r \sigma \right) 
\frac{1}{\nbar} 
\sum_{i=1}^{\nbar} \frac{1}{\sqrt{C_{ii}}} \left[\phicov (\gl + \tau \idnbar)^{\beta} \phicov \right]_{ii}.
\end{equation}
}

\bibliographystyle{elsarticle-num} 
\bibliography{refs}

@phdthesis{poschl2008tikhonov,
  title={Tikhonov regularization with general residual term},
  author={P{\"o}schl, Christiane},
school = {Leopold Franzens Universität Innsbruck},
  year={2008}
}

@book{Woodbury50,
    AUTHOR = {Woodbury, Max A.},
     TITLE = {Inverting modified matrices},
      NOTE = {Statistical Research Group, Memo. Rep. no. 42},
 PUBLISHER = {Princeton University, Princeton, N. J.},
      YEAR = {1950},
     PAGES = {4},
   MRCLASS = {65.0X},
  MRNUMBER = {0038136},
MRREVIEWER = {J. Kuntzmann},
}

@article{FowlkesBelongieChungMalik04,
  title={Spectral grouping using the {N}ystr{\"o}m method},
  author={Fowlkes, Charless and Belongie, Serge and Chung, Fan and Malik, Jitendra},
  journal={IEEE transactions on pattern analysis and machine intelligence},
  volume={26},
  number={2},
  pages={214--225},
  year={2004},
  publisher={IEEE}
}

@inproceedings{BelongieFowlkesChungMalik02,
  title={Spectral partitioning with indefinite kernels using the {N}ystr{\"o}m extension},
  author={Belongie, Serge and Fowlkes, Charless and Chung, Fan and Malik, Jitendra},
  booktitle={European conference on computer vision},
  pages={531--542},
  year={2002},
  organization={Springer}
}

@article{Nystrom30,
    title={{{\"U}ber Die Praktische Aufl\"osung von Integralgleichungen mit Anwendungen auf Randwertaufgaben}},
    author={Nystr\"om, E. J.},
    journal={Acta Math.},
    volume={54},
    pages={185--204},
    year={1930}
}

@article{BebendorfKunis09,
author = {Bebendorf, Mario and Kunis, Stefan},
title = {Recompression techniques for adaptive cross approximation},
volume = {21},
journal = {Journal of Integral Equations and Applications},
number = {3},
publisher = {Rocky Mountain Mathematics Consortium},
pages = {331--357},
year = {2009},
doi = {10.1216/JIE-2009-21-3-331}
}

@article{AlfkePottsStollVolkmer18,
author={Alfke, Dominik and Potts, Daniel and Stoll, Martin and Volkmer, Toni},   
title={N{FFT} Meets {K}rylov Methods: Fast Matrix-Vector Products for the Graph {L}aplacian of Fully Connected Networks},      
journal={Frontiers in Applied Mathematics and Statistics},      
volume={4},      
year={2018},         
issn={2297-4687}
}

@article{perdikaris2015multi,
  title={Multi-fidelity modelling via recursive co-kriging and {G}aussian--{M}arkov random fields},
  author={Perdikaris, Paris and Venturi, Daniele and Royset, Johannes O and Karniadakis, George Em},
  journal={Proceedings of the Royal Society A: Mathematical, Physical and Engineering Sciences},
  volume={471},
  number={2179},
  pages={20150018},
  year={2015},
  publisher={The Royal Society Publishing}
}

@article{Goel2007,
    author={Goel, T. and Haftka, R. T. and Shyy, W. and Queipo, N. V.},
    title={Ensemble of surrogates},
    journal={Structural and Multidisciplinary Optimization},
    year={2007},
    month={3},
    day={01},
    volume={33},
    number={3},
    pages={199-216},
    issn={1615-1488},
    doi={10.1007/s00158-006-0051-9}
    }

@article{gramacy2009adaptive,
  title={Adaptive design and analysis of supercomputer experiments},
  author={Gramacy, Robert B and Lee, Herbert KH},
  journal={Technometrics},
  volume={51},
  number={2},
  pages={130--145},
  year={2009},
  publisher={Taylor \& Francis}
}

@article{Ng2014,
author = {Ng, L. W. T. and Willcox, K. E.},
title = {Multifidelity approaches for optimization under uncertainty},
journal = {International Journal for Numerical Methods in Engineering},
volume = {100},
number = {10},
pages = {746-772},
doi = {10.1002/nme.4761},
year = {2014}
}

@article{Allaire_2014,
  title={A mathematical and computational framework for multifidelity design and analysis with computer models},
  author={Allaire, Douglas and Willcox, Karen},
  journal={International Journal for Uncertainty Quantification},
  volume={4},
  number={1},
  year={2014},
  publisher={Begel House Inc.}
}

@article{Kaipio2007,
title = {Statistical inverse problems: Discretization, model reduction and inverse crimes},
journal = {Journal of Computational and Applied Mathematics},
volume = {198},
number = {2},
pages = {493-504},
year = {2007},
note = {Special Issue: Applied Computational Inverse Problems},
issn = {0377-0427},
doi = {10.1016/j.cam.2005.09.027},
author = {Kaipio, Jari and Somersalo, Erkki}
}

@article{hoffmann2020consistency,
  title={Consistency of semi-supervised learning algorithms on graphs: Probit and one-hot methods},
  author={Hoffmann, Franca and Hosseini, Bamdad and Ren, Zhi and Stuart, Andrew M},
  journal={Journal of Machine Learning Research},
  volume={21},
  pages={1--55},
  year={2020},
  publisher={Journal of Machine Learning Research}
}

@article{bertozzi2018uncertainty,
  title={Uncertainty quantification in graph-based classification of high dimensional data},
  author={Bertozzi, Andrea L and Luo, Xiyang and Stuart, Andrew M and Zygalakis, Konstantinos C},
  journal={SIAM/ASA Journal on Uncertainty Quantification},
  volume={6},
  number={2},
  pages={568--595},
  year={2018},
  publisher={SIAM}
}

@article{belkin2001laplacian,
  title={{L}aplacian eigenmaps and spectral techniques for embedding and clustering},
  author={Belkin, Mikhail and Niyogi, Partha},
  journal={Advances in neural information processing systems},
  volume={14},
  year={2001}
}

@article{belkin2006convergence,
  title={Convergence of {L}aplacian eigenmaps},
  author={Belkin, Mikhail and Niyogi, Partha},
  journal={Advances in neural information processing systems},
  volume={19},
  year={2006}
}

@article{COIFMAN20065,
	author = {Ronald R. Coifman and St{\'e}phane Lafon},
	doi = {10.1016/j.acha.2006.04.006},
	issn = {1063-5203},
	journal = {Applied and Computational Harmonic Analysis},
	number = {1},
	pages = {5-30},
	title = {Diffusion maps},
	volume = {21},
	year = {2006}}

@article{slepcev2019analysis,
  title={Analysis of p-{L}aplacian regularization in semisupervised learning},
  author={Slepcev, Dejan and Thorpe, Matthew},
  journal={SIAM Journal on Mathematical Analysis},
  volume={51},
  number={3},
  pages={2085--2120},
  year={2019},
  publisher={SIAM}
}

@inproceedings{belkin2004regularization,
  title={Regularization and semi-supervised learning on large graphs},
  author={Belkin, Mikhail and Matveeva, Irina and Niyogi, Partha},
  booktitle={International Conference on Computational Learning Theory},
  pages={624--638},
  year={2004},
  organization={Springer}
}

@article{bertozzi2021posterior,
  title={Posterior consistency of semi-supervised regression on graphs},
  author={Bertozzi, Andrea L and Hosseini, Bamdad and Li, Hao and Miller, Kevin and Stuart, Andrew M},
  journal={Inverse Problems},
  volume={37},
  number={10},
  pages={105011},
  year={2021},
  publisher={IOP Publishing}
}

@article{NakatsukasaPark23,
author = {Nakatsukasa, Yuji and Park, Taejun},
title = {Randomized Low-Rank Approximation for Symmetric Indefinite Matrices},
journal = {SIAM Journal on Matrix Analysis and Applications},
volume = {44},
number = {3},
pages = {1370--1392},
year = {2023},
doi = {10.1137/22M1538648}
}

@article{dunlop2020large,
  title={Large data and zero noise limits of graph-based semi-supervised learning algorithms},
  author={Dunlop, Matthew M and Slep{\v{c}}ev, Dejan and Stuart, Andrew M and Thorpe, Matthew},
  journal={Applied and Computational Harmonic Analysis},
  volume={49},
  number={2},
  pages={655--697},
  year={2020},
  publisher={Elsevier}
}

@article{trillos2018variational,
  title={A variational approach to the consistency of spectral clustering},
  author={Trillos, Nicolas Garcia and Slep{\v{c}}ev, Dejan},
  journal={Applied and Computational Harmonic Analysis},
  volume={45},
  number={2},
  pages={239--281},
  year={2018},
  publisher={Elsevier}
}

@article{sarvazyan2012mechanical,
  title={Mechanical imaging-a technology for 3-d visualization and characterization of soft tissue abnormalities: A review},
  author={Sarvazyan, Armen and Egorov, Vladimir},
  journal={Current Medical Imaging},
  volume={8},
  number={1},
  pages={64--73},
  year={2012},
  publisher={Bentham Science Publishers}
}

@inproceedings{MuscoMusco15,
author = {Musco, Cameron and Musco, Christopher},
title = {Randomized block {K}rylov methods for stronger and faster approximate singular value decomposition},
year = {2015},
publisher = {MIT Press},
address = {Cambridge, MA, USA},
booktitle = {Proceedings of the 28th International Conference on Neural Information Processing Systems - Volume 1},
pages = {1396–1404},
numpages = {9},
location = {Montreal, Canada},
series = {NIPS'15}
}

@article{barbone2010review,
  title={A review of the mathematical and computational foundations of biomechanical imaging},
  author={Barbone, Paul E and Oberai, Assad A},
  journal={Computational Modeling in Biomechanics},
  pages={375--408},
  year={2010},
  publisher={Springer}
}

@article{fernandez2016review,
  title={Review of multi-fidelity models},
  author={Fern{\'a}ndez-Godino, M Giselle and Park, Chanyoung and Kim, Nam-Ho and Haftka, Raphael T},
  journal={arXiv preprint arXiv:1609.07196},
  year={2016}
}

@article{Peherstorfer2016,
  title={Optimal model management for multifidelity {M}onte {C}arlo estimation},
  author={Peherstorfer, Benjamin and Willcox, Karen and Gunzburger, Max},
  journal={SIAM Journal on Scientific Computing},
  volume={38},
  number={5},
  pages={A3163--A3194},
  year={2016},
  publisher={SIAM},
  doi = {10.1137/15M1046472}
}

@article{zelnik2004self,
  title={Self-tuning spectral clustering},
  author={Zelnik-Manor, Lihi and Perona, Pietro},
  journal={Advances in Neural Information Processing Systems},
  volume={17},
  year={2004}
}

@article{durantin2017multifidelity,
  title={Multifidelity surrogate modeling based on radial basis functions},
  author={Durantin, C{\'e}dric and Rouxel, Justin and D{\'e}sid{\'e}ri, Jean-Antoine and Gli{\`e}re, Alain},
  journal={Structural and Multidisciplinary Optimization},
  volume={56},
  pages={1061--1075},
  year={2017},
  publisher={Springer},
  doi={10.1007/s00158-017-1703-7}
}

@article{forrester2007multi,
  title={Multi-fidelity optimization via surrogate modelling},
  author={Forrester, Alexander IJ and S{\'o}bester, Andr{\'a}s and Keane, Andy J},
  journal={Proceedings of the royal society a: mathematical, physical and engineering sciences},
  volume={463},
  number={2088},
  pages={3251--3269},
  year={2007},
  publisher={The Royal Society London},
  doi={10.1098/rspa.2007.1900}
}

@article{zhou2017variable,
  title={A variable fidelity information fusion method based on radial basis function},
  author={Zhou, Qi and Jiang, Ping and Shao, Xinyu and Hu, Jiexiang and Cao, Longchao and Wan, Li},
  journal={Advanced Engineering Informatics},
  volume={32},
  pages={26--39},
  year={2017},
  publisher={Elsevier},
  doi = {10.1016/j.aei.2016.12.005},
  issn = {1474-0346}
}

@article{song2019radial,
  title={A radial basis function-based multi-fidelity surrogate model: exploring correlation between high-fidelity and low-fidelity models},
  author={Song, Xueguan and Lv, Liye and Sun, Wei and Zhang, Jie},
  journal={Structural and Multidisciplinary Optimization},
  volume={60},
  pages={965--981},
  year={2019},
  publisher={Springer},
  doi={10.1007/s00158-019-02248-0}
}

@article{meng2021multi,
  title={Multi-fidelity {B}ayesian neural networks: Algorithms and applications},
  author={Meng, Xuhui and Babaee, Hessam and Karniadakis, George Em},
  journal={Journal of Computational Physics},
  volume={438},
  pages={110361},
  year={2021},
  doi = {10.1016/j.jcp.2021.110361},
  issn = {0021-9991},
  publisher={Elsevier}
}

@article{chakraborty2021transfer,
  title={Transfer learning based multi-fidelity physics informed deep neural network},
  author={Chakraborty, Souvik},
  journal={Journal of Computational Physics},
  volume={426},
  pages={109942},
  year={2021},
  publisher={Elsevier},
  issn={0021-9991},
  doi={10.1016/j.jcp.2020.109942}
}

@article{Geneva2020,
title = {Multi-fidelity generative deep learning turbulent flows},
journal = {Foundations of Data Science},
volume = {2},number = {4},pages = {391-428},
year = {2020},
issn = {},
doi = {10.3934/fods.2020019},
author = {Nicholas Geneva and Nicholas Zabaras},
}

@inproceedings{perron2020development,
  title={Development of a multi-fidelity reduced-order model based on manifold alignment},
  author={Perron, Christian and Rajaram, Dushhyanth and Mavris, Dimitri},
  booktitle={AIAA Aviation 2020 Forum},
  pages={3124},
  year={2020},
  doi={10.2514/6.2020-3124}
}

@article{li2020multi,
  title={Multi-fidelity {B}ayesian optimization via deep neural networks},
  author={Li, Shibo and Xing, Wei and Kirby, Robert and Zhe, Shandian},
  journal={Advances in Neural Information Processing Systems},
  volume={33},
  pages={8521--8531},
  year={2020},
  editor = {H. Larochelle and M. Ranzato and R. Hadsell and M.F. Balcan and H. Lin},
  publisher = {Curran Associates, Inc.}
}

@article{MENG2020109020,
	author = {Xuhui Meng and George Em Karniadakis},
	doi = {10.1016/j.jcp.2019.109020},
	issn = {0021-9991},
	journal = {Journal of Computational Physics},
	pages = {109020},
	title = {A composite neural network that learns from multi-fidelity data: Application to function approximation and inverse {PDE} problems},
	volume = {401},
	year = {2020}
 }

@article{RAISSI2019686,
	author = {M. Raissi and P. Perdikaris and G.E. Karniadakis},
	doi = {10.1016/j.jcp.2018.10.045},
	issn = {0021-9991},
	journal = {Journal of Computational Physics},
	pages = {686-707},
	title = {Physics-informed neural networks: A deep learning framework for solving forward and inverse problems involving nonlinear partial differential equations},
	volume = {378},
	year = {2019},
	}

@article{PENWARDEN2022110844,
	author = {Michael Penwarden and Shandian Zhe and Akil Narayan and Robert M. Kirby},
	doi = {10.1016/j.jcp.2021.110844},
	issn = {0021-9991},
	journal = {Journal of Computational Physics},
	pages = {110844},
	title = {Multifidelity modeling for Physics-Informed Neural Networks ({PINNs})},
	volume = {451},
	year = {2022}}

@article{park2017remarks,
  title={Remarks on multi-fidelity surrogates},
  author={Park, Chanyoung and Haftka, Raphael T and Kim, Nam H},
  journal={Structural and Multidisciplinary Optimization},
  volume={55},
  pages={1029--1050},
  year={2017},
  publisher={Springer},
  doi={10.1007/s00158-016-1550-y}
}

@article{kennedy2000predicting,
ISSN = {00063444},
 author={Kennedy, Marc C and O'Hagan, Anthony},
 journal = {Biometrika},
 number = {1},
 pages = {1--13},
publisher = {[Oxford University Press, Biometrika Trust]},
 title = {Predicting the Output from a Complex Computer Code When Fast Approximations Are Available},
 volume = {87},
 year = {2000},
 doi={10.1093/biomet/87.1.1}
}

@article{de2020transfer,
  title={On transfer learning of neural networks using bi-fidelity data for uncertainty propagation},
  author={De, Subhayan and Britton, Jolene and Reynolds, Matthew and Skinner, Ryan and Jansen, Kenneth and Doostan, Alireza},
  journal={International Journal for Uncertainty Quantification},
  volume={10},
  number={6},
  year={2020},
  pages={543-573},
  publisher={Begel House Inc.},
  doi={10.1615/Int.J.UncertaintyQuantification.2020033267}
}

@article{narayan2014stochastic,
  title={A stochastic collocation algorithm with multifidelity models},
  author={Narayan, Akil and Gittelson, Claude and Xiu, Dongbin},
  journal={SIAM Journal on Scientific Computing},
  volume={36},
  number={2},
  pages={A495--A521},
  year={2014},
  publisher={SIAM},
  doi={10.1137/130929461}
}

@article{keshavarzzadeh2019parametric,
  title={Parametric topology optimization with multiresolution finite element models},
  author={Keshavarzzadeh, Vahid and Kirby, Robert M and Narayan, Akil},
  journal={International Journal for Numerical Methods in Engineering},
  volume={119},
  number={7},
  pages={567--589},
  year={2019},
  publisher={Wiley Online Library},
  doi={10.1002/nme.6063}
}

@article{pinti2022multi,
  title={Multi-fidelity approach to predicting multi-rotor aerodynamic interactions},
  author={Pinti, Orazio and Oberai, Assad A and Healy, Richard and Niemiec, Robert J and Gandhi, Farhan},
  journal={AIAA Journal},
  volume={60},
  number={6},
  pages={3894--3908},
  year={2022},
  publisher={American Institute of Aeronautics and Astronautics},
  doi={10.2514/1.J060227}
}

@article{giles_2015, title={Multilevel {M}onte {C}arlo methods}, volume={24}, DOI={10.1017/S096249291500001X}, journal={Acta Numerica}, publisher={Cambridge University Press}, author={Giles, Michael B.}, year={2015}, pages={259–328}}

@inproceedings{heinrich2001multilevel,
  title={Multilevel {M}onte {C}arlo methods},
  author={Heinrich, Stefan},
  booktitle={Large-Scale Scientific Computing: Third International Conference, LSSC 2001 Sozopol, Bulgaria, June 6--10, 2001 Revised Papers 3},
  pages={58--67},
  year={2001},
  organization={Springer}
}

@article{allaire2014mathematical,
  title={A mathematical and computational framework for multifidelity design and analysis with computer models},
  author={Allaire, Douglas and Willcox, Karen},
  journal={International Journal for Uncertainty Quantification},
  volume={4},
  number={1},
  year={2014},
  publisher={Begel House Inc.}
}

@article{goel2007ensemble,
  title={Ensemble of surrogates},
  author={Goel, Tushar and Haftka, Raphael T and Shyy, Wei and Queipo, Nestor V},
  journal={Structural and Multidisciplinary Optimization},
  volume={33},
  pages={199--216},
  year={2007},
  publisher={Springer}
}

@article{gramacy2016modeling,
  title={Modeling an augmented {L}agrangian for blackbox constrained optimization},
  author={Gramacy, Robert B and Gray, Genetha A and Le Digabel, S{\'e}bastien and Lee, Herbert KH and Ranjan, Pritam and Wells, Garth and Wild, Stefan M},
  journal={Technometrics},
  volume={58},
  number={1},
  pages={1--11},
  year={2016},
  publisher={Taylor \& Francis}
}

@article{gramacy2008bayesian,
  title={Bayesian treed {G}aussian process models with an application to computer modeling},
  author={Gramacy, Robert B and Lee, Herbert K H},
  journal={Journal of the American Statistical Association},
  volume={103},
  number={483},
  pages={1119--1130},
  year={2008},
  publisher={Taylor \& Francis}
}

@inproceedings{geraci2017multifidelity,
  title={A multifidelity multilevel {M}onte {C}arlo method for uncertainty propagation in aerospace applications},
  author={Geraci, Gianluca and Eldred, Michael S and Iaccarino, Gianluca},
  booktitle={19th AIAA non-deterministic approaches conference},
  pages={1951},
  year={2017}
}

@article{pinti2023graph,
  title={Graph {L}aplacian-based spectral multi-fidelity modeling},
  author={Pinti, Orazio and Oberai, Assad A},
  journal={Scientific Reports},
  volume={13},
  number={1},
  pages={16618},
  year={2023},
  publisher={Nature Publishing Group UK London}
}

@article{CHENG2024116793,
title = {Bi-fidelity variational auto-encoder for uncertainty quantification},
journal = {Computer Methods in Applied Mechanics and Engineering},
volume = {421},
pages = {116793},
year = {2024},
issn = {0045-7825},
doi = {10.1016/j.cma.2024.116793},
author = {Nuojin Cheng and Osman Asif Malik and Subhayan De and Stephen Becker and Alireza Doostan}
}

@article{Huang-KL,
author = {Huang, S. P. and Quek, S. T. and Phoon, K. K.},
title = {Convergence study of the truncated {K}arhunen–{L}oeve expansion for simulation of stochastic processes},
journal = {International Journal for Numerical Methods in Engineering},
volume = {52},
number = {9},
pages = {1029-1043},
doi = {10.1002/nme.255},
year = {2001}
}

@article{Stefanou2007AssessmentOS,
  title={Assessment of spectral representation and {K}arhunen–{L}o{\`e}ve expansion methods for the simulation of {G}aussian stochastic fields},
  author={George Stefanou and Manolis Papadrakakis},
  journal={Computer Methods in Applied Mechanics and Engineering},
  year={2007},
  volume={196},
  pages={2465-2477}
}

@article{McInnes2018, 
doi = {10.21105/joss.00861}, 
year = {2018}, 
publisher = {The Open Journal},
volume = {3},
number = {29},
pages = {861}, 
author = {Leland McInnes and John Healy and Nathaniel Saul and Lukas Großberger}, 
title = {{UMAP: Uniform Manifold Approximation and Projection}}, 
journal = {Journal of Open Source Software}
}

@book{Saad2003,
author = {Saad, Yousef},
title = {Iterative Methods for Sparse Linear Systems},
publisher = {Society for Industrial and Applied Mathematics},
year = {2003},
doi = {10.1137/1.9780898718003},
address = {},
edition   = {Second}
}

@Article{vanGennip14,
author={van Gennip, Yves
and Guillen, Nestor
and Osting, Braxton
and Bertozzi, Andrea L.},
title={Mean Curvature, Threshold Dynamics, and Phase Field Theory on Finite Graphs},
journal={Milan Journal of Mathematics},
year={2014},
month={Jun},
day={01},
volume={82},
number={1},
pages={3-65},
issn={1424-9294},
doi={10.1007/s00032-014-0216-8}
}

@article{Hoffmann2019SpectralAO,
  title={Spectral analysis of weighted Laplacians arising in data clustering},
  author={Franca Hoffmann and Bamdad Hosseini and Assad A. Oberai and Andrew M. Stuart},
  journal={Applied and Computational Harmonic Analysis},
  year={2019},
  url={https://api.semanticscholar.org/CorpusID:202577644}
}

@article{MINRES,
author = {Paige, C. C. and Saunders, M. A.},
title = {Solution of Sparse Indefinite Systems of Linear Equations},
journal = {SIAM Journal on Numerical Analysis},
volume = {12},
number = {4},
pages = {617-629},
year = {1975},
doi = {10.1137/0712047}
}

@article{Bertozzi2018,
author = {Bertozzi, Andrea L. and Luo, Xiyang and Stuart, Andrew M. and Zygalakis, Konstantinos C.},
title = {Uncertainty Quantification in Graph-Based Classification of High Dimensional Data},
journal = {SIAM/ASA Journal on Uncertainty Quantification},
volume = {6},
number = {2},
pages = {568-595},
year = {2018},
doi = {10.1137/17M1134214},

URL = { 
    
        https://doi.org/10.1137/17M1134214
    
    

},
eprint = { 
    
        https://doi.org/10.1137/17M1134214
    
    

}
}

@article{MengIPOL2017,
    title   = {{Hyperspectral Image Classification Using Graph Clustering Methods}},
    author  = {Meng, Zhaoyi and Merkurjev, Ekaterina and Koniges, Alice and Bertozzi, Andrea L.},
    journal = {{Image Processing On Line}},
    volume  = {7},
    pages   = {218--245},
    year    = {2017},
    note    = {\url{https://doi.org/10.5201/ipol.2017.204}}
}

@InProceedings{pmlr-v40-Dasarathy15,
  title = 	 {S2: An Efficient Graph Based Active Learning Algorithm with Application to Nonparametric Classification},
  author = 	 {Dasarathy, Gautam and Nowak, Robert and Zhu, Xiaojin},
  booktitle = 	 {Proceedings of The 28th Conference on Learning Theory},
  pages = 	 {503--522},
  year = 	 {2015},
  editor = 	 {Grünwald, Peter and Hazan, Elad and Kale, Satyen},
  volume = 	 {40},
  series = 	 {Proceedings of Machine Learning Research},
  address = 	 {Paris, France},
  month = 	 {03--06 Jul},
  publisher =    {PMLR},
  url = 	 {https://proceedings.mlr.press/v40/Dasarathy15.html}
}

@misc{bhusal2024maladymulticlassactivelearning,
      title={{MALADY}: Multiclass Active Learning with Auction Dynamics on Graphs}, 
      author={Gokul Bhusal and Kevin Miller and Ekaterina Merkurjev},
      year={2024},
      eprint={2409.09475},
      archivePrefix={arXiv},
      primaryClass={cs.LG},
}

@article{MillerCalder23,
author = {Miller, Kevin and Calder, Jeff},
title = {Poisson Reweighted {L}aplacian Uncertainty Sampling for Graph-Based Active Learning},
journal = {SIAM Journal on Mathematics of Data Science},
volume = {5},
number = {4},
pages = {1160-1190},
year = {2023},
doi = {10.1137/22M1531981}
}

@InProceedings{pmlr-v22-ji12,
  title = 	 {A Variance Minimization Criterion to Active Learning on Graphs},
  author = 	 {Ji, Ming and Han, Jiawei},
  booktitle = 	 {Proceedings of the Fifteenth International Conference on Artificial Intelligence and Statistics},
  pages = 	 {556--564},
  year = 	 {2012},
  editor = 	 {Lawrence, Neil D. and Girolami, Mark},
  volume = 	 {22},
  series = 	 {Proceedings of Machine Learning Research},
  address = 	 {La Palma, Canary Islands},
  month = 	 {21--23 Apr},
  publisher =    {PMLR},
  url = 	 {https://proceedings.mlr.press/v22/ji12.html}
}

@article{soize2017nonparametric,
  title={A nonparametric probabilistic approach for quantifying uncertainties in low-dimensional and high-dimensional nonlinear models},
  author={Soize, Christian and Farhat, Charbel},
  journal={International Journal for Numerical methods in engineering},
  volume={109},
  number={6},
  pages={837--888},
  year={2017},
  publisher={Wiley Online Library}
}

@article{xiao2017random,
  title={A random matrix approach for quantifying model-form uncertainties in turbulence modeling},
  author={Xiao, Heng and Wang, Jian-Xun and Ghanem, Roger G},
  journal={Computer Methods in Applied Mechanics and Engineering},
  volume={313},
  pages={941--965},
  year={2017},
  publisher={Elsevier}
}

@article{geneva2019quantifying,
  title={Quantifying model form uncertainty in Reynolds-averaged turbulence models with Bayesian deep neural networks},
  author={Geneva, Nicholas and Zabaras, Nicholas},
  journal={Journal of Computational Physics},
  volume={383},
  pages={125--147},
  year={2019},
  publisher={Elsevier}
}

@article{najm2009uncertainty,
  title={Uncertainty quantification in chemical systems},
  author={Najm, Habib N and Debusschere, Bert J and Marzouk, Youssef M and Widmer, Steve and Le Ma{\^\i}tre, OP},
  journal={International journal for numerical methods in engineering},
  volume={80},
  number={6-7},
  pages={789--814},
  year={2009},
  publisher={Wiley Online Library}
}

@article{teichert2019machine,
  title={Machine learning materials physics: Surrogate optimization and multi-fidelity algorithms predict precipitate morphology in an alternative to phase field dynamics},
  author={Teichert, Gregory H and Garikipati, Krishna},
  journal={Computer Methods in Applied Mechanics and Engineering},
  volume={344},
  pages={666--693},
  year={2019},
  publisher={Elsevier}
}

@article{soize2019probabilistic,
  title={Probabilistic learning for modeling and quantifying model-form uncertainties in nonlinear computational mechanics},
  author={Soize, Christian and Farhat, Charbel},
  journal={International Journal for Numerical methods in engineering},
  volume={117},
  number={7},
  pages={819--843},
  year={2019},
  publisher={Wiley Online Library}
}





% \end{thebibliography}
\end{document}